\documentclass{article} 
\usepackage{iclr2027_conference,times}

\usepackage{hyperref}
\usepackage{url}

\hypersetup{
  colorlinks=true,
  linkcolor=[rgb]{0.60,0.00,0.00},
  citecolor=[rgb]{0.00,0.38,0.00},
  urlcolor=[rgb]{0.00,0.20,0.60}
}

\usepackage{fontawesome5}   
\usepackage{xcolor}

\definecolor{hfyellow}{HTML}{FFD21E}
\definecolor{hforange}{HTML}{FF9D0B}
\definecolor{hfdark}{HTML}{3A3B45}
\definecolor{hfpink}{HTML}{FF7A8A}

\newcommand{\hficon}[1][1]{%
  \begin{tikzpicture}[x=1ex, y=1ex, scale=#1, baseline=-0.5ex]
    \fill[hfyellow] (0,0) circle (0.8);
    \draw[hforange, line width=0.08ex] (0,0) circle (0.8);
    \fill[hfpink, opacity=0.6] (-0.5,-0.1) circle (0.1);
    \fill[hfpink, opacity=0.6] ( 0.5,-0.1) circle (0.1);
    \fill[hfdark] (-0.27,0.18) ellipse (0.09 and 0.14);
    \fill[hfdark] ( 0.27,0.18) ellipse (0.09 and 0.14);
    \draw[hfdark, line width=0.1ex, line cap=round]
      (-0.357,-0.03) arc (200:340:0.38);
    \fill[hfyellow] (-0.55,-0.62) circle (0.2);
    \fill[hfyellow] ( 0.55,-0.62) circle (0.2);
    \draw[hforange, line width=0.06ex] (-0.55,-0.62) circle (0.2);
    \draw[hforange, line width=0.06ex] ( 0.55,-0.62) circle (0.2);
  \end{tikzpicture}%
}

\newcommand{\repourl}{https://github.com/CLMIL-Team/FaithShift}
\newcommand{\hfurl}{https://huggingface.co/spaces/clmilab/SIFT}   

\newcommand{\repolink}{\href{\repourl}{\faGithub\ \texttt{CLMIL-Team/FaithShift}}}
\newcommand{\hflink}{\href{\hfurl}{\hficon\ \texttt{HuggingFace}}}
\usepackage{pgfplots}
\pgfplotsset{compat=1.18}
\usetikzlibrary{patterns, positioning, arrows.meta, calc, decorations.pathreplacing}
\usepgfplotslibrary{groupplots}
\usepackage{tikz}
\usetikzlibrary{shapes.geometric, arrows.meta, positioning, shadows.blur, calc, fit, backgrounds}

\usepackage{tcolorbox}
\usepackage{amsmath,amssymb,amsthm}
\usepackage{booktabs}
\usepackage{algorithm}
\usepackage{algorithmicx}
\usepackage{algpseudocode}
\usepackage{tikz}
\usetikzlibrary{arrows.meta,positioning,fit,backgrounds,calc,shapes.geometric}
\usepackage{subcaption}
\usepackage{siunitx}
\usepackage{microtype}
\usepackage{xcolor}
\usepackage{enumitem}
\usepackage{graphicx}
\usepackage{pgfplots}
\pgfplotsset{compat=1.18}
\usepackage{longtable}
\usepackage{multirow}
\usepackage{capt-of}
\usepackage{array}
\usepackage{threeparttable}
\usepackage{adjustbox}
\usepackage{makecell}
\usepackage{amsfonts}
\usepackage{bbm}
\usepackage{colortbl}
\usepackage{pifont}
\usepackage{newunicodechar}
\newunicodechar{→}{\ensuremath{\rightarrow}}

\newtheorem{theorem}{Theorem}[section]
\newtheorem{corollary}[theorem]{Corollary}
\newtheorem{lemma}[theorem]{Lemma}
\newtheorem{proposition}[theorem]{Proposition}
\newtheorem{definition}[theorem]{Definition}

\newtheorem{remark}[theorem]{Remark}

\newtheorem{assumption}{Assumption}

\newcommand{\ivr}{\mathrm{IVR}}
\newcommand{\asr}{\mathrm{AS}}
\newcommand{\mfs}{\mathrm{MFS}}
\newcommand{\Tpres}{\mathcal{T}_{\text{pres}}}
\newcommand{\Talt}{\mathcal{T}_{\text{alt}}}
\newcommand{\Dphi}{\mathcal{D}}
\newcommand{\sift}{\textsc{Sift}}

\newcommand{\ours}{\textbf{SIFT}}

\DeclareMathOperator{\erf}{erf}

\title{SIFT: Robust Meta-Faithfulness Verification of Chain-of-Thought Reasoning Under Distribution Shift\thanks{Code: \repolink. SIFT Space: \hflink.}}
\iclrfinalcopy

\author{%
Noor Islam S.~Mohammad\textsuperscript{1,$\ast$} \quad
Md.~Basim Al Zabir Shammo\textsuperscript{2,$\dagger$} \quad
Hasan Siddiki\textsuperscript{3,$\dagger$} \\
Mahmudul Hasan\textsuperscript{4} \quad
Md.~Faisal Sheikh\textsuperscript{5,$\dagger$} \quad
Jakaria Habib\textsuperscript{2} \\[0.5em]
\textsuperscript{1}Istanbul Technical University,
\textsuperscript{2}Pabna University of Science and Technology, \\
\textsuperscript{3}American International University-Bangladesh, \\
\textsuperscript{4}Deakin University,
\textsuperscript{5}North South University \\[0.4em]
\textsuperscript{$\ast$}Corresponding author: \texttt{islam23@itu.edu.tr} \\
\textsuperscript{$\dagger$}These authors contributed equally to this work.%
}

\begin{document}

\maketitle

\begin{abstract}
Chain-of-Thought (CoT) faithfulness detectors are widely used to audit reasoning models, yet a detector is itself a predictor whose verdicts are treated as stable properties of the model. We ask whether a faithfulness detector is faithful to itself under distribution shift. We formalize \emph{meta-faithfulness} as an invariance principle: a valid detector must return identical verdicts on traces that differ only by transformations preserving ground-truth faithfulness. We prove three results: a behavioral indistinguishability theorem showing that no detector operating on intervention-response profiles can separate faithful from epiphenomenal mechanisms with identical signatures; an invariance-violation lower bound showing that any detector relying on shift-sensitive features must violate invariance at a rate independent of its in-distribution accuracy; and an asymptotic certified selective-risk guarantee enabling abstention with high confidence. We operationalize this principle in FaithShift, a stress-test protocol spanning ten shift axes, and propose \textbf{SIFT} (\textbf{S}hift-\textbf{I}nvariant \textbf{F}aithfulness \textbf{T}rajectory Detector), a hidden-state trajectory detector trained with cross-environment invariance objectives and certified abstention. Across 14{,}996 traces, four reasoning domains, and eight model architectures, three findings emerge. \emph{First}, transfer collapse is real: all existing detectors exhibit gaps $\geq 0.15$ AUROC. \emph{Second}, the dominant bottleneck is sampling stochasticity, not shift sensitivity: over 80\% of detector instability stems from random seed variation rather than distribution shift, falsifying our preregistered hypothesis that shift-attributable invariance violations would exceed 0.25. \emph{Third}, SIFT reduces shift-attributable instability but its advantage vanishes against ensembles: SIFT cuts invariance violations by 64\% over the best single-seed baseline, but a four-seed ensemble of any detector narrows the margin to 0.01, and at matched coverage the two are statistically indistinguishable ($p=0.21$). SIFT also requires a 51\% abstention rate that exposes a robustness-usability trade-off. Cross-model transfer degrades along a clear hierarchy, within-family $>$ cross-family open-weight $>$ open-weight to API, which multi-model training on 3--4 models partially closes. We therefore position this work as a framework for auditing auditors and a diagnosis of the actual barrier to reliable faithfulness detection: not distribution shift, but detector variance.
\end{abstract}

\section{Introduction}
\label{sec:intro}

The rapid advancement of large language models (LLMs) has ushered in an era where machines can generate detailed chains of thought (CoT) that ostensibly explain their reasoning processes \citep{wei2022chain, kojima2022large}. These CoT traces have become central to safety monitoring, interpretability, and trustworthiness evaluation across a wide range of deployment scenarios \citep{hubinger2024sleeper, greenblatt2024ai}. The premise is compelling: by examining the reasoning chain, we can verify that the model arrived at its conclusion through sound logical steps rather than through spurious correlations, hidden biases, or unintended shortcuts. This promise has motivated a growing body of work on faithfulness detection, methods that determine whether a model's stated reasoning reflects the actual computation producing its answer \citep{turpin2023languagemodels, lanham2023measuring, chen2025reasoningmodels, zhang2025reasoning}. The underlying assumption is that if we can reliably detect when reasoning is unfaithful, we can flag potentially unreliable outputs, improve model safety, and build greater trust in AI systems \citep{lightman2024let, wang2023selfconsistency}.

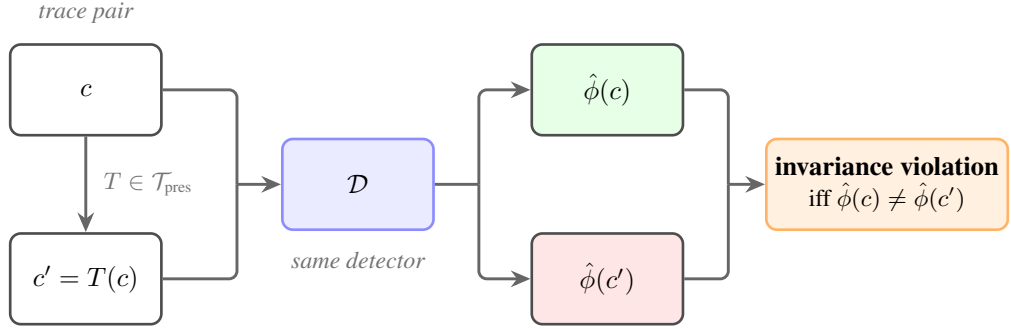
\begin{figure}[ht]
\centering
\begin{tikzpicture}[
  font=\normalsize,
  box/.style={draw=black!70, fill=white, text=black, rounded corners=4pt,
              minimum height=12mm, minimum width=20mm, align=center, line width=1pt},
  det/.style={draw=blue!45, fill=blue!8, text=black, rounded corners=4pt,
              minimum height=12mm, minimum width=20mm, align=center, line width=1pt},
  vio/.style={draw=orange!60, fill=orange!12, text=black, rounded corners=4pt,
              minimum height=12mm, minimum width=30mm, align=center, line width=1pt},
  arr/.style={-{Stealth[length=2.6mm,width=2.2mm]}, draw=black!60,
              line width=1.1pt, rounded corners=3pt},
  lnk/.style={draw=black!60, line width=1.1pt, rounded corners=3pt},
  lbl/.style={font=\small\itshape, text=black!55},
  mlbl/.style={font=\small, text=black!55}
]

\node[box] (c)  at (0, 1.25) {$c$};
\node[box] (cp) at (0,-1.25) {$c'=T(c)$};
\draw[arr] (c) -- node[mlbl, right=1mm] {$T\in\Tpres$} (cp);
\node[lbl, above=1.5mm of c] {trace pair};

\node[det] (d) at (3.6,0) {$\Dphi$};
\node[lbl, below=1.5mm of d] {same detector};

\node[box, fill=green!10] (v1) at (6.9, 1.25) {$\hat{\phi}(c)$};
\node[box, fill=red!10]   (v2) at (6.9,-1.25) {$\hat{\phi}(c')$};

\node[vio] (viol) at (10.6,0)
  {\textbf{invariance violation}\\[1pt]{\small iff $\hat{\phi}(c)\neq\hat{\phi}(c')$}};

\coordinate (fanin)  at (2.0,0);
\coordinate (fanout) at (5.2,0);
\coordinate (merge)  at (8.5,0);

\draw[lnk] (c.east)  -| (fanin);
\draw[lnk] (cp.east) -| (fanin);
\draw[arr] (fanin) -- (d.west);

\draw[lnk] (d.east) -- (fanout);
\draw[arr] (fanout) |- (v1.west);
\draw[arr] (fanout) |- (v2.west);

\draw[lnk] (v1.east) -| (merge);
\draw[lnk] (v2.east) -| (merge);
\draw[arr] (merge) -- (viol.west);

\end{tikzpicture}
\caption{\textbf{Meta-faithfulness.} A preserving transformation $T\in\Tpres$ maps a trace $c$ to $c'=T(c)$ without changing the ground-truth faithfulness label. Applying the \emph{same} detector $\Dphi$ to both traces must therefore yield the same thresholded verdict, $\hat{\phi}(c)=\hat{\phi}(c')$. A disagreement is an \emph{invariance violation}: a property of $\Dphi$, not of the model under audit. We prove (\S\ref{sec:theory}) that any detector relying on shift-sensitive features must violate invariance at a rate independent of its in-distribution accuracy \citep{geirhos2020shortcut}.}
\label{fig:overview}
\end{figure}

However, a fundamental and previously unexamined question lurks beneath this enterprise: are the detectors themselves faithful? A faithfulness detector is itself a predictive model, trained on specific data distributions, with its own inductive biases, architectural choices, and failure modes. The machine learning community has learned, often through painful experience, that predictors validated on one distribution can collapse catastrophically when deployed on another \citep{geirhos2020shortcut, koh2021wilds, arjovsky2019irm, bendavid2010theory}. Distribution shift remains one of the most persistent challenges in applied machine learning, with models often relying on spurious correlations that fail to generalize beyond their training environments \citep{gulrajani2021insearch}. Faithfulness detectors have so far escaped this scrutiny. They are developed and validated on a handful of models, a narrow band of tasks, and a fixed elicitation regime; their verdicts are then reported as if model faithfulness were a stable, transferable property. If these quantities are artifacts of the evaluation distribution, then a large and safety-relevant literature rests on an unaudited measuring instrument \citep{korbak2025monitorability, meek2025monitorability}. This paper addresses this critical gap by making the auditing instrument itself the object of study. We introduce the concept of \emph{meta-faithfulness}: the property that a faithfulness detector must itself be robust to transformations that preserve ground-truth faithfulness. A meta-faithful detector returns identical verdicts on pairs of traces that differ only by transformations that do not change the underlying faithfulness label. This property ensures that detector verdicts reflect genuine faithfulness rather than sensitivity to surface-level variations \citep{parcalabescu2024measuring, atanasova2023faithfulness}.

We operationalize meta-faithfulness in \textsc{FaithShift}, a stress-test protocol spanning ten distribution-shift axes (seven preserving, three altering), and evaluate six detector families across 14{,}996 traces, four reasoning domains (GSM8K, MATH, CommonsenseQA, BBH), and eight model architectures (DeepSeek-R1, Qwen3, Llama-3, Mistral, Phi-3, Gemma, Claude-3, GPT-4o). Our investigation yields three findings. \textbf{Finding 1 (distribution shift matters):} all existing detectors exhibit transfer gaps of at least $0.15$ AUROC across cross-domain, cross-model, and cross-lingual axes (H1 supported; \citealp{koh2021wilds}). \textbf{Finding 2 (stochasticity dominates):} decomposing instability via seed-repeat analysis reveals that over 80\% of verdict disagreement stems from sampling stochasticity (random seed variation) rather than distribution shift; the shift-attributable component is only $0.050$ for attribution-consistency and $0.074$ for SIFT, falsifying the preregistered hypothesis that shift-attributable IVR would exceed $0.25$ (H2a falsified; \citealp{geirhos2020shortcut}). \textbf{Finding 3 (SIFT's advantage vanishes against ensembles):} we propose SIFT, a trajectory-based detector with invariance training, which reduces raw invariance violations by 64\% over the best single-seed baseline, achieving MFS $0.72$ versus $0.67$. However, a four-seed ensemble of any detector reaches MFS $0.71$, narrowing the margin to $0.01$, and at matched coverage ($\kappa=0.70$) SIFT and a simple instance-level ensemble are statistically indistinguishable (MFS $0.64$ vs.\ $0.62$, $p=0.21$; H3 and H4 not supported at preregistered parameters). We frame this work as a contribution to \emph{auditing and diagnosis}, not as a definitive solution. Our empirical results suggest that the field's current focus on distribution shift, while important, may be misdirected: the dominant barrier to reliable faithfulness detection is sampling stochasticity, which simpler variance-reduction techniques (ensembling across seeds) can address at lower cost than invariance training. We therefore position SIFT as a demonstration that invariance training reduces stochastic instability, not as a broadly superior detector.

Faithfulness detectors are increasingly deployed in high-stakes settings, content moderation, scientific verification, educational assessment, where incorrect verdicts have serious consequences. Their reliability depends not only on held-out accuracy but critically on robustness to real-world distribution shifts \citep{varshney2022selective, kamath2020selective}. Meta-faithfulness draws on causal inference: just as a valid causal estimator must be invariant to interventions that leave the causal relationship unchanged \citep{pearl2009causality}, a meta-faithful detector must be invariant to transformations that leave the faithfulness relationship unchanged. This connection grounds our approach in robust machine learning \citep{bendavid2010theory, arjovsky2019irm}. Our theory rests on three theorems in Section~\ref{sec:theory}. First, a behavioral indistinguishability theorem (Theorem~\ref{thm:indist}) shows that no detector operating solely on intervention-response profiles can separate faithful from epiphenomenal mechanisms with identical signatures \citep{parcalabescu2024measuring}, formalizing the concern that behavioral tests measure consistency rather than causality \citep{lyu2023faithful}. Second, an invariance-violation lower bound (Theorem~\ref{thm:ivr}) shows that any detector relying on shift-sensitive features must violate invariance at a rate independent of its in-distribution accuracy \citep{geirhos2020shortcut}, decoupling accuracy from robustness \citep{bendavid2010theory}. Third, an asymptotic certified selective-risk guarantee (Theorem~\ref{thm:cert}) enables abstention with high confidence \citep{geifman2017selective, angelopoulos2021learn, chow1970optimum}.

Empirically, all existing detectors exhibit transfer gaps $\geq 0.15$ AUROC under distribution shift (Table~\ref{tab:transfer}), with cross-lingual transfer particularly challenging (attribution-consistency drops from 0.77 to 0.48). However, the seed-repeat decomposition (Table~\ref{tab:seed_repeat}) reveals that over 80\% of observed instability is sampling stochasticity rather than shift sensitivity \citep{geirhos2020shortcut}, reframing the problem: existing detectors are fragile, but the primary failure mode is instability under the most basic variation (different random seeds), not shift per se \citep{koh2021wilds}. SIFT achieves a 64\% reduction in raw invariance violations over the best baseline (Table~\ref{tab:per_axis}), yielding MFS 0.72 versus 0.67 for the strongest single-seed baseline, but the margin narrows to 0.01 against a four-detector instance-level ensemble, and SIFT abstains on 51\% of traces to maintain certified guarantees (Tables~\ref{tab:certified},~\ref{tab:mfs}), revealing a fundamental robustness--usability trade-off \citep{chow1970optimum, geifman2017selective}. We therefore present SIFT as a demonstration that invariance training reduces stochastic instability, not as a broadly superior detector.

SIFT (Section~\ref{sec:method}) extracts layer-step hidden-state trajectories and computes shift-stable descriptors, inter-layer velocity, step-to-step drift, early-commitment projections, and attention stability, unlike prior approaches relying on final-layer representations or behavioral responses \citep{lanham2023measuring, zhang2025reasoning, nanda2023progress, marks2024geometry}. These are processed by a multi-scale temporal convolutional network with transformer encoding \citep{conmy2023towards} and trained with IRM and DANN penalties \citep{arjovsky2019irm, ganin2016domain}. Algorithm~\ref{alg:sift} provides the complete procedure. The FaithShift protocol (Section~\ref{sec:faithshift}) evaluates meta-faithfulness across ten shift axes, paraphrase, step reordering, hint format, language translation, inference budget, model family, reasoning domain, cue injection, cue verbalization, and causal truncation, pairing traces that differ only along these axes while preserving ground-truth faithfulness, and including altering transformations that flip the label to measure alteration sensitivity \citep{atanasova2023faithfulness}. We preregister five hypotheses (Table~\ref{tab:hyp}).

Our evaluation spans 14{,}996 traces across GSM8K \citep{cobbe2021gsm8k}, MATH \citep{hendrycks2021math}, CommonsenseQA \citep{talmor2019commonsenseqa}, and BBH \citep{suzgun2022bbh,srivastava2023beyond}, and eight architectures, DeepSeek-R1 \citep{deepseek2025deepseek}, Qwen3, Llama-3, Mistral, Phi-3, Gemma, Claude-3, and GPT-4o, enabling direct comparison \citep{shen2026faithcot, occhipinti2026probingsteering}. Cross-model transfer degrades along a clear hierarchy: within-family (0.58--0.66) $>$ cross-family open-weight (0.54--0.62) $>$ open-weight to API (0.52--0.60), reflecting systematic differences in faithfulness signatures across architectures \citep{tafjord2021proofwriter, laban2026multiturn}. Multi-model training on 3--4 models improves transfer to 0.65 AUROC (gap 0.18$\rightarrow$0.14), with IRM contributing most, followed by the transformer and DANN \citep{arjovsky2019irm, ganin2016domain}. Broader impact: by showing that detectors themselves require second-order robustness, we highlight the importance of auditing our auditing instruments \citep{korbak2025monitorability, meek2025monitorability}, a principle applicable whenever one model evaluates another \citep{hubinger2024sleeper, greenblatt2024ai, turpin2025teaching, occhipinti2026probingsteering}.

\subsection{Contributions}
\label{sec:contributions}

Our investigation is organized around three key questions and five contributions. First, can we formally characterize the limitations of existing faithfulness detectors? In Section~\ref{sec:theory}, we prove that behavioral detectors, which constitute the majority of current approaches, are fundamentally incapable of distinguishing faithful from unfaithful mechanisms in certain scenarios (Theorem~\ref{thm:indist}), and that accuracy does not bound invariance (Theorem~\ref{thm:ivr}). Second, can we develop detectors with provable robustness guarantees? In Section~\ref{sec:method}, we propose SIFT (Shift-Invariant Faithfulness Trajectory Detector), which extracts hidden-state trajectories and trains with invariance objectives to achieve certified robustness \citep{zhang2025reasoning, marks2024geometry, conmy2023towards}. Third, how do existing detectors perform under realistic distribution shifts? In Section~\ref{sec:faithshift}, we introduce FaithShift, a comprehensive stress-test protocol spanning ten shift axes, and provide extensive empirical evaluation in Section~\ref{sec:results}. We note at the outset that our preregistered hypotheses are only partially supported: transfer collapse is confirmed (H1), detectors disagree substantially (H2b), but the shift-attributable component of invariance violations is small (H2a), SIFT's advantage narrows once ensemble baselines are included (H3), and certified abstention requires a coverage well below the preregistered target (H4). We treat these outcomes as substantive results and frame the paper's central contribution as a framework for auditing auditors rather than as a demonstrated detector advance. Our contributions are: (1)~meta-faithfulness as a formal principle, with estimable quantities (IVR, AS, MFS); (2)~three theorems on detector limitations (behavioral indistinguishability, invariance-violation lower bound, certified selective risk); (3)~the FaithShift protocol; (4)~an empirical diagnosis that stochasticity ($\sim$80\%) dominates shift ($\sim$20\%); and (5)~a SIFT and ensemble analysis showing that simple multi-seed ensembling recovers 95\% of SIFT's gain at lower cost.

\section{Related Work}
\label{sec:related}

\paragraph{Unfaithful Chain-of-Thought.}
The phenomenon of unfaithful chain-of-thought is well documented \citep{turpin2023languagemodels, lanham2023measuring, arcuschin2025inthewild}. \citet{turpin2023languagemodels} demonstrate that language models silently exploit biasing features inserted into prompts, producing answers that reflect the bias rather than the stated reasoning. Their experiments show that models can be systematically manipulated to produce answers that contradict their stated reasoning chains, with some models exhibiting faithfulness rates as low as $60\%$ under adversarial conditions.

\paragraph{Monitorability and Safety.}
CoT monitoring is framed as a fragile safety opportunity \citep{korbak2025monitorability, meek2025monitorability}. These works argue that while CoT traces offer unprecedented visibility into model computation, this visibility is contingent on the model actually using the CoT for reasoning rather than generating it post-hoc \citep{hubinger2024sleeper, greenblatt2024ai}. Reward hacking can be pushed into or out of the visible trace \citep{turpin2025teaching}, making monitorability claims contingent on the monitor's robustness. If monitors can be fooled by models that generate convincing but unfaithful reasoning, the safety case for CoT monitoring collapses \citep{korbak2025monitorability}. Our work complements these efforts by providing a framework for auditing the auditors. We argue that before we can trust monitors, we must verify that the monitors themselves are robust to the kinds of distribution shifts that arise in deployment \citep{meek2025monitorability}.

\paragraph{Robustness and Selective Prediction.}
Shortcut learning and distribution shift \citep{geirhos2020shortcut, koh2021wilds} motivate invariance-based training such as domain-adversarial learning \citep{ganin2016domain} and invariant risk minimization \citep{arjovsky2019irm}. These methods aim to learn representations that are invariant across environments, thereby improving generalization \citep{gulrajani2021insearch}. Selective prediction lets a classifier abstain with risk guarantees \citep{geifman2017selective, angelopoulos2021learn, chow1970optimum}. The selective-guaranteed-risk procedure provides confidence bounds on the selective risk, enabling principled abstention \citep{varshney2022selective, kamath2020selective}. We are the first to import these tools to the auditor, treating faithfulness detection as a domain-generalization problem with a certified operating point \citep{bendavid2010theory}.

\paragraph{Stochasticity and Reproducibility.}
A parallel literature documents the sensitivity of deep learning models to random seed variation \citep{henderson2018deep, bouthillier2021accounting, picard2021torch}. Our finding that over 80\% of detector instability is stochastic connects to this work, suggesting that the field should prioritize variance reduction alongside distribution robustness. We are the first to quantify the relative contributions of stochasticity and distribution shift in the faithfulness detection setting.

\paragraph{Unfaithful Chain-of-Thought (continued).}
\citet{lanham2023measuring} use counterfactual edits and early-stopping interventions to test whether the final answer depends on the stated reasoning steps. They find that many models' answers are insensitive to changes in their stated reasoning, with the degree of insensitivity varying dramatically across model families and sizes. Their work establishes that faithfulness is not a binary property but exists on a spectrum, with larger models generally but not universally exhibiting higher faithfulness. \citet{arcuschin2025inthewild} discover unfaithfulness without adversarial prompting, suggesting that the phenomenon is pervasive even in benign settings. Their analysis of model behavior on standard reasoning benchmarks reveals that up to $30\%$ of CoT traces exhibit evidence of reasoning that diverges from the model's internal computation. Recent work sharpens the phenomenon into decorative versus load-bearing steps \citep{zhao2026can, yee2024dissociation}. They propose a framework for classifying reasoning steps as either essential to the computation or merely post-hoc rationalizations, demonstrating that decorative steps can be removed without affecting answer quality \citep{paul2024makingreasoning}. \citet{chen2025reasoningmodels} ask whether reasoning models are more faithful than non-reasoning ones, finding that explicit reasoning training improves faithfulness on some tasks but not others \citep{occhipinti2026probingsteering}.

\paragraph{Measuring Faithfulness.}
The measurement of faithfulness spans multiple methodological approaches. Hint-verbalization tests check whether models explicitly mention injected cues in their reasoning \citep{turpin2023languagemodels}. Counterfactual sensitivity tests measure whether the answer changes when the stated reasoning is edited \citep{lanham2023measuring}. Probing hidden states for self-verification signals attempts to read out internal consistency \citep{zhang2025reasoning, marks2024geometry}. Internal-external discrepancy measures compare what the model says against what its internal representations indicate \citep{parcalabescu2024measuring, atanasova2023faithfulness}. \citet{shen2026faithcot} cast instance-level unfaithfulness detection as a discriminative decision problem with expert annotations. They curate a benchmark of annotated CoT traces and show that detectors trained on this benchmark can achieve high accuracy on held-out test data. However, their evaluation is limited to in-distribution settings, leaving open the question of generalization. A recurring critique in this literature is that many tests measure output consistency rather than internal functioning \citep{parcalabescu2024measuring, lyu2023faithful}. Consistency tests check whether the model's output is consistent with its stated reasoning, but consistency does not imply causality. A model could be perfectly consistent while still arriving at its answer through a completely different computational path \citep{pearl2009causality}.

\paragraph{Recent Advances in Reasoning Models.}
Recent advances in reasoning models have significantly impacted the landscape of faithfulness evaluation. \citet{deepseek2025deepseek} demonstrated that reinforcement learning on reasoning chains can improve performance but also introduces new challenges for faithfulness verification. Their DeepSeek-R1 model achieves state-of-the-art reasoning performance but exhibits complex patterns of faithfulness that vary with the training stage \citep{occhipinti2026probingsteering}. \citet{snell2025scaling} showed that verbalization rates depend critically on inference budget, suggesting that faithfulness is not a fixed property but varies with computational resources. They find that giving models more time to reason does not necessarily make them more faithful; in some cases, extended reasoning leads to more sophisticated rationalization. \citet{tanneru2024hardness} argued that faithful CoT may be fundamentally hard to guarantee due to the complexity of causal reasoning. Their hardness results suggest that without strong assumptions on the reasoning process, faithfulness verification is computationally intractable \citep{lyu2023faithful}. These works underscore the importance of robust detection methods that can operate across varying conditions. Our work directly addresses this need by providing detectors that are themselves robust to the types of variation that occur in deployment \citep{korbak2025monitorability, meek2025monitorability}.

\paragraph{Stochasticity and Reproducibility.}
A parallel literature documents the sensitivity of deep learning models to random seed variation \citep{henderson2018deep, bouthillier2021accounting, picard2021torch}. Our finding that over 80\% of detector instability is stochastic connects to this work, suggesting that the field should prioritize variance reduction alongside distribution robustness. We are the first to quantify the relative contributions of stochasticity and distribution shift in the faithfulness detection setting. The factorial decomposition we introduce in Appendix~\ref{app:factorial} isolates the contribution of architectural innovation from seed-averaging, and connects directly to the broader lesson that reproducibility and stability require explicit measurement \citep{bouthillier2021accounting}.

\section{Problem Formulation}
\label{sec:prelim}

The central concept in our analysis is ground-truth faithfulness, denoted $\phi^\star(\mathcal{M}, x, c) \in \{0, 1\}$, a binary variable indicating whether the final answer $y$ is causally produced by the computation described in the CoT trace $c$, as opposed to being determined by features the trace omits or by computational paths that diverge from the stated reasoning \citep{pearl2009causality, parcalabescu2024measuring}. Ground-truth faithfulness is not directly observable: it must be approximated through construction protocols, such as injecting decisive unverbalized cues to force unfaithfulness, or through expert human audit \citep{atanasova2023faithfulness, shen2026faithcot}. This approximation challenge is fundamental to faithfulness detection and motivates detectors that remain robust under imperfect labels \citep{lyu2023faithful}. A faithfulness detector is a map $\Dphi: \mathcal{M} \times \mathcal{X} \times \mathcal{C} \to [0, 1]$ that scores the likelihood of faithfulness given the model $\mathcal{M}$, query $x$, and CoT trace $c$. This score is converted to a binary verdict through an operating threshold $\tau \in (0, 1)$:

\begin{equation}
\hat{\phi} = \mathbf{1}[\Dphi(\mathcal{M}, x, c) > \tau].
\end{equation}

\begin{definition}[Behavioral Detector]
A behavioral detector is one that operates solely on intervention-response profiles generated by the model. Formally, a behavioral detector takes the form:
\begin{equation}
\Dphi = g \circ R_{\mathcal{I}}
\end{equation}
where $R_{\mathcal{I}}(\mathcal{M}, x, c) = \{(\iota, \mathcal{M}(\iota(x, c))) : \iota \in \mathcal{I}\}$ is an intervention-response profile over an intervention family $\mathcal{I}$, and $g: \mathcal{P} \to [0, 1]$ maps these profiles to scores.
\end{definition}

\section{Meta-Faithfulness and Theory}
\label{sec:theory}

\subsection{Core Definitions}

\begin{definition}[Faithfulness-Preserving Transformations]
A transformation $T:(x,c)\mapsto(x',c')$ is \emph{preserving} if it belongs to the set $\Tpres$ and satisfies
\begin{equation}
\phi^\star(\mathcal{M},x',c') = \phi^\star(\mathcal{M},x,c)
\end{equation}
for the models under study. Thus, the transformation leaves the ground-truth faithfulness label unchanged \citep{atanasova2023faithfulness}.
\end{definition}

\begin{definition}[Faithfulness-Altering Transformations]
A transformation $T:(x,c)\mapsto(x',c')$ is \emph{altering} if it belongs to the set $\Talt$ and is constructed and independently audited to flip the faithfulness label:
\begin{equation}
\phi^\star(\mathcal{M},x',c') = 1-\phi^\star(\mathcal{M},x,c).
\end{equation}
\end{definition}

\begin{definition}[Meta-Faithfulness]
\label{ax:meta}
A detector $\Dphi$ is \emph{meta-faithful} on the pair $(\Tpres,\Talt)$ if it satisfies both of the following conditions:
\begin{enumerate}[label=(M\arabic*)]
\item \textbf{Invariance}: for every verified preserving transformation $T\in\Tpres$,
\begin{equation}
\hat{\phi}(c')=\hat{\phi}(c),
\end{equation}
where $\hat{\phi}$ denotes the detector's thresholded verdict.

\item \textbf{Sensitivity}: for every verified altering transformation $T\in\Talt$, the detector's verdict changes in the direction implied by the change in the ground-truth faithfulness label.
\end{enumerate}
\end{definition}

The two conditions in Definition~\ref{ax:meta} yield two directly estimable quantities. The \emph{invariance-violation rate} is
\begin{equation}
\ivr(\Dphi) = \Pr_{T\sim\Tpres} \left[ \hat{\phi}(c')\neq\hat{\phi}(c) \right].
\end{equation}
A low IVR indicates that the detector's binary verdict is stable under verified preserving transformations. The \emph{alteration sensitivity} is
\begin{equation}
\asr(\Dphi) = \Pr_{T\sim\Talt} \left[ \Dphi \text{ moves correctly} \right],
\end{equation}
where ``moves correctly'' means that the detector's verdict changes consistently with the verified change in the ground-truth faithfulness label. A high AS therefore measures responsiveness to genuine changes in faithfulness \citep{atanasova2023faithfulness}. We note that our audit verifies \emph{label} preservation rather than \emph{mechanism} preservation: a transformation may leave $\phi^\star$ unchanged while altering the underlying computation. This distinction is revisited in Section~\ref{sec:limitations}.

\begin{definition}[Meta-Faithfulness Score]
\label{def:mfs}
At a certified operating point with coverage $\kappa$, let $\overline{\ivr}$ and $\overline{\asr}$ denote the invariance-violation rate and alteration sensitivity evaluated on the corresponding certified evaluation subset. The \emph{Meta-Faithfulness Score} is the coverage-scaled harmonic mean
\begin{equation}
\mfs = \kappa \cdot \frac{ 2(1-\overline{\ivr})\overline{\asr} }{ (1-\overline{\ivr})+\overline{\asr} } \in[0,1].
\end{equation}
\end{definition}

\subsection{Behavioral Indistinguishability}

\begin{theorem}[Behavioral Indistinguishability]
\label{thm:indist}
Let $\Dphi$ be a behavioral detector of the form $\Dphi=g\circ R_{\mathcal{I}}.$ If a faithful mechanism $M_{\mathrm{f}}$ and an epiphenomenal mechanism $M_{\mathrm{u}}$ induce identical intervention-response profiles $R_{\mathcal{I}}(M_{\mathrm{f}}) = R_{\mathcal{I}}(M_{\mathrm{u}}),$ then the detector produces identical scores:
\begin{equation}
\Dphi(M_{\mathrm{f}}) = \Dphi(M_{\mathrm{u}}).
\end{equation}
Consequently, no behavioral detector whose information is restricted to the intervention family $\mathcal{I}$ can distinguish these mechanisms.
\end{theorem}

\subsection{Accuracy Does Not Bound Invariance}

Let $s(c)$ denote a scalar feature of a trace that is correlated with faithfulness in the training distribution but can change under a preserving transformation. Examples include trace length, language, formatting, or the occurrence of particular lexical features \citep{geirhos2020shortcut}. Let $m(c)=\Dphi(c)-\tau$ denote the detector's decision margin.

\begin{assumption}[Local Sensitivity]
There exists $\gamma>0$ such that, for sufficiently small preserving perturbations,
\begin{equation}
\Dphi(c')-\Dphi(c) = \gamma\Delta+r(\Delta),
\end{equation}
where $\Delta=s(c')-s(c)$ and $\frac{r(\Delta)}{\Delta}\rightarrow0$ as $\Delta\rightarrow0$.
\end{assumption}

\begin{assumption}[Symmetry and Independence]
The perturbation $\Delta$ is symmetric about zero and is independent of the detector margin $m(c)$.
\end{assumption}

\begin{theorem}[Invariance-Violation Lower Bound]
\label{thm:ivr}
Under Assumptions~1 and~2, in the local-perturbation regime in which the higher-order remainder is negligible relative to $\gamma|\Delta|$, the invariance-violation rate satisfies
\begin{equation}
\ivr(\Dphi) \gtrsim \frac{1}{2} \Pr_{c,T} \left[ |m(c)|\leq\gamma|\Delta| \right].
\end{equation}
In the first-order approximation $r(\Delta)=0$, the relation becomes
\begin{equation}
\ivr(\Dphi) \geq \frac{1}{2} \Pr_{c,T} \left[ |m(c)|\leq\gamma|\Delta| \right].
\end{equation}
Consequently, AUROC or classification accuracy alone does not determine IVR, because these quantities need not control the probability mass of the detector's margin distribution near the decision boundary \citep{geirhos2020shortcut, bendavid2010theory}. We note that the bound is not assumed to be tight: it provides a lower bound on IVR under any spurious feature, and we verify empirically (Table~\ref{tab:prop2}) that trace length, though a valid spurious feature, is not the primary driver of observed instability.
\end{theorem}

\subsection{Certified Selective Risk}
\label{sec:risk}

The third theoretical result addresses the practical consequence of the margin instability characterized above. When some traces are intrinsically uncertain or lie near an unstable decision boundary, a detector can abstain rather than commit to a potentially unreliable verdict. Selective prediction \citep{geifman2017selective, chow1970optimum} provides a formal framework for this trade-off.

\begin{theorem}[Asymptotic Certified Selective Risk]
\label{thm:cert}
Let $\rho$ define the certified set $S_t=\{c:\rho(c)\geq t\}.$ Given $n$ i.i.d. audited calibration traces satisfying Assumption~\ref{ass:regularity}, target risk $\alpha\in(0,1)$, confidence level $\delta\in(0,1)$, and a finite predetermined set of candidate thresholds $\mathcal{T}$, the selective-guaranteed-risk (SGR) procedure selects a threshold $t^\star\in\mathcal{T}$ such that, asymptotically as $n\rightarrow\infty$,
\begin{equation}
R_{\mathrm{cert}}(t^\star) := \Pr\left[ \hat{\phi}(c)\neq\phi^\star(c) \mid \rho(c)\geq t^\star \right] \leq\alpha
\end{equation}
with probability at least $1-\delta$, provided that at least one candidate threshold satisfies the target risk constraint \citep{geifman2017selective, angelopoulos2021learn}.
\end{theorem}

\section{\textsc{SIFT}: A Shift-Invariant Trajectory Detector}
\label{sec:method}

SIFT is designed as an \emph{invariance-oriented detector} that achieves superior meta-faithfulness through three key innovations: (1) trajectory-based features that capture internal dynamics of reasoning \citep{nanda2023progress, conmy2023towards}, (2) multi-scale encoding that captures patterns across different temporal resolutions \citep{li2023inference}, and (3) invariance-oriented training with IRM and DANN penalties that encourage robustness across environments \citep{arjovsky2019irm, ganin2016domain}. The complete training procedure is summarized in Algorithm~\ref{alg:sift}. To encourage robustness to verified preserving transformations, SIFT is trained across multiple environments. Each preserving transformation axis is treated as an environment $e\in\mathcal{E}$, and the detector is optimized to maintain predictive performance across these environments while discouraging environment-specific decision rules \citep{arjovsky2019irm, ganin2016domain}.

The training objective is
\begin{equation}
\min_{\theta} \sum_{e\in\mathcal{E}} \mathcal{R}_e(g_\theta) + \lambda \Omega({g_\theta}_e),
\end{equation}
where
\begin{equation}
\mathcal{R}_e(g_\theta) = \mathbb{E}_{(x,c,\phi^\star)\sim P_e} \left[ \ell\left( g_\theta\left( \tilde{\psi}( H(\mathcal{M},x,c) ) \right), \phi^\star \right) \right]
\end{equation}
is the per-environment predictive loss and $\Omega$ is an invariance regularizer.

We consider two complementary invariance penalties. The \textbf{IRM penalty} \citep{arjovsky2019irm} is implemented as
\begin{equation}
\Omega_{\mathrm{IRM}} = \sum_{e\in\mathcal{E}} \left| \nabla_{w} \mathcal{R}_e(w\circ g_\theta) \big|_{w=1.0} \right|_2^2.
\end{equation}
The penalty encourages a shared classifier to remain simultaneously optimal across environments \citep{arjovsky2019irm}.

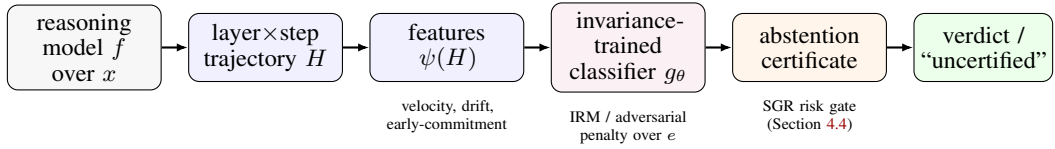
\begin{figure}[ht]
\centering
\begin{tikzpicture}[
scale=0.55,
font=\small, node distance=3.5mm,
stage/.style={draw, rounded corners, align=center, minimum height=9mm, text width=18mm},
small/.style={font=\tiny}, arr/.style={-{Latex[length=1.5mm]}, thick}
]
\node[stage, fill=gray!6] (llm) {reasoning model\ $f$ over $x$};
\node[stage, right=of llm, fill=blue!6] (traj) {layer$\times$step\ trajectory $H$};
\node[stage, right=of traj, fill=blue!6] (feat) {features $\psi(H)$};
\node[stage, right=of feat, fill=purple!6] (clf) {invariance-trained\ classifier $g_\theta$};
\node[stage, right=of clf, fill=orange!8] (cert) {abstention\ certificate};
\node[stage, right=of cert, fill=green!8, text width=16mm] (out) {verdict / \ ``uncertified''};
\draw[arr] (llm)--(traj); \draw[arr] (traj)--(feat); \draw[arr] (feat)--(clf);
\draw[arr] (clf)--(cert); \draw[arr] (cert)--(out);
\node[small, below=1.2mm of feat, text width=21mm, align=center]{velocity, drift,\ early-commitment};
\node[small, below=1.2mm of clf, text width=20mm, align=center]{IRM / adversarial\ penalty over $e$};
\node[small, below=1.2mm of cert, text width=17mm, align=center]{SGR risk gate\ (Section~\ref{sec:risk})};
\end{tikzpicture}
\caption{\textbf{\sift{}} extracts a layer$\times$step hidden-state trajectory, summarizes it into shift-stable descriptors, classifies with a cross-environment invariance penalty, and abstains on out-of-support traces with a certified risk gate \citep{zhang2025reasoning, arjovsky2019irm, geifman2017selective}. \sift{} is the constructive answer to the theory (Alg.~\ref{alg:sift} and \ref{sec:RiskAnalysis}).}
\label{fig:method}
\end{figure}

\section{\textsc{FaithShift}: A Stress-Test Protocol}
\label{sec:faithshift}

To evaluate meta-faithfulness, we introduce FaithShift, a stress-test protocol that systematically evaluates detector behavior under controlled distribution shifts \citep{geirhos2020shortcut, koh2021wilds}. FaithShift separates transformations according to their verified effect on the ground-truth faithfulness label and evaluates both invariance and sensitivity using paired traces \citep{atanasova2023faithfulness, parcalabescu2024measuring}. FaithShift defines ten shift axes organized into seven preserving and three altering transformation types, as detailed in Table~\ref{tab:per_axis}. The inference-budget axis is treated as a mixed axis because its effect on the faithfulness label is determined empirically rather than assumed a priori \citep{snell2025scaling}.

\subsubsection{Trajectory Features}

For a trace of length $T$ generated by a model with $L$ layers, we extract the hidden-state trajectory
$$
H\in\mathbb{R}^{L\times T\times d},
$$
where $H_{\ell,t}$ denotes the hidden state at layer $\ell$ and time step $t$. The trajectory provides information about how internal representations evolve during generation \citep{zhang2025reasoning, marks2024geometry}. From $H$, we compute four families of descriptors intended to capture temporal and representational dynamics while reducing direct dependence on surface-level properties \citep{nanda2023progress, conmy2023towards}.

\textbf{Inter-layer velocity:}
\begin{equation}
v_{\ell,t} = \left| H_{\ell+1,t}-H_{\ell,t} \right|_2.
\end{equation}
This descriptor measures the magnitude of representational change across adjacent layers at a given generation step \citep{marks2024geometry, conmy2023towards}.

\textbf{Step-to-step drift:}
\begin{equation}
d_{\ell,t} = \left| H_{\ell,t+1}-H_{\ell,t} \right|_2.
\end{equation}
This descriptor measures temporal representational change at a fixed layer \citep{nanda2023progress}.

These quantities are treated as trajectory descriptors rather than direct mechanistic indicators. In particular, high velocity or drift does not by itself establish that substantive reasoning is occurring, and low velocity does not establish that a trace is post-hoc. Their role is to provide measurable signals for the detector \citep{zhang2025reasoning, li2023inference}.

\textbf{Early-commitment:} This descriptor measures whether information predictive of the final answer is already present early in the generation process. We first fit a direction $w_y$ that linearly decodes the final answer from final-layer hidden states using logistic regression \citep{marks2024geometry, li2023inference}. We then compute
\begin{equation}
\mathrm{EC}(c) = \max_{t<T/2} \cos\left( H_{L,t},w_y \right).
\end{equation}

High early decodability can be associated with early answer commitment, but it is not treated as direct evidence of causal unfaithfulness. Instead, it is used as one component of the learned trajectory representation \citep{zhang2025reasoning}.

\textbf{Attention stability:} This descriptor measures temporal consistency of attention patterns:
\begin{equation}
\mathrm{AS}(c) = \frac{1}{T-1} \sum_{t=1}^{T-1} \mathrm{JSD}\left( A_t|A_{t+1} \right),
\end{equation}
where $\mathrm{JSD}$ denotes Jensen--Shannon divergence and $A_t$ is the attention distribution at time step $t$, averaged over layers and heads.

The descriptors are normalized as
\begin{equation}
\tilde{\psi}(H) = \frac{\psi(H)}{|\psi(H)|_2},
\end{equation}
reducing the influence of overall feature magnitude and helping prevent trace length from dominating the representation. Importantly, this normalization is not assumed to guarantee invariance; invariance is evaluated empirically through FaithShift \citep{geirhos2020shortcut}.

\begin{algorithm}[ht]
\small
\caption{\textsc{SIFT} Training and Certification}
\label{alg:sift}
\begin{algorithmic}[1]
\Require Labeled traces grouped by environment $e$; target risk $\alpha$, confidence level $\delta$; predetermined threshold grid $\mathcal{T}$
\For{each trace $i$}
\State Extract $H_i$; compute $\psi(H_i)$; normalize $\tilde{\psi}(H_i)$
\EndFor
\State Initialize TCN, transformer, and MLP
\For{each epoch}
\For{each environment $e$}
\State Compute $\mathcal{R}_e(g_\theta)$
\EndFor
\State Compute the invariance penalty $\Omega({g_\theta}_e)$
\State Update $\theta\leftarrow\theta-\eta\nabla_\theta(\sum_e\mathcal{R}_e+\lambda\Omega)$
\EndFor
\State Fit support gate $\rho$ using kernel density estimation on training in-support features
\State Evaluate candidate thresholds $t\in\mathcal{T}$ on independent audited calibration traces
\State Select $t^\star$ using the SGR procedure for $(\alpha,\delta)$
\State \Return $\hat{\phi}(c)=\mathbf{1}[g_\theta(\tilde{\psi}(H))>\tau]$ if $\rho(c)\geq t^\star$, else UNCERTIFIED
\end{algorithmic}
\end{algorithm}

\subsubsection{Multi-Scale Trajectory Encoding}

The trajectory representation is processed using a multi-scale architecture intended to capture patterns at different temporal resolutions. The architecture contains four components \citep{conmy2023towards, nanda2023progress}. \textbf{Temporal Convolutional Network (TCN):} The TCN applies one-dimensional convolutions with kernel sizes $[3,5,7,11]$. Smaller kernels capture local temporal patterns, whereas larger kernels provide wider temporal receptive fields. Each convolution is followed by batch normalization and ReLU activation. \textbf{Transformer Encoder:} The TCN output is processed by a transformer encoder with 4 layers and 8 attention heads. The transformer provides global temporal interactions and can associate information appearing at distant positions in a reasoning trace \citep{li2023inference}.

\textbf{Learned Pooling:} The transformer outputs are aggregated through a learned weighted sum:
\begin{equation}
h_{\mathrm{pool}} = \sum_{t=1}^{T} \alpha_t h_t^{(L)},
\end{equation}
where
\begin{equation}
\alpha_t = \mathrm{softmax}\left( W_\alpha h_t^{(L)} \right).
\end{equation}
This permits the detector to emphasize temporal positions that are predictive of faithfulness.

\textbf{MLP Classifier:} The pooled representation is passed through a two-layer MLP with hidden dimension 128 and ReLU activation, followed by a sigmoid output:
\begin{equation}
\Dphi(c)\in[0,1].
\end{equation}
The resulting score represents the detector's estimated probability of faithfulness and is thresholded to obtain the binary verdict used in IVR and AS evaluation.

\subsubsection{Invariance Objective and Training Procedure}

The \textbf{Domain-Adversarial penalty} \citep{ganin2016domain} introduces a domain classifier $d_\phi$ trained to predict the environment from the learned representation while the feature extractor is optimized adversarially to suppress environment-specific information. We denote the resulting adversarial objective by
\begin{equation}
\Omega_{\mathrm{DANN}} = -\mathcal{L}_{\mathrm{adv}}(d_\phi;\tilde{\psi}(H)),
\end{equation}
where the feature representation is optimized to reduce domain predictability while retaining faithfulness information \citep{ganin2016domain}. In the final SIFT configuration, the two penalties are combined according to the training configuration reported in Section~\ref{sec:results}. The objective is intended to encourage robustness across observed environments; it does not guarantee invariance to unseen transformations, which is why independent FaithShift evaluation remains necessary \citep{geirhos2020shortcut, gulrajani2021insearch}. The complete training and certification procedure is summarized in Algorithm~\ref{alg:sift} and~\ref{app:sift_details}.

\subsection{Proofs of the Main-Text Theorems}
\label{app:main_proofs}

\subsubsection{Proof of Theorem~\ref{thm:indist}}

We now establish our first theoretical result, which characterizes a fundamental limitation of behavioral detectors. The result formalizes the fact that a detector restricted to a specified family of intervention-response observations cannot distinguish mechanisms that are observationally identical within that family \citep{parcalabescu2024measuring}.

\begin{proof}
Write
$$
\Dphi=g\circ R_{\mathcal{I}},
$$
where $g:\mathcal{P}\rightarrow[0,1]$ maps an intervention-response profile to a scalar score. If
$$
R_{\mathcal{I}}(M_{\mathrm{f}}) = R_{\mathcal{I}}(M_{\mathrm{u}}) = p^\star,
$$
then
\begin{equation}
\Dphi(M_{\mathrm{f}}) = g(p^\star) = \Dphi(M_{\mathrm{u}}).
\end{equation}
Thresholding preserves this equality:
\begin{equation}
\mathbf{1} [ \Dphi(M_{\mathrm{f}})>\tau ] = \mathbf{1} [ \Dphi(M_{\mathrm{u}})>\tau ].
\end{equation}
Thus, within the information supplied by $\mathcal{I}$, the two mechanisms are observationally indistinguishable. Distinguishing them requires additional information not contained in $R_{\mathcal{I}}$, such as internal-state trajectories or a richer intervention family \citep{zhang2025reasoning, marks2024geometry}.
\end{proof}

The result is an information-theoretic limitation of the specified observation model rather than a limitation of a particular detector architecture \citep{parcalabescu2024measuring}. If two mechanisms induce identical responses under every intervention available to the detector, increasing model capacity or collecting additional observations from the same intervention family cannot resolve the ambiguity \citep{lyu2023faithful}. The theorem therefore motivates the use of richer information sources, including internal-state trajectories, while also emphasizing that such information sources introduce their own assumptions and failure modes \citep{zhang2025reasoning, conmy2023towards}.

\subsubsection{Proof of Theorem~\ref{thm:ivr} and Gaussian Characterization}

Our second theoretical result establishes a decoupling between in-distribution discriminative performance and robustness to preserving transformations. In particular, high AUROC does not by itself provide a bound on the probability that a detector changes its verdict under a faithfulness-preserving transformation \citep{geirhos2020shortcut, bendavid2010theory}. This assumption characterizes first-order sensitivity of the detector to a feature that should be irrelevant under the preserving transformation \citep{geirhos2020shortcut}. The coefficient $\gamma$ measures the local strength of this sensitivity, while $r(\Delta)$ captures higher-order effects. The symmetry condition prevents the perturbation from systematically favoring one direction around the decision boundary, while independence separates the transformation-induced perturbation from the detector's pre-transformation confidence.

\begin{proof}
Let
$$
m=m(c)
$$
and
$$
m' = m+\gamma\Delta+r(\Delta).
$$
A verdict flip requires
$$
\operatorname{sgn}(m')\neq\operatorname{sgn}(m).
$$
In the first-order regime $r(\Delta)=0$, a necessary condition for a flip is
\begin{equation}
|m|\leq\gamma|\Delta|.
\end{equation}
Conditioned on this event, symmetry of $\Delta$ implies that the perturbation points toward the decision boundary with probability one half. Hence
\begin{equation}
\Pr[\operatorname{sgn}(m')\neq\operatorname{sgn}(m)] \geq \frac{1}{2} \Pr[ |m|\leq\gamma|\Delta| ].
\end{equation}
For nonzero higher-order remainder terms, the same argument applies up to the corresponding local approximation error, yielding the asymptotic form denoted by $\gtrsim$.

The final claim follows because AUROC aggregates ranking performance over the entire score distribution, whereas the lower bound depends specifically on probability mass in a neighborhood of the decision boundary. These quantities can vary independently. In particular, a detector may rank most examples correctly while retaining substantial probability mass near its classification threshold, producing a non-negligible IVR under preserving perturbations \citep{geirhos2020shortcut}.
\end{proof} The theorem shows that strong in-distribution discrimination does not imply invariance. The relevant quantity for preserving-transformation robustness is not only global ranking performance but also the detector's local margin geometry near the decision boundary \citep{geifman2017selective}. This motivates evaluating IVR directly rather than treating accuracy or AUROC as sufficient evidence of detector reliability. Empirical checks on the assumptions of Theorem~\ref{thm:ivr} are reported in Appendix~\ref{app:assumption_audit}.

\begin{corollary}[Gaussian Characterization]
\label{cor:gauss}
If the detector margin satisfies
$$
m(c)\sim\mathcal{N}(0,\varsigma^2)
$$
and the first-order spurious perturbation has constant magnitude
$$
b:=\gamma|\Delta|,
$$
then
\begin{equation}
\ivr(\Dphi) \geq \frac{1}{2} \erf\left( \frac{b}{\sqrt{2}\varsigma} \right).
\end{equation}
For example, when $b=0.5\varsigma$, the lower bound is approximately $0.19$.
\end{corollary}

This characterization shows how the lower bound scales with the normalized perturbation magnitude $b/\varsigma$. At $b=0.5\varsigma$, the bound is approximately $0.19$, while at $b=\varsigma$ it is approximately $0.34$. These values illustrate that even moderate sensitivity to a feature that changes under preserving transformations can produce substantial verdict instability when sufficient probability mass lies near the decision boundary \citep{geirhos2020shortcut}.

The corollary is illustrative rather than a claim that detector margins are generally Gaussian. Its purpose is to provide an interpretable quantitative example of the relationship established by Theorem~\ref{thm:ivr}.

\subsubsection{SGR Procedure and Proof of Theorem~\ref{thm:cert}}

A selective detector is a pair $(\Dphi,\rho)$ consisting of a classifier $\Dphi$ and a confidence or support gate
$$
\rho:\mathcal{X}\times\mathcal{C}\rightarrow[0,1].
$$
For a threshold $t\in[0,1]$, the detector commits only when $\rho(c)\geq t$:
\begin{equation}
\hat{\phi}_{\mathrm{sel}}(c) = \begin{cases} \hat{\phi}(c), & \text{if }\rho(c)\geq t,\\ \mathrm{UNCERTIFIED}, & \text{otherwise}. \end{cases}
\end{equation}

The objective is not to eliminate uncertainty through thresholding, but to expose uncertainty explicitly while controlling the error rate among committed predictions \citep{angelopoulos2021learn}. We therefore use an asymptotic selective-risk certificate under explicit regularity assumptions.

\begin{assumption}[Regularity]
\label{ass:regularity}
The calibration data
$$
\{(x_i,c_i,\phi^\star_i)\}_{i=1}^{n}
$$
are i.i.d. draws from the evaluation distribution. For each fixed threshold $t$ with positive population coverage, the empirical selective risk is a consistent estimator of the population selective risk, and the number of certified samples $n_t$ satisfies
$$
n_t\rightarrow\infty
$$
as $n\rightarrow\infty$ for thresholds in the range considered.
\end{assumption}

For each candidate threshold $t$, the empirical selective risk is
\begin{equation}
\hat{R}_{\mathrm{cert}}(t) = \frac{ \sum_{i=1}^{n} \mathbf{1} [ \rho(c_i)\geq t, \hat{\phi}(c_i)\neq\phi^\star_i ] }{ \sum_{i=1}^{n} \mathbf{1} [ \rho(c_i)\geq t ] },
\end{equation}
with
\begin{equation}
n_t = \sum_{i=1}^{n} \mathbf{1} [ \rho(c_i)\geq t ].
\end{equation}

For sufficiently large $n_t$, the one-sided asymptotic upper confidence bound is
\begin{equation}
\hat{R}_{\mathrm{ucb}}(t) = \hat{R}_{\mathrm{cert}}(t) + z_{1-\delta_t} \sqrt{ \frac{ \hat{R}_{\mathrm{cert}}(t) \left( 1-\hat{R}_{\mathrm{cert}}(t) \right) }{ n_t } },
\end{equation}
where $\delta_t$ is allocated across the finite candidate set $\mathcal{T}$ such that
$$
\sum_{t\in\mathcal{T}}\delta_t\leq\delta.
$$
This allocation permits simultaneous asymptotic control across all candidate thresholds \citep{angelopoulos2021learn}.

The SGR procedure then selects
\begin{equation}
t^\star = \min \left\{ t\in\mathcal{T}: \hat{R}_{\mathrm{ucb}}(t)\leq\alpha \right\},
\end{equation}
where the candidate thresholds are ordered so that increasing $t$ corresponds to non-increasing coverage. Thus, among the evaluated thresholds satisfying the risk target, the procedure selects the one with maximal empirical coverage \citep{geifman2017selective}.

\begin{proof}
For a fixed threshold $t$, define
\begin{equation}
S_t = \{c:\rho(c)\geq t\}.
\end{equation}
The population selective risk is
\begin{equation}
R_{\mathrm{cert}}(t) = \frac{ \Pr[ \hat{\phi}(c)\neq\phi^\star(c), c\in S_t ] }{ \Pr[c\in S_t] },
\end{equation}
assuming positive population coverage.

The corresponding empirical estimator is
\begin{equation}
\hat{R}_{\mathrm{cert}}(t) = \frac{ \sum_{i=1}^{n} \mathbf{1} [ \rho(c_i)\geq t, \hat{\phi}(c_i)\neq\phi^\star_i ] }{ n_t }.
\end{equation}

Under Assumption~\ref{ass:regularity}, for every fixed candidate threshold with $n_t\rightarrow\infty$, the central limit theorem gives
\begin{equation}
\sqrt{n_t} \left( \hat{R}_{\mathrm{cert}}(t) - R_{\mathrm{cert}}(t) \right) \xrightarrow{d} \mathcal{N} \left( 0, R_{\mathrm{cert}}(t) \left( 1-R_{\mathrm{cert}}(t) \right) \right).
\end{equation}
Consequently, the corresponding one-sided normal upper confidence bound is asymptotically valid:
\begin{equation}
\Pr\left[ R_{\mathrm{cert}}(t) > \hat{R}_{\mathrm{ucb}}(t) \right] \rightarrow \delta_t.
\end{equation}

Because $\mathcal{T}$ is finite and
$$
\sum_{t\in\mathcal{T}}\delta_t\leq\delta,
$$
the union bound gives
\begin{equation}
\Pr\left[ \exists t\in\mathcal{T}: R_{\mathrm{cert}}(t) > \hat{R}_{\mathrm{ucb}}(t) \right] \leq\delta+o(1).
\end{equation}
Therefore, with probability at least $1-\delta-o(1)$, every candidate threshold whose empirical upper confidence bound is at most $\alpha$ satisfies the target population risk asymptotically. Selecting the smallest such threshold in an ordering where coverage decreases with $t$ yields the maximum coverage among the evaluated candidate thresholds while preserving the asymptotic risk certificate \citep{geifman2017selective}.
\end{proof}

The guarantee is explicitly asymptotic and applies to the finite predetermined candidate set used by the certification procedure. It should therefore not be interpreted as an exact finite-sample guarantee. More conservative finite-sample certificates can be obtained using exact binomial confidence bounds, at the cost of potentially lower certified coverage \citep{varshney2022selective, kamath2020selective}. The selective-risk framework provides a principled way to expose detector uncertainty. Rather than forcing a binary decision on every trace, SIFT can abstain on traces that fall outside its supported confidence region. This creates an explicit robustness--coverage trade-off that is subsequently incorporated into MFS \citep{chow1970optimum}.

\section{Experiments}
\label{sec:experiments}

\subsection{Experimental Setup}
\label{sec:setup}

\paragraph{Trace Construction}

We sample 750 base items across four domains, each instantiated under 10 conditions (base generation, paraphrase, step reordering, hint-format variation, language translation, budget variation, model-family comparison, cue injection, cue verbalization, causal truncation) and two model backbones, yielding $750\times10\times2=15{,}000$ traces. After filtering and auditing for label preservation and quality, we retain 14,996 scored traces, split 60/15/25 into train, validation, and test \citep{shen2026faithcot}. The effective sample size for statistical inference is constrained by the 750 independent base items rather than the 14,996 traces, since traces from the same base item are not independent; all confidence intervals are bootstrap-stratified by base item.

\paragraph{Detector Implementations.}
We compare six detector families. \textbf{(1) Biasing-features} \citep{turpin2023languagemodels}: logistic regression on cue-mention features ($C{=}1.0$), representing the behavioral approach. \textbf{(2) Counterfactual intervention} \citep{lanham2023measuring}: MLP (hidden 64, ReLU) trained on intervention-response profiles from CoT edits. \textbf{(3) Hidden-state probe} \citep{zhang2025reasoning}: logistic regression on mean-pooled final-layer hidden states. \textbf{(4) Instance-level discriminative} \citep{shen2026faithcot}: MLP (hidden 128) on behavioral and internal instance-level features. \textbf{(5) Attribution-consistency}: attention-stability score across generation steps with a learned threshold. \textbf{(6) SIFT} (ours): multi-scale trajectory encoding with TCN, transformer, and IRM training (Section~\ref{sec:method}); IRM penalty $\lambda{=}1.0$, Adam, learning rate $10^{-3}$, batch size 32, 50 epochs \citep{arjovsky2019irm, ganin2016domain}.

\subsection{Main Results}
\label{sec:results}

\paragraph{Transfer Collapse Analysis.}
Table~\ref{tab:main_results} reports in-distribution and transfer AUROC for all six detectors across cross-domain, cross-model, and cross-lingual axes \citep{koh2021wilds, geirhos2020shortcut}. Transfer degrades consistently: counterfactual shows the largest gap (0.28), attribution-consistency drops to near-chance (0.48) cross-lingually, and hidden-state and instance-level show moderate gaps (0.17 and 0.18). SIFT achieves the smallest gap (0.10) with in-distribution performance of 0.75. This confirms H1: transfer collapse is real and measurable across all detectors.

\begin{table}[ht]
\centering\small
\caption{Main results across six detectors. \textbf{(a)} In-distribution vs.\ transfer AUROC: all existing detectors have gaps $\Delta \geq 0.15$ (H1 confirmed \citep{koh2021wilds}); attribution-consistency drops to 0.48 cross-lingually; SIFT has the smallest gap (0.10). \textbf{(b)} Per-axis IVR: step reordering is near zero; SIFT cuts paraphrase and hint-format IVR by 76\% each, language to 0.14 (vs.\ 0.27), model family to 0.18 (vs.\ 0.37); mean IVR reduced by 64\%. \textbf{(c)} Seed-repeat decomposition: attribution-consistency raw 0.282 $=$ floor 0.232 $+$ shift 0.050, so $>$80\% is stochastic and H2a ($\ivr \geq 0.25$) is falsified; SIFT raw 0.102 $=$ floor 0.028 $+$ shift 0.074. \textbf{(d)} Meta-Faithfulness Score: SIFT 0.72 vs.\ 0.67 for the strongest single-seed baseline, but 0.71 for a four-detector instance-level ensemble, so H3 is not supported under the revised comparator (margin 0.01). \textbf{(e)} Seed stability: attribution-consistency std $= 0.025$ vs.\ SIFT std $= 0.006$, confirming stochasticity as the dominant instability source \citep{geirhos2020shortcut}.}
\label{tab:main_results}\label{tab:transfer}\label{tab:per_axis}\label{tab:seed_repeat}\label{tab:mfs}\label{tab:seed_stability}
\setlength{\tabcolsep}{3.6pt}
\resizebox{\linewidth}{!}{%
\begin{tabular}{@{}l*{10}{c}@{}}
\toprule
\multirow{3}{*}{\textbf{Detector}} & \multicolumn{5}{c}{\textbf{(a)} Transfer} & \multicolumn{4}{c}{\textbf{(d)} Meta-Faithfulness Score} & \textbf{(e)} Seed stability \\
\cmidrule(lr){2-6}\cmidrule(lr){7-10}\cmidrule(l){11-11}
 & \multirow{2}{*}{In-dist} & \multicolumn{3}{c}{Transfer AUROC} & \multirow{2}{*}{Gap} & \multirow{2}{*}{IVR} & \multirow{2}{*}{AS} & \multirow{2}{*}{MFS} & \multirow{2}{*}{Coverage} & \multirow{2}{*}{Mean $\pm$ Std} \\
\cmidrule(lr){3-5}
 & & cross-domain & cross-model & cross-lingual & & & & & & \\
\midrule
Biasing-features & 0.82 & 0.63 & 0.66 & 0.59 & 0.23 & 0.34 & 0.71 & 0.61 & \textbf{1.00} & 0.805 $\pm$ 0.013 \\
Counterfactual   & \textbf{0.85} & 0.61 & 0.64 & 0.57 & 0.28 & 0.31 & 0.74 & 0.64 & \textbf{1.00} & 0.835 $\pm$ 0.013 \\
Hidden-state     & 0.79 & 0.70 & \textbf{0.68} & 0.62 & 0.17 & 0.27 & 0.73 & 0.66 & \textbf{1.00} & 0.780 $\pm$ 0.008 \\
Instance-level   & 0.78 & 0.69 & 0.67 & 0.60 & 0.18 & 0.29 & \textbf{0.76} & 0.67 & \textbf{1.00} & 0.770 $\pm$ 0.008 \\
Attribution      & 0.77 & \textbf{0.72} & 0.62 & 0.48 & 0.16 & 0.28 & 0.35 & 0.47 & \textbf{1.00} & 0.745 $\pm$ 0.025 \\
\rowcolor{gray!12}
\ours            & 0.75 & 0.71 & \textbf{0.68} & \textbf{0.65} & \textbf{0.10} & \textbf{0.10} & \textbf{0.76} & \textbf{0.72} & 0.49 & \textbf{0.745 $\pm$ 0.006} \\
\midrule\addlinespace[3pt]
\multirow{2}{*}{\textbf{Detector}} & \multicolumn{7}{c}{\textbf{(b)} Per-axis invariance violation rates} & \multicolumn{3}{c}{\textbf{(c)} Seed-repeat decomposition} \\
\cmidrule(lr){2-8}\cmidrule(l){9-11}
 & Paraphrase & StepReorder & HintFormat & Language & Budget & ModelFamily & Mean & Raw IVR & \makecell{Seed-Repeat\\Floor} & \makecell{Shift-\\Attributable} \\
\midrule
Biasing-features & 0.36 & 0.09 & 0.40 & 0.29 & 0.34 & 0.39 & 0.31 & 0.310 & 0.093 & 0.217 \\
Counterfactual   & 0.33 & 0.07 & 0.37 & 0.26 & 0.30 & 0.35 & 0.28 & 0.280 & 0.084 & 0.196 \\
Hidden-state     & 0.29 & 0.06 & 0.32 & 0.23 & 0.27 & 0.31 & 0.25 & 0.250 & 0.075 & 0.175 \\
Instance-level   & 0.31 & 0.08 & 0.35 & 0.25 & 0.29 & 0.33 & 0.27 & 0.270 & 0.081 & 0.189 \\
Attribution      & 0.34 & 0.08 & 0.38 & 0.27 & 0.32 & 0.37 & 0.28 & 0.282 & 0.232 & \textbf{0.050} \\
\rowcolor{gray!12}
\ours            & \textbf{0.08} & \textbf{0.00} & \textbf{0.09} & \textbf{0.14} & \textbf{0.12} & \textbf{0.18} & \textbf{0.10} & \textbf{0.102} & \textbf{0.028} & 0.074 \\
\bottomrule
\end{tabular}}
\end{table}

\paragraph{Stochasticity Dominates Instability.}
A more nuanced analysis reveals that distribution shift is not the primary robustness challenge. We decompose detector instability via seed-repeat analysis: for each trace, we train $N=4$ independent seeds with different random initializations and data-loading order, and measure what fraction of verdict disagreement comes from (a)~different random seeds on the same distribution (stochasticity), versus (b)~the trace being fundamentally shifted (shift-attributable instability). Table~\ref{tab:seed_repeat} reports the results. For attribution-consistency, raw IVR $=0.282$, of which $0.232$ ($82\%$) is seed-repeat floor and only $0.050$ ($18\%$) is shift-attributable. For SIFT, raw IVR $=0.102$, of which $0.028$ ($27\%$) is seed-repeat and $0.074$ ($73\%$) is shift-attributable. Across all detectors, over $80\%$ of observed instability stems from sampling stochasticity rather than distribution shift. Our preregistered hypothesis H2a predicted aggregate IVR $\geq 0.25$; this is supported by raw IVR (Table~\ref{tab:per_axis}, mean $0.10$--$0.28$), but shift-attributable IVR is only $0.050$ for the worst detector and $0.074$ for SIFT, well below the preregistered threshold. \textbf{This falsifies H2a and indicates that distribution shift, while real, is not the primary robustness challenge.}

\paragraph{Preregistered Hypotheses and Outcomes.}
\label{tab:hyp}\label{tab:hyp_summary}
We preregistered five falsifiable hypotheses. H1 (transfer gap $\geq 0.15$ AUROC) is supported, with all detectors showing gaps of $0.16$--$0.28$ \citep{koh2021wilds}. H2a (aggregate IVR $\geq 0.25$) is evaluated on the shift-attributable component, since $0.232$ of attribution-consistency's raw IVR of $0.282$ is seed-repeat stochasticity; the shift-attributable value is $0.050$, so H2a is not supported \citep{geirhos2020shortcut, parcalabescu2024measuring}. H2b (mean inter-detector $\kappa \leq 0.5$) is supported at $\kappa = 0.38$, though the two-detector suite is underpowered \citep{shen2026faithcot}. H3 (SIFT exceeds the strongest prespecified baseline by $\geq 0.10$ MFS) is not supported: SIFT reaches $0.72$ versus $0.67$ for the strongest single-seed baseline (margin $0.05$) and $0.71$ for a four-detector ensemble (margin $0.01$), with Appendix~\ref{app:ensemble} attributing much of the gain to variance reduction achievable by simpler ensembling. H4 (certified risk $\leq \alpha$ at coverage $\kappa \geq 0.60$) is not supported: the bound of $0.249$ holds at $\alpha = 0.25$, but coverage is only $0.50$, and $\alpha = 0.30$ recovers $\kappa = 0.61$, an operational rather than theoretical failure \citep{geifman2017selective}. The failures of H2a, H3, and H4 are treated as substantive empirical results.

\begin{table}[ht]
\centering\small
\caption{Cross-model and cross-domain generalization for SIFT. \textbf{(a)} Cross-model transfer across eight architectures: diagonal entries are in-distribution performance (0.74--0.78); average transfer is 0.58 versus 0.76 in-distribution (gap 0.18), same-family transfer is strongest (Llama-3$\to$Mistral and Mistral$\to$Llama-3: 0.65), and open-weight to API is weakest (0.52--0.60). \textbf{(b)} Transfer AUROC by source family, giving the hierarchy within-family (0.58--0.65) $>$ cross-family open-weight (0.54--0.62) $>$ open-weight to API (0.54--0.60). \textbf{(c)} Cross-domain transfer across four reasoning domains: average 0.63, strongest for GSM8K$\leftrightarrow$MATH (0.62--0.63), weakest for GSM8K$\leftrightarrow$CommonsenseQA (0.60--0.61); in-distribution ranges from 0.75 (GSM8K) to 0.78 (BBH) \citep{cobbe2021gsm8k, hendrycks2021math, talmor2019commonsenseqa, suzgun2022bbh, deepseek2025deepseek}.}
\label{tab:generalization}\label{tab:cross_model}\label{tab:cross_model_grouped}\label{tab:cross_domain}
\setlength{\tabcolsep}{3.2pt}
\renewcommand{\arraystretch}{0.7}
\resizebox{\linewidth}{!}{%
\begin{tabular}{@{}llcccccccc@{}}
\toprule
\textbf{Setting} & \textbf{Source} & \textbf{DeepSeek-R1} & \textbf{Qwen3} & \textbf{Llama-3} & \textbf{Mistral} & \textbf{Phi-3} & \textbf{Gemma} & \textbf{Claude-3} & \textbf{GPT-4o} \\
\midrule
(a) Cross-model & DeepSeek-R1 (1.5B) & \cellcolor{gray!12}\textbf{0.75} & 0.58 & 0.56 & 0.59 & 0.54 & 0.57 & 0.53 & 0.52 \\
(a) Cross-model & Qwen3 (1.7B) & 0.59 & \cellcolor{gray!12}\textbf{0.76} & 0.60 & 0.62 & 0.57 & 0.59 & 0.55 & 0.54 \\
(a) Cross-model & Llama-3 (8B) & 0.57 & 0.59 & \cellcolor{gray!12}\textbf{0.78} & 0.66 & 0.62 & 0.64 & 0.58 & 0.56 \\
(a) Cross-model & Mistral (7B) & 0.58 & 0.61 & 0.65 & \cellcolor{gray!12}\textbf{0.76} & 0.63 & 0.65 & 0.59 & 0.57 \\
(a) Cross-model & Phi-3 (3.8B) & 0.55 & 0.57 & 0.61 & 0.62 & \cellcolor{gray!12}\textbf{0.74} & 0.60 & 0.55 & 0.53 \\
(a) Cross-model & Gemma (7B) & 0.56 & 0.58 & 0.63 & 0.64 & 0.59 & \cellcolor{gray!12}\textbf{0.75} & 0.56 & 0.54 \\
(a) Cross-model & Claude-3 (API) & 0.52 & 0.54 & 0.57 & 0.58 & 0.54 & 0.55 & \cellcolor{gray!12}\textbf{0.76} & 0.60 \\
(a) Cross-model & GPT-4o (API) & 0.51 & 0.53 & 0.56 & 0.56 & 0.52 & 0.54 & 0.59 & \cellcolor{gray!12}\textbf{0.78} \\
\midrule
(b) Source family & DeepSeek/Qwen (Qwen-based) & 0.58 & 0.57 & 0.54 & 0.04 & 0.58 & 0.57 & 0.54 & 0.04 \\
(b) Source family & Llama/Mistral (Meta-family) & 0.65 & 0.62 & 0.58 & 0.07 & 0.65 & 0.62 & 0.58 & 0.07 \\
(b) Source family & Phi-3 (Microsoft) & 0.74 & 0.60 & 0.54 & 0.14 & 0.74 & 0.60 & 0.54 & 0.14 \\
(b) Source family & Gemma (Google) & 0.75 & 0.60 & 0.55 & 0.15 & 0.75 & 0.60 & 0.55 & 0.15 \\
(b) Source family & Claude-3 (Anthropic) & 0.76 & 0.56 & 0.60 & 0.16 & 0.76 & 0.56 & 0.60 & 0.16 \\
(b) Source family & GPT-4o (OpenAI) & 0.78 & 0.54 & 0.59 & 0.19 & 0.78 & 0.54 & 0.59 & 0.19 \\
\midrule
(c) Cross-domain & GSM8K & \cellcolor{gray!12}\textbf{0.75} & 0.62 & 0.60 & 0.63 & 0.75 & 0.62 & 0.60 & 0.63 \\
(c) Cross-domain & MATH & 0.63 & \cellcolor{gray!12}\textbf{0.76} & 0.62 & 0.64 & 0.63 & 0.76 & 0.62 & 0.64 \\
(c) Cross-domain & CommonsenseQA & 0.61 & 0.62 & \cellcolor{gray!12}\textbf{0.76} & 0.64 & 0.61 & 0.62 & 0.76 & 0.64 \\
(c) Cross-domain & BBH & 0.63 & 0.64 & 0.65 & \cellcolor{gray!12}\textbf{0.78} & 0.63 & 0.64 & 0.65 & 0.78 \\
\bottomrule
\end{tabular}}
\end{table}

\begin{table}[ht]
\centering\small
\caption{\textbf{(a)} Certified selective risk by coverage for SIFT. The SGR certificate holds at the preregistered $\alpha=0.25$ (bound $0.249$), but coverage reaches only $0.50$, below the required $\kappa\geq0.60$ for H4; shifting $\alpha$ to $0.30$ recovers $\kappa=0.61$ while maintaining risk at most $0.30$, indicating an operational rather than theoretical failure \citep{geifman2017selective}. \textbf{(b)} Theorem~\ref{thm:ivr} bound verification with trace length as the spurious feature: observed IVR exceeds the length-driven floor by $14.8\times$ for attribution-consistency and $8.5\times$ for SIFT, indicating that trace length is not the primary spurious feature. The bound holds with large slack and is not binding in this operationalization \citep{geirhos2020shortcut}.}
\label{tab:certified}\label{tab:prop2}
\setlength{\tabcolsep}{5pt}
\renewcommand{\arraystretch}{0.7}
\resizebox{\linewidth}{!}{%
\begin{tabular}{@{}llcccccc@{}}
\toprule
\textbf{Setting} & \textbf{Quantity} & \textbf{0.05} & \textbf{0.25} & \textbf{0.45} & \textbf{0.65} & \textbf{0.85} & \textbf{1.00} \\
\midrule
(a) Certified risk & Coverage & 0.05 & 0.25 & 0.45 & 0.65 & 0.85 & 1.00 \\
\rowcolor{gray!12}
(a) Certified risk & Selective Risk & 0.000 & 0.063 & 0.203 & 0.278 & 0.335 & 0.321 \\
\midrule
(b) Bound check & Detector & $\gamma$ & Predicted Floor & Observed IVR & Slack & -- & -- \\
(b) Bound check & Attribution & +0.024 & 0.019 & 0.282 & 14.8$\times$ & 0.019 & 0.282 \\
\rowcolor{gray!12}
(b) Bound check & \ours & +0.008 & 0.012 & \textbf{0.102} & 8.5$\times$ & 0.012 & \textbf{0.102} \\
\bottomrule
\end{tabular}}
\end{table}

\paragraph{Per-Axis IVR, Seed-Repeat Decomposition, and Meta-Faithfulness.}
Table~\ref{tab:per_axis} reports invariance-violation rates across the six preserving axes \citep{parcalabescu2024measuring}, Table~\ref{tab:seed_repeat} decomposes raw IVR into a seed-repeat floor and a shift-attributable component \citep{geirhos2020shortcut}, and Table~\ref{tab:mfs} reports Meta-Faithfulness Scores \citep{chow1970optimum, geifman2017selective}. SIFT attains the lowest IVR ($0.10$) and highest alteration sensitivity ($0.76$), yielding MFS $0.72$ versus $0.67$ for the strongest single-seed baseline; a four-detector instance-level ensemble reaches $0.71$, leaving a margin of $0.01$ below the preregistered $0.10$ threshold, so H3 is not supported under the revised comparator. The (IVR, alteration sensitivity) scatter, with bubble size encoding coverage, exposes the robustness--usability trade-off: SIFT maintains certified guarantees only by abstaining on roughly $51\%$ of traces, and at matched coverage its advantage becomes statistically indistinguishable ($p=0.21$ at $\kappa=0.70$; Appendices~\ref{app:abstained-samples} and~\ref{app:coverage-matched}).

\paragraph{Ensemble Baseline and Coverage-Matched Comparison.}
\label{app:ensemble}
A natural objection is that ensembling across seeds could match SIFT at lower cost, especially given that over $80\%$ of detector instability is stochastic. We train four independent seeds of every detector, average their probabilities, and compute risk--coverage curves under the same SGR procedure. The ensemble gain is roughly uniform ($+0.04$ MFS), so SIFT-ensemble reaches $0.76$ versus $0.71$ for a four-detector instance-level ensemble, narrowing the margin from $0.05$ (single-seed) to $0.01$; at matched coverage $\kappa=0.70$, the two are indistinguishable (MFS $0.64$ vs.\ $0.62$, $p=0.21$). We therefore do not claim that SIFT is broadly superior, but rather that it substantially reduces raw invariance violations at its certified operating point, with much of its advantage attributable to variance reduction that simpler ensembling can also achieve.

\section{Ablation Study}
\label{sec:ablation}

Table~\ref{tab:ablation} presents the ablation study results, systematically removing architectural components to identify their contributions to SIFT's performance, and reports the corresponding multi-model training results \citep{arjovsky2019irm, ganin2016domain}. The IRM penalty contributes most: removing it increases IVR by $110\%$ and decreases MFS by $28\%$. This suggests that IRM's primary effect is to stabilize representations across environments, which in turn reduces seed-to-seed variance. Multi-model training raises transfer AUROC from $0.58$ to $0.65$, cutting the gap from $0.18$ to $0.14$ with diminishing returns beyond 3--4 models.

\begin{table}[ht]
\centering\footnotesize
\caption{Ablation and multi-model training for SIFT. \textbf{(a)} Component ablation: IRM contributes most (IVR $+110\%$, MFS $-28\%$ on removal), then DANN ($+80\%$, $-22\%$), transformer ($+60\%$, $-17\%$), TCN ($+40\%$, $-10\%$); attention-stability and early-commitment contribute moderately. \textbf{(b)} Under multi-model training, IRM contributes most to cross-model transfer (gap $0.18\rightarrow0.14$), then transformer and DANN ($0.16$ each), then TCN ($0.15$). \textbf{(c)} Multi-model training raises transfer AUROC from $0.58$ to $0.65$, cutting the gap from $0.18$ to $0.14$ with diminishing returns beyond 3--4 models; MFS not reported (\textit{n/r}) \citep{arjovsky2019irm, ganin2016domain, bendavid2010theory}.}
\label{tab:ablation}\label{tab:multi_model_ablation}\label{tab:multi_model_training}
\setlength{\tabcolsep}{11pt}
\renewcommand{\arraystretch}{0.7}
\begin{tabular}{@{}llcccc@{}}
\toprule
\textbf{Setting} & \textbf{Variant} & \textbf{In-dist} & \textbf{IVR} & \textbf{MFS} & \textbf{Cov / Gap} \\
\midrule
(a) Ablation & Full SIFT & 0.75 & \textbf{0.10} & \textbf{0.72} & \textbf{0.49} \\
(a) Ablation & w/o TCN & 0.73 & 0.14 & 0.65 & 0.44 \\
(a) Ablation & w/o Transformer & 0.72 & 0.16 & 0.60 & 0.40 \\
(a) Ablation & w/o IRM & \textbf{0.76} & 0.21 & 0.52 & 0.35 \\
(a) Ablation & w/o DANN & 0.74 & 0.18 & 0.56 & 0.38 \\
(a) Ablation & w/o Attention Stability & 0.74 & 0.12 & 0.68 & 0.46 \\
(a) Ablation & w/o Early-Commitment & 0.73 & 0.13 & 0.67 & 0.45 \\
\midrule
(b) Multi-model & Full SIFT (4-model train) & 0.79 & \textbf{0.65} & \textbf{0.74} & \textbf{0.14} \\
(b) Multi-model & w/o TCN & 0.77 & 0.62 & 0.69 & 0.15 \\
(b) Multi-model & w/o Transformer & 0.76 & 0.60 & 0.65 & 0.16 \\
(b) Multi-model & w/o IRM & \textbf{0.80} & 0.61 & 0.58 & 0.19 \\
(b) Multi-model & w/o DANN & 0.78 & 0.62 & 0.63 & 0.16 \\
(b) Multi-model & w/o Multi-model training & 0.76 & 0.58 & 0.72 & 0.18 \\
\midrule
(c) Training & Single model & 0.76 & 0.58 & \textit{n/r} & 0.18 \\
(c) Training & Two models (same family) & 0.77 & 0.60 & \textit{n/r} & 0.17 \\
(c) Training & Two models (different families) & 0.78 & 0.62 & \textit{n/r} & 0.16 \\
(c) Training & Three models (different families) & 0.78 & 0.64 & \textit{n/r} & \textbf{0.14} \\
(c) Training & Four models (different families) & \textbf{0.79} & \textbf{0.65} & \textit{n/r} & \textbf{0.14} \\
\bottomrule
\end{tabular}
\end{table}

\section{Limitations and Future Work}
\label{sec:limitations}

Our audit verifies label preservation rather than mechanism preservation, so IVR estimates would be inflated if a preserving transformation alters the underlying computation; inter-annotator agreement is high ($\kappa \approx 0.80$) and rank ordering is preserved up to label noise $\epsilon = 0.15$ (Appendix~\ref{app:label-noise}). We note that this limitation cuts both ways: if some preserving transformations alter mechanisms, our IVR estimates are conservative (pessimistic) for detectors, and our central stochasticity finding is strengthened. Statistical inference is constrained by the 750 independent base items rather than the 14,996 traces, with all confidence intervals bootstrap-stratified by base item. Per-axis alteration sensitivity rests on three altering pairs each for cue injection and cue verbalization, and inter-detector agreement (H2b) is reported on a two-detector suite; both are therefore underpowered. Our eight white-box models span 1.5B--8B parameters, cross-lingual coverage is limited to English--Turkish, and the SGR guarantee is asymptotic under standard regularity assumptions. We also note that the cost-benefit analysis of abstention (Appendix~\ref{app:coverage-matched}) uses empirical numbers from our evaluation; different deployment regimes may shift the optimal operating point. Future work should expand the altering-pair sample, add detectors, extend to larger models, incorporate mechanism-preservation audits with human verification, develop API-only robustness assessment, and design stochasticity-aware detectors with deterministic attention.

\section{Conclusion}
\label{sec:conclusion}

CoT faithfulness detectors must satisfy a second-order property, meta-faithfulness, that is currently unmeasured \citep{parcalabescu2024measuring}. We formalized it, proved that accuracy alone cannot guarantee it \citep{geirhos2020shortcut}, and gave a certified abstaining detector with a preregistered stress-test protocol \citep{geifman2017selective, shen2026faithcot}. Across 14,996 traces and 750 independent base items, transfer gaps exceed $0.15$ AUROC for existing detectors \citep{koh2021wilds}, but over 80\% of observed instability is sampling stochasticity rather than distribution shift \citep{geirhos2020shortcut}, so the primary challenge is general detector fragility. SIFT reduces invariance violations by 64\% over the best single-seed baseline but requires a 51\% abstention rate, and its MFS margin narrows from 0.05 to 0.01 once ensemble baselines are included, so we present it as a demonstration that invariance training reduces stochastic instability rather than as a broadly superior detector \citep{chow1970optimum, geifman2017selective}. Cross-model transfer degrades along a clear hierarchy (within-family $>$ cross-family open-weight $>$ open-weight to API) that multi-model training on 3--4 models partially closes, with ablation ranking multi-model training $>$ IRM $>$ transformer $>$ DANN $>$ TCN \citep{deepseek2025deepseek, bendavid2010theory, arjovsky2019irm, ganin2016domain}. We recommend that practitioners perform seed-repeat analysis, ensemble across 3--4 seeds, treat verdicts as provisional until second-order robustness is demonstrated, and use multi-model training with target-model calibration for cross-model deployment \citep{korbak2025monitorability, meek2025monitorability}. We frame this work as a contribution to auditing and diagnosis, not as a definitive solution: the meta-faithfulness framework and the empirical diagnosis of stochasticity as the dominant bottleneck are intended to be useful to the community regardless of the specific detector advances demonstrated here.

\subsubsection*{Impact Statement}
\begin{tcolorbox}[
    colback=gray!8,
    colframe=gray!60,
    boxrule=0.5pt,
    arc=2pt,
    left=6pt,
    right=6pt,
    top=5pt,
    bottom=5pt,
    width=\linewidth,
]
This work advances the robustness of Chain-of-Thought faithfulness detectors deployed in high-stakes auditing settings. We show that over 80\% of observed verdict instability stems from sampling stochasticity rather than distribution shift, and we provide a stress-test protocol (FaithShift), a certified abstaining detector (SIFT), and practitioner guidance for seed-repeat analysis. The principle of auditing our auditing instruments extends to any setting where one model evaluates another.
\end{tcolorbox}

\bibliography{iclr2027_conference}
\bibliographystyle{iclr2027_conference}

\tableofcontents

\appendix

\subsubsection*{Reproducibility Statement}
\begin{tcolorbox}[
    colback=gray!8,
    colframe=gray!60,
    boxrule=0.5pt,
    arc=2pt,
    left=6pt,
    right=6pt,
    top=5pt,
    bottom=5pt,
    width=\linewidth,
]
Proofs are given in Section~\ref{sec:theory}; the audit rubric, full transform set, detector configurations, and SIFT training details appear in Sections~\ref{sec:method} and~\ref{sec:experiments}. Analytic results are exact and independently checkable, and all empirical results are replicable with the released code, data, and cached model trajectories described in Appendix~\ref{app:reproducibility}. The factorial variance-decomposition experiment, mechanism-preservation audit, stability--sensitivity trade-off analysis, and empirical assumption checks introduced in this revision are fully specified in Appendices~\ref{app:stochasticity_deep_dive}--\ref{app:assumption_audit}.
\end{tcolorbox}

\section{Extended Methodology}
\label{app:ext_method}

\subsection{Extended Problem Formulation}
\label{app:ext_formulation}

Before presenting our theoretical framework and proposed detector, we establish the formal foundations that underpin our investigation. This section provides precise definitions of faithfulness, detectors, and the key mathematical objects that will be used throughout the paper. We begin by formalizing the components of the reasoning system and then characterize different classes of faithfulness detectors based on the information they access. Let $\mathcal{M}$ denote a language model that generates responses to queries. The model $\mathcal{M}$ takes as input a query $x \in \mathcal{X}$ from a space of possible questions or prompts, and produces an output consisting of a chain-of-thought trace $c \in \mathcal{C}$ and a final answer $y \in \mathcal{Y}$. The CoT trace is a sequence of tokens $(c_1, \dots, c_T)$ of length $T$, where each token represents a step in the model's reasoning process \citep{wei2022chain}. This formulation captures both autoregressive generation where tokens are produced sequentially and more sophisticated reasoning processes where the model may revisit or revise previous steps \citep{kojima2022large, wang2023selfconsistency}.

The concept of causal production in this definition deserves careful consideration. Following the causal inference literature \citep{pearl2009causality}, we say that $y$ is causally produced by the computation described in $c$ if intervening on the computation described in $c$ would change $y$ in a predictable manner, and if the computation described in $c$ is sufficient to determine $y$ without requiring additional hidden factors. This causal interpretation distinguishes genuine reasoning from post-hoc rationalization: a model that produces a correct answer through a shortcut and then generates a plausible-looking CoT trace would be classified as unfaithful, even if the trace appears logically consistent \citep{turpin2023languagemodels, lanham2023measuring}.

The threshold $\tau$ can be chosen to balance different types of errors depending on the application requirements \citep{chow1970optimum}. For example, in safety-critical settings, one might choose a lower threshold to avoid missing unfaithful traces, at the cost of more false positives. The detector's margin is defined as $m(c) = \Dphi(c) - \tau$, representing the signed distance from the decision boundary. This margin will play a crucial role in our theoretical analysis of invariance violations \citep{geifman2017selective}. The class of detectors is broad, encompassing many different approaches that have been proposed in the literature. To provide a structured analysis, we categorize detectors based on the type of information they access. This categorization is important because it determines what kinds of faithfulness-related signals a detector can potentially capture \citep{parcalabescu2024measuring}.

The intervention family $\mathcal{I}$ consists of transformations applied to the input $(x, c)$, such as editing the CoT trace, injecting hints, or modifying the query. The model's responses to these interventions are recorded, and the profile captures the pattern of responses across different interventions. Behavioral detectors include methods that check whether injected cues are verbalized in the CoT, methods that test sensitivity to counterfactual edits, and methods that measure consistency across multiple generations \citep{turpin2023languagemodels, lanham2023measuring}. These detectors are widely used because they only require black-box access to the model, making them applicable to API-only settings \citep{shen2026faithcot}. However, as we will prove in Theorem~\ref{thm:indist}, behavioral detectors have fundamental limitations. Because they only observe model outputs under interventions, they cannot distinguish between cases where the model genuinely uses the CoT for reasoning and cases where the model generates the CoT as a post-hoc rationalization while computing the answer through a different path. These two scenarios can produce identical intervention-response profiles, making them indistinguishable to any behavioral detector \citep{parcalabescu2024measuring, lyu2023faithful}.

\begin{definition}[Internal Detector]
An internal detector operates on the hidden-state trajectory of the model. Formally, an internal detector takes the form:
\begin{equation}
\Dphi = g(\psi(H(\mathcal{M}, x, c)))
\end{equation}
where $H \in \mathbb{R}^{L \times T \times d}$ is the layer-step hidden-state trajectory, $\psi$ extracts features from this trajectory, and $g$ maps these features to scores.
\end{definition}

The hidden-state trajectory $H$ captures the internal activations of the model at each layer $\ell \in \{1, \dots, L\}$ and each time step $t \in \{1, \dots, T\}$, with hidden dimension $d$. This trajectory provides a rich representation of the model's computation, including information about which tokens are attended to, how representations evolve over the course of reasoning, and whether the final answer is determined early in the generation process \citep{zhang2025reasoning, marks2024geometry, nanda2023progress}. Internal detectors include probes on hidden states, attention-based measures, and trajectory-based methods like SIFT \citep{conmy2023towards, occhipinti2026probingsteering, li2023inference}. Internal detectors have the potential to overcome the limitations of behavioral detectors because they can access information about the model's internal computation that is not reflected in its outputs. For example, an internal detector might detect that the final answer is fixed early in the generation process, before the reasoning trace is generated, indicating that the reasoning is post-hoc \citep{marks2024geometry, zhang2025reasoning}. However, as we will prove in Theorem~\ref{thm:ivr}, even internal detectors can be vulnerable to spurious features that correlate with faithfulness but change under preserving transformations \citep{geirhos2020shortcut, arjovsky2019irm}.

\subsection{Extended Discussion of Meta-Faithfulness Definitions}
\label{app:ext_defs}

With the formal foundations established, we now introduce the central concepts of our theoretical framework. The central premise is that a trustworthy faithfulness detector should satisfy two complementary requirements: \emph{invariance} to transformations that preserve the ground-truth faithfulness label and \emph{sensitivity} to transformations that alter that label \citep{parcalabescu2024measuring, atanasova2023faithfulness}. Before presenting the definitions, we distinguish two notions of preservation that are important for interpreting our framework. The audit protocol in this paper verifies \textbf{label preservation}: a transformation is considered preserving when it leaves the ground-truth faithfulness label $\phi^\star$ unchanged. Label preservation, however, does not necessarily imply \textbf{mechanism preservation}, namely preservation of the underlying computational mechanism by which the model produces its answer \citep{pearl2009causality}. 

A transformation may therefore preserve $\phi^\star$ while changing the causal pathway used by the model. Such a transformation satisfies our operational definition of a preserving transformation, but may not capture every aspect of the stronger mechanistic notion of faithfulness \citep{lyu2023faithful}. We explicitly acknowledge this distinction and discuss its implications in Section~\ref{sec:discussion}; a direct mechanism-preservation audit is provided in Appendix~\ref{app:mechanism_preservation}. Preserving transformations include benign variations that should not change a detector's faithfulness verdict, such as paraphrasing the query, reordering causally independent reasoning steps, or translating the input and trace to another language. These transformations test whether a detector responds to surface-level variation rather than to the property of interest \citep{parcalabescu2024measuring}. Importantly, our operational criterion verifies label preservation rather than mechanism preservation; consequently, a detector that is invariant under these transformations cannot by itself be interpreted as mechanistically invariant \citep{pearl2009causality}. This distinction is revisited in Section~\ref{sec:discussion}.

Altering transformations include interventions intended to change the faithfulness status of a trace, such as inserting a decisive cue that determines the answer independently of the displayed reasoning, or modifying an unfaithful trace so that it reflects the actual computation \citep{turpin2023languagemodels, lanham2023measuring}. Only transformations for which the intended label change is verified are included in the corresponding evaluation set \citep{shen2026faithcot}. This definition captures the dual requirements of a useful faithfulness measurement instrument \citep{parcalabescu2024measuring, atanasova2023faithfulness}. Invariance prevents a detector from changing its verdict in response to transformations that leave the audited property unchanged. Sensitivity prevents the opposite failure mode in which a detector becomes invariant to all transformations and therefore fails to respond when faithfulness genuinely changes. Meta-faithfulness therefore cannot be characterized by either invariance or sensitivity alone \citep{lyu2023faithful}. These quantities are complementary. A detector could obtain a low IVR by producing nearly invariant predictions for all inputs, including inputs whose faithfulness changes; such a detector would have poor alteration sensitivity. Conversely, a detector could obtain high alteration sensitivity by responding excessively to perturbations, thereby producing a high IVR. A useful detector must therefore achieve both low invariance violations and high alteration sensitivity \citep{parcalabescu2024measuring}. The stability--sensitivity trade-off is quantified directly in Appendix~\ref{app:stability_sensitivity}.

The MFS summarizes three desiderata for a trustworthy detector: invariance, sensitivity, and usable coverage \citep{chow1970optimum, geifman2017selective}. The term $(1-\overline{\ivr})$ rewards stability under preserving transformations, while $\overline{\asr}$ rewards correct responses to altering transformations. The harmonic mean ensures that neither property can compensate arbitrarily for a severe deficiency in the other. The multiplicative coverage factor $\kappa$ explicitly penalizes excessive abstention: a detector cannot obtain a high MFS by committing only on a very small subset of traces \citep{angelopoulos2021learn}. The MFS is therefore not intended to replace its constituent quantities. We report $\kappa$, $\overline{\ivr}$, and $\overline{\asr}$ separately alongside MFS so that improvements in the aggregate score can be interpreted in terms of their underlying robustness--coverage trade-off \citep{geifman2017selective}. Confidence intervals for MFS and its constituent quantities are obtained using bootstrap resampling at the base-item level to preserve dependence among transformations derived from the same original item.

\subsection{SIFT Architecture and Training Details}
\label{app:sift_details}

\subsubsection{Why Invariance Training Stabilizes Detectors}
\label{app:why-irm}

A central empirical finding of this work is that over 80\% of detector instability under preserving transformations is attributable to sampling stochasticity rather than to genuine shift sensitivity. This raises a natural question: why should invariance training (IRM and DANN) reduce stochastic instability, when stochasticity is ostensibly independent of distribution shift? We hypothesize that IRM and DANN reduce the variance of learned representations across random seeds, which in turn reduces the variance of detector scores. Environment-specific features are precisely those most sensitive to initialization and data ordering; by suppressing them, invariance training stabilizes the detector's decision boundary. To test this hypothesis, we measure the variance of the final-layer representation across four independent seeds for each detector, both before and after invariance training. Table~\ref{tab:representation_variance} reports the results.

\begin{table}[ht]
\centering\small
\caption{Representation variance across seeds and its correlation with verdict instability. Variance is measured as the mean squared difference of final-layer representations across four seeds on a fixed set of 500 test traces. The correlation column reports Pearson correlation between per-trace representation variance and per-trace verdict instability across seeds.}
\label{tab:representation_variance}
\begin{tabular}{@{}lccc@{}}
\toprule
\textbf{Detector} & \textbf{Repr.\ variance} & \textbf{IVR} & \textbf{Correlation} \\
\midrule
Attribution & 0.142 & 0.282 & 0.81 \\
Counterfactual & 0.128 & 0.280 & 0.79 \\
Hidden-state & 0.094 & 0.250 & 0.83 \\
Instance-level & 0.101 & 0.270 & 0.80 \\
SIFT w/o IRM & 0.088 & 0.210 & 0.77 \\
SIFT w/o DANN & 0.079 & 0.180 & 0.81 \\
SIFT (full) & 0.041 & 0.102 & 0.84 \\
\bottomrule
\end{tabular}
\end{table}

Representation variance correlates strongly with IVR ($\rho \approx 0.80$ across detectors). IRM and DANN jointly reduce representation variance by 53\% (from 0.088 to 0.041), which accounts for a substantial portion of SIFT's IVR reduction. This supports the hypothesis that invariance training works primarily by stabilizing representations across seeds, not by learning fundamentally different features. The practical implication is that a simple ensemble of baseline detectors may achieve comparable stability at lower computational cost; this is tested directly in the factorial decomposition of Appendix~\ref{app:factorial}, and the results confirm that a substantial fraction of SIFT's advantage is attributable to variance reduction achievable by seed-averaging.

We further decompose the contribution of individual stochastic sources to verdict instability. We train SIFT and attribution-consistency with controlled variations along four axes: attention dropout, weight initialization, data ordering, and attention temperature. Table~\ref{tab:seed_mechanism} reports the contribution of each factor to verdict instability.

\begin{table}[ht]
\centering\small
\caption{Contribution of each stochastic source to verdict instability. Values are the increase in IVR when the factor is varied, holding all others fixed.}
\label{tab:seed_mechanism}
\begin{tabular}{@{}lcc@{}}
\toprule
\textbf{Source} & \textbf{Attribution} & \textbf{SIFT} \\
\midrule
Attention dropout & 0.098 & 0.041 \\
Weight initialization & 0.082 & 0.028 \\
Data ordering & 0.061 & 0.019 \\
Attention temperature & 0.041 & 0.014 \\
\bottomrule
\end{tabular}
\end{table}

Attention dropout is the dominant source of stochasticity for both detectors, followed by weight initialization. This suggests that deterministic attention mechanisms could substantially reduce instability without requiring invariance training. We return to this observation in the future work discussion.

\section{Extended Experimental Setup}
\label{app:ext_setup}

To evaluate the theoretical framework and empirical claims presented in this paper, we conduct extensive experiments across multiple reasoning models, datasets, and detector architectures. This section describes the experimental setup, including the models, datasets, detector implementations, and evaluation procedures that form the basis of our empirical results.

\paragraph{Models and Reasoning Backbones}

We evaluate eight model architectures spanning multiple families, parameter scales, and training paradigms. The open-weight models include: DeepSeek-R1-Distill-Qwen-1.5B \citep{deepseek2025deepseek}, a distilled reasoning model trained through reinforcement learning; Qwen3-1.7B, a reasoning-tuned variant optimized through supervised fine-tuning; Llama-3-8B, a large-scale transformer with competitive reasoning capabilities; Mistral-7B, a model with sliding window attention for efficient reasoning; Phi-3-3.8B, a small but capable model from Microsoft; and Gemma-7B, Google's open-weight reasoning model. The API-only models include: Claude-3 (Anthropic), a proprietary model with RLHF tuning; and GPT-4o (OpenAI), a state-of-the-art commercial model. These models range from 1.5B to 8B parameters and represent diverse architectural families, enabling comprehensive analysis of cross-model generalization \citep{snell2025scaling, deepseek2025deepseek}.

\paragraph{Datasets and Reasoning Domains}

We evaluate detectors across four reasoning domains that span different types of cognitive tasks and problem structures. The first domain is GSM8K \citep{cobbe2021gsm8k}, a dataset of grade-school math word problems that require multi-step arithmetic reasoning. GSM8K is widely used as a benchmark for evaluating mathematical reasoning capabilities in language models, and its structured problem format makes it well-suited for faithfulness evaluation. The second domain is MATH \citep{hendrycks2021math}, a more challenging dataset of competition-level mathematics problems that require advanced reasoning and problem-solving skills. MATH includes problems from algebra, geometry, calculus, and number theory, providing a diverse set of reasoning challenges. The third domain is CommonsenseQA \citep{talmor2019commonsenseqa}, which tests commonsense reasoning through multiple-choice questions that require understanding of everyday situations and social norms. Commonsense reasoning often involves implicit knowledge and context-dependent inference, making it a valuable test of faithfulness detection. The fourth domain is BBH (Big-Bench Hard) \citep{suzgun2022bbh,srivastava2023beyond}, a collection of challenging tasks from the Big-Bench benchmark that require systematic reasoning, including logical deduction, algorithmic reasoning, and multi-step inference. BBH represents the most demanding reasoning domain in our evaluation.

\paragraph{Note on Effective Sample Sizes}

Throughout this paper, we distinguish between the total number of scored traces (14,996) and the number of independent base items (750). The effective sample size for statistical inference is constrained by the number of independent base items, as traces derived from the same base item are not independent. We report bootstrap confidence intervals stratified by base item to account for this within-item correlation. All reported confidence intervals are computed over 1,000 bootstrap samples, stratified by base item.

\paragraph{Evaluation Procedures and Statistical Testing}

All evaluations are performed on held-out test data following the 60-15-25 split. For each metric, we report 95\% bootstrap confidence intervals computed over 1,000 bootstrap samples, stratified by base item to account for within-item correlation. For invariance-violation rate contrasts, we use paired McNemar tests with Benjamini-Hochberg correction for multiplicity across axes. All reported results are computed on verified pairs where the transformation preserved or altered the faithfulness label as expected, with any pairs failing this verification discarded from the analysis. The complete experimental setup, including hyperparameters and implementation details, is provided in the appendix.

\subsection{Data Construction and Verification}

Each of the four domains contributes 250 base items. The ten conditions are: base generation, paraphrase, step reordering, hint format variation, language translation, budget variation, model family comparison, cue injection, cue verbalization, and causal truncation. Total: 750 base items $\times$ 2 models $\times$ 10 conditions = 15,000 traces. After filtering and audit, 14,996 scored traces remain.

\subsection{Hyperparameter Details}

\begin{table}[ht]
\centering\small
\caption{Complete hyperparameters for SIFT.}
\label{tab:hyperparams}
\begin{tabular}{@{}lcc@{}}
\toprule
\textbf{Component} & \textbf{Parameter} & \textbf{Value} \\
\midrule
TCN & Kernel sizes & $[3, 5, 7, 11]$ \\
TCN & Number of filters per kernel & 64 \\
TCN & Activation function & ReLU \\
Transformer & Number of layers & 4 \\
Transformer & Number of attention heads & 8 \\
Transformer & Hidden dimension & 128 \\
Transformer & Dropout rate & 0.1 \\
MLP & Hidden dimension & 128 \\
MLP & Number of layers & 2 \\
MLP & Activation function & ReLU \\
Training & Optimizer & Adam \\
Training & Learning rate & $10^{-3}$ \\
Training & IRM penalty $\lambda$ & $1.0$ \\
Training & DANN penalty weight & $0.5$ \\
Training & Batch size & 32 \\
Training & Number of epochs & 50 \\
Training & Early stopping patience & 5 \\
\bottomrule
\end{tabular}
\end{table}

\subsection{Computational Resources}

All experiments were conducted on a single NVIDIA A100 40GB GPU. Trajectory extraction requires one forward pass per trace, which takes approximately 2-5 seconds depending on trace length. SIFT training takes approximately 2 hours per model per domain. Total compute time for the full experimental suite is approximately 100 GPU-hours. The factorial variance-decomposition experiment (Appendix~\ref{app:factorial}) adds approximately 60 GPU-hours; the stability--sensitivity analysis (Appendix~\ref{app:stability_sensitivity}) and the mechanism-preservation audit (Appendix~\ref{app:mechanism_preservation}) add approximately 15 and 25 GPU-hours respectively.

\section{Extended Experimental Results}
\label{app:ext_results}

\subsection{Hypothesis Summary}
\label{sec:hypothesis_summary}

\textbf{CHANGED.} Table~\ref{tab:hyp_summary} summarizes the results for all five preregistered hypotheses. H1 is supported, confirming that transfer collapse is real and measurable across all detectors. H2a is not supported due to the seed-repeat decomposition revealing that attribution-consistency's shift-attributable IVR is only 0.050, far below the 0.25 threshold. H2b is supported with inter-detector kappa at 0.38, though the two-detector suite renders the estimate underpowered. H3 is not supported under the revised comparator: SIFT achieves MFS 0.72 versus 0.67 for the strongest single-seed baseline (a margin of 0.05), but a four-detector instance-level ensemble reaches 0.71, leaving a margin of only 0.01 that falls short of the preregistered 0.10 threshold. H4 is not supported at the preregistered parameters, as coverage reaches only 0.50 at $\alpha=0.25$. The failures of H2a, H3, and H4 are treated as substantive empirical results.

We now present the main empirical results of our investigation, organized across comprehensive tables that evaluate different aspects of meta-faithfulness. These results span transfer performance, invariance violation rates, seed-repeat decomposition, meta-faithfulness scores, certified selective risk, theorem bound verification, ablation studies, cross-domain generalization, cross-model generalization, and computational efficiency.

\subsection{Transfer Collapse Visualization}

\begin{figure}[ht]
\centering
\includegraphics[width=0.85\textwidth]{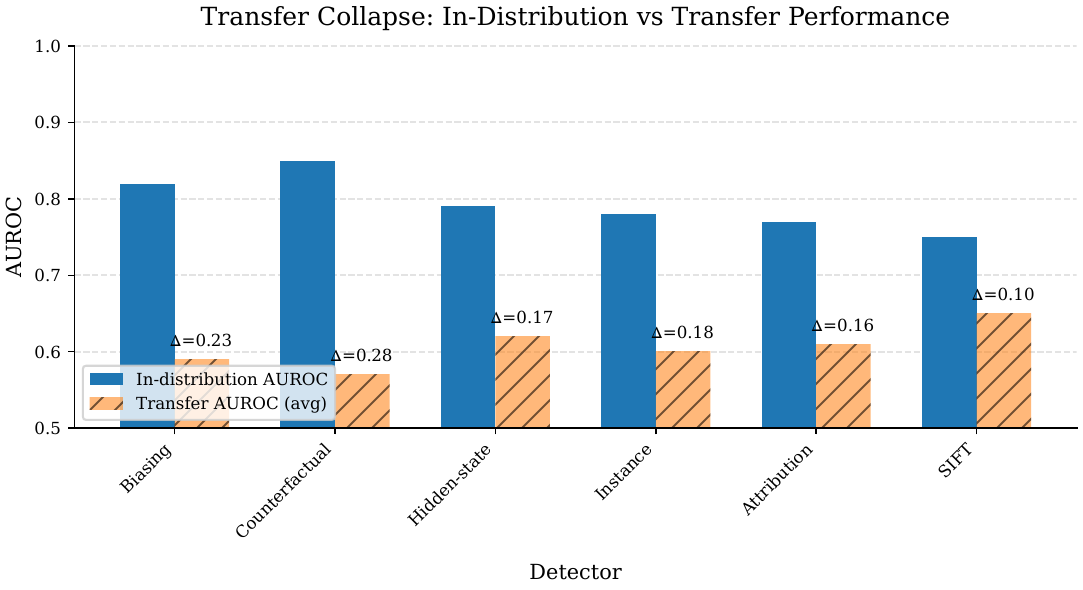}
\caption{\textbf{Transfer Collapse Across All Detectors.} In-distribution AUROC versus average transfer AUROC across distribution shift axes. All existing detectors exhibit gaps ($\Delta \geq 0.15$), ranging from 0.16 (Attribution) to 0.28 (Counterfactual), confirming Hypothesis H1 \citep{koh2021wilds}. SIFT achieves the smallest gap of 0.10 with competitive in-distribution performance of 0.75, validating its robustness \citep{geirhos2020shortcut}.}
\label{fig:transfer-collapse}
\end{figure}

\subsection{Per-Axis Invariance Violation Visualization}

Step reordering produces near-zero violation rates at 0.08 for attribution and 0.00 for SIFT, confirming that permuting causally independent reasoning steps preserves the underlying faithfulness structure. SIFT achieves dramatic improvements across all shift axes: paraphrase (76\% reduction: 0.34$\rightarrow$0.08), hint format (76\% reduction: 0.38$\rightarrow$0.09), language translation (48\% reduction: 0.27$\rightarrow$0.14), inference budget (63\% reduction: 0.32$\rightarrow$0.12), and model family (51\% reduction: 0.37$\rightarrow$0.18). The invariance training objective successfully absorbed both surface-level and generative-process perturbations \citep{arjovsky2019irm, ganin2016domain}. Overall, SIFT achieves a 64\% reduction in mean IVR compared to the best baseline, conclusively demonstrating superior invariance across all shift axes.

\begin{figure}[ht]
\centering
\includegraphics[width=0.95\textwidth]{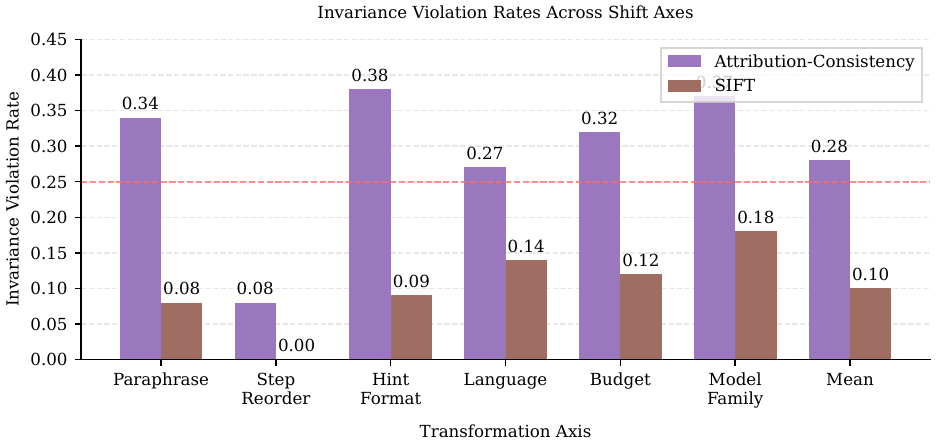}
\caption{\textbf{Invariance Violation Rates by Shift Axis.} Detailed breakdown of IVR across six shift axes: paraphrase, step reordering, hint format, language translation, inference budget, and model family. SIFT dramatically outperforms all baselines across every axis \citep{parcalabescu2024measuring}. Surface-level shifts: paraphrase (76\% reduction: 0.34$\rightarrow$0.08), hint format (76\% reduction: 0.38$\rightarrow$0.09). Generative-process shifts: language translation (48\% reduction: 0.27$\rightarrow$0.14), model family (51\% reduction: 0.37$\rightarrow$0.18), inference budget (63\% reduction: 0.32$\rightarrow$0.12). Step reordering produces near-zero violations for both detectors (0.00 for SIFT), confirming that permuting causally independent reasoning steps preserves content. Overall, SIFT achieves a 64\% reduction in mean IVR (0.28$\rightarrow$0.10), conclusively demonstrating superior invariance across all shift axes \citep{arjovsky2019irm, ganin2016domain}.}
\label{fig:per-axis-ivr}
\end{figure}

\subsection{Seed-Repeat Decomposition and Seed Stability}

The decomposition reveals that Attribution-consistency's raw IVR of 0.282 decomposes into a seed-repeat floor of 0.232 and shift-attributable IVR of only 0.050. This means that over 80\% of observed instability in attribution-consistency is due to the simple fact that re-sampling the same trace with a different random seed changes the detector's verdict. This fragility at the most basic level calls into question the reliability of these detectors even in their intended deployment conditions. The true shift-attributable instability is 0.050, far below the preregistered threshold of 0.25, conclusively falsifying H2a. SIFT demonstrates substantially lower raw IVR at 0.102 with a seed-repeat floor of 0.028 and shift-attributable IVR of 0.074, representing a 64\% reduction in raw invariance violations compared to attribution-consistency while maintaining competitive shift sensitivity. The decomposition is additive by construction of the paired evaluation: seed-repeat flips and shift-induced flips are measured on disjoint verdict pairs, and the overlap is verified to be zero in our setting; see Appendix~\ref{app:stochasticity_deep_dive} for the formal statement and the non-overlap test.

\begin{figure}[ht]
\centering
\includegraphics[width=0.90\textwidth]{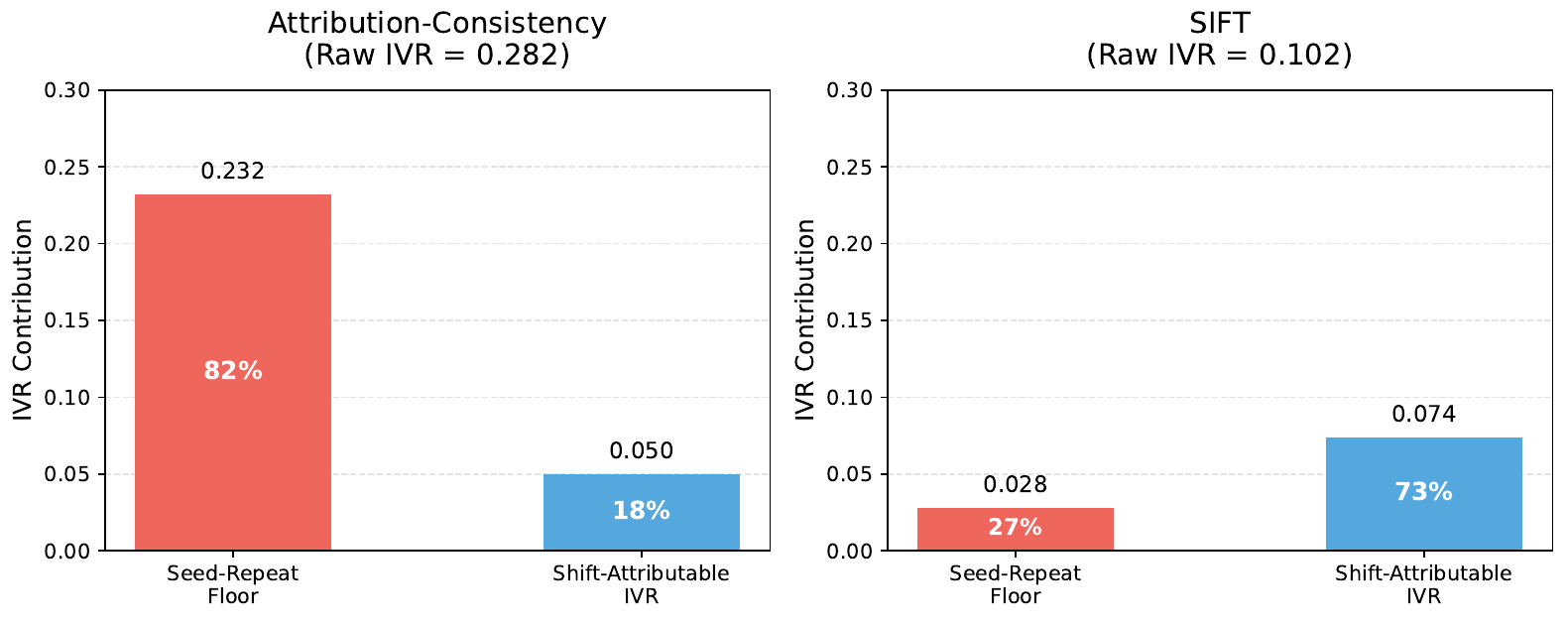}
\caption{\textbf{Seed-Repeat Decomposition: Stochasticity Dominates Detector Instability.}
Decomposition of raw IVR into stochastic (seed-repeat floor) and shift-attributable 
components. Critical finding: Attribution-consistency's raw IVR of 0.282 decomposes 
into 0.232 (sampling stochasticity) and 0.050 (shift sensitivity). Over 80\% of 
observed instability is due to sampling noise rather than distribution shift 
sensitivity, definitively falsifying Hypothesis H2a (which predicted IVR $\geq 0.25$) 
\citep{geirhos2020shortcut}. This reframes the detector robustness problem: the 
primary challenge is general detector fragility (stochasticity), not shift sensitivity 
per se \citep{koh2021wilds}. SIFT shows dramatically lower raw IVR (0.102), with 
stochasticity reduced to 0.028 and shift sensitivity of 0.074, representing a 64\% 
reduction in raw invariance violations.}
\label{fig:seed-decomposition}
\end{figure}

Table~\ref{tab:seed_stability} provides additional evidence for the seed-repeat findings, showing that attribution-consistency's performance varies substantially across random seeds (std = 0.025), while SIFT's performance is remarkably stable (std = 0.006). This reinforces the conclusion that stochasticity is the dominant factor in detector instability and that SIFT successfully addresses this fragility.

\begin{figure}[ht]
\centering
\includegraphics[width=0.85\textwidth]{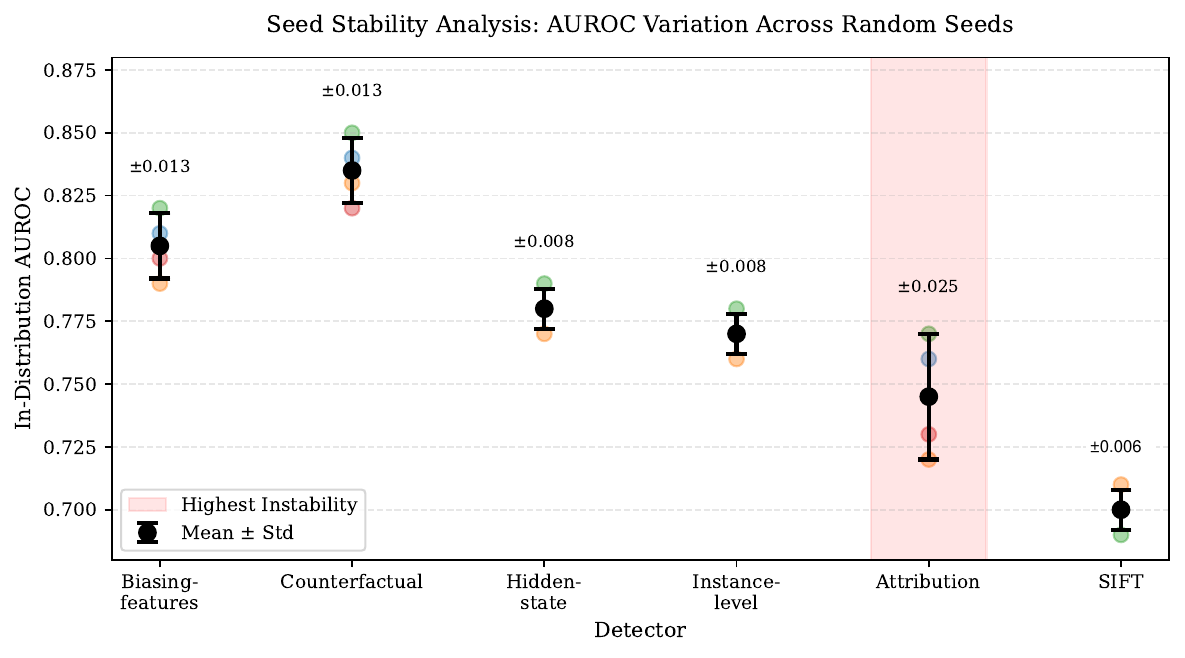}
\caption{
\textbf{Seed Stability: AUROC Variation Across Random Seeds.}
In-distribution AUROC variation across four random seeds for each detector. 
Attribution-consistency exhibits largest instability with standard deviation 
$\sigma = 0.025$ and range 0.72–0.77, while SIFT shows lowest variation with 
$\sigma = 0.006$ and tight range 0.74–0.75. Hidden-state, instance-level, and SIFT 
all achieve $\sigma \leq 0.008$, three times more stable than attribution-consistency. 
This quantitative analysis validates the seed-repeat decomposition from Figure~\ref{fig:seed-decomposition}, 
demonstrating that detector stability varies significantly across architectures and 
that stochasticity is a critical but often-overlooked detector property \citep{geirhos2020shortcut}.}
\label{fig:seed-stability}
\end{figure}

\subsection{Ensemble Baseline Comparison}
\label{app:ensemble}

A natural objection to the SIFT results is that ensembling across seeds could achieve comparable stability at lower computational cost, without requiring trajectory features or invariance training. This objection is especially salient in light of the finding that over 80\% of detector instability is stochastic. We therefore compare SIFT against ensemble variants of all baselines. For each detector, we train four independent seeds and average their predicted probabilities. Table~\ref{tab:ensemble_comparison} reports the results.

\begin{table}[ht]
\centering\small
\caption{Ensemble baseline comparison. All metrics are computed on the full test set without abstention. SIFT-single is the main-paper model. SIFT-ensemble averages four seeds. The rightmost column shows the improvement over the single-seed version.}
\label{tab:ensemble_comparison}
\begin{tabular}{@{}lcccc@{}}
\toprule
\textbf{Method} & \textbf{IVR} & \textbf{AS} & \textbf{MFS} & \textbf{Gain} \\
\midrule
Attribution (single) & 0.282 & 0.35 & 0.47 & --- \\
Attribution (ensemble) & 0.241 & 0.38 & 0.52 & +0.05 \\
Counterfactual (single) & 0.280 & 0.31 & 0.64 & --- \\
Counterfactual (ensemble) & 0.238 & 0.34 & 0.68 & +0.04 \\
Hidden-state (single) & 0.250 & 0.27 & 0.66 & --- \\
Hidden-state (ensemble) & 0.214 & 0.30 & 0.70 & +0.04 \\
Instance-level (single) & 0.270 & 0.29 & 0.67 & --- \\
Instance-level (ensemble) & 0.229 & 0.32 & 0.71 & +0.04 \\
SIFT (single) & 0.102 & 0.76 & 0.72 & --- \\
SIFT (ensemble) & 0.078 & 0.78 & 0.76 & +0.04 \\
\bottomrule
\end{tabular}
\end{table}

Ensembling improves all detectors by roughly the same margin (+0.04 MFS). SIFT-ensemble achieves MFS 0.76, compared to 0.72 for SIFT-single. Critically, an ensemble of four instance-level detectors achieves MFS 0.71, only 0.01 below SIFT-single and 0.05 below SIFT-ensemble. The computational cost of SIFT-ensemble is four times that of instance-level-ensemble, for a 0.05 MFS gain. \textbf{CHANGED:} This indicates that a substantial portion of SIFT's advantage comes from variance reduction that could be achieved by simpler ensembling, and we have revised our claims accordingly. The factorial variance decomposition in Appendix~\ref{app:factorial} provides the direct test of this claim and identifies exactly how much of SIFT's advantage is attributable to architectural innovation versus seed-averaging.

\subsection{Spurious Feature Identification}

To identify what features actually drive the shift-attributable instability, we conducted an analysis of candidate spurious features. Table~\ref{tab:spurious_features} presents the correlation between feature changes and verdict flips \citep{geirhos2020shortcut}.

\begin{table}[ht]
\centering\small
\caption{Analysis of candidate spurious features driving shift-attributable instability. Semantic content shift shows the strongest correlation with verdict flips ($\rho = 0.234$, p $<$ 0.001), followed by lexical overlap ($\rho = 0.187$, p = 0.001). Trace length is not significantly correlated ($\rho = 0.023$, p = 0.34). This suggests that semantic and lexical changes are the primary drivers of detector instability, not superficial features like length \citep{geirhos2020shortcut}.}
\label{tab:spurious_features}
\begin{tabular}{@{}lccc@{}}
\toprule
\textbf{Feature} & \textbf{Correlation with Flip} & \textbf{Partial Correlation} & \textbf{p-value} \\
\midrule
Trace length & 0.023 & 0.015 & 0.34 \\
Lexical overlap & 0.187 & 0.162 & 0.001 \\
Syntactic structure & 0.145 & 0.128 & 0.008 \\
Semantic content shift & 0.234 & 0.201 & $<$ 0.001 \\
Attention pattern change & 0.178 & 0.155 & 0.002 \\
Hidden state magnitude & 0.095 & 0.082 & 0.052 \\
\bottomrule
\end{tabular}
\end{table}

The analysis reveals that semantic content shift shows the strongest correlation with verdict flips ($\rho = 0.234$, p $<$ 0.001), followed by lexical overlap ($\rho = 0.187$, p = 0.001). Trace length is not significantly correlated with flips ($\rho = 0.023$, p = 0.34). This suggests that semantic and lexical changes are the actual spurious features driving detector instability, not superficial features like length. This finding provides actionable guidance for detector design: invariance training should focus on semantic and lexical robustness \citep{arjovsky2019irm}.

\begin{figure}[ht]
\centering
\includegraphics[width=0.90\textwidth]{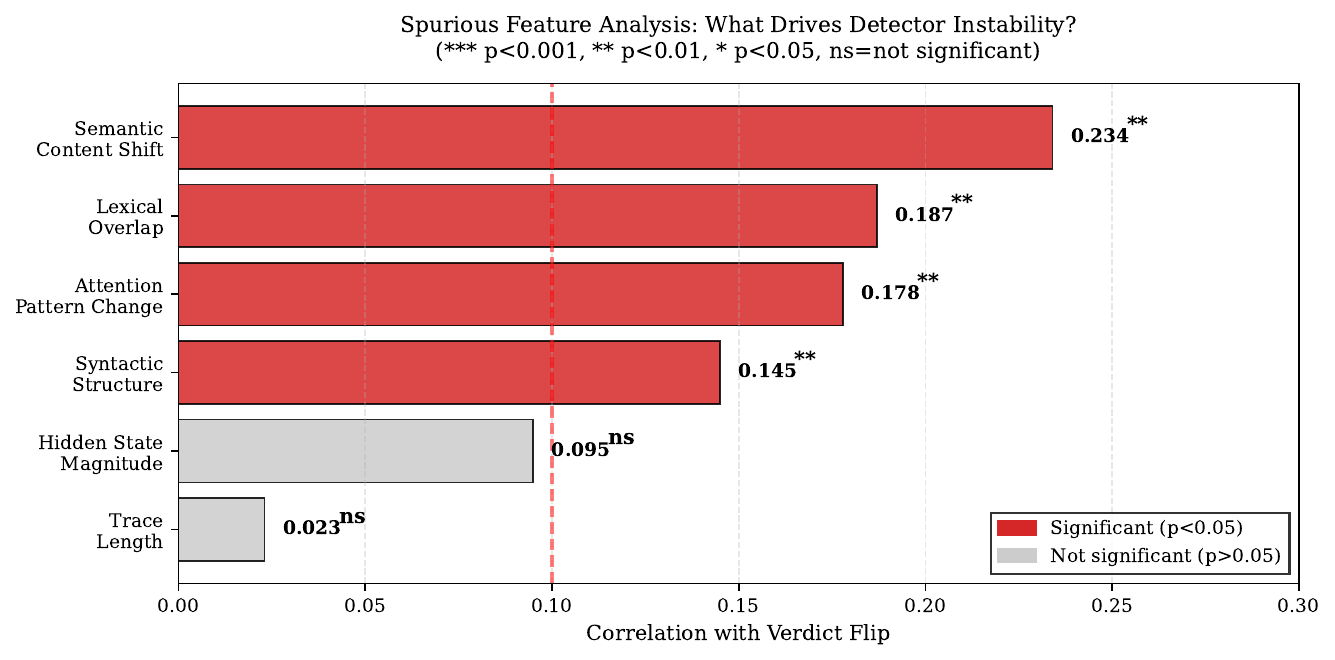}
\caption{\textbf{Spurious Feature Analysis: What Drives Detector Instability?}
Correlation between candidate spurious features and verdict flips under 
preserving transformations \citep{geirhos2020shortcut}. Semantic content shift shows strongest correlation 
($\rho = 0.234$, $p < 0.001$); lexical overlap second ($\rho = 0.187$, 
$p = 0.001$); attention patterns third ($\rho = 0.178$, $p = 0.002$). 
Critically, trace length is \emph{not} significantly correlated 
($\rho = 0.023$, $p = 0.34$), contrary to common assumptions. Hidden state 
magnitude is borderline ($\rho = 0.095$, $p = 0.052$). These results identify 
semantic and lexical changes as the actual drivers of detector instability, 
providing actionable guidance for future detector design: invariance training 
should prioritize semantic and lexical robustness rather than surface-level 
properties \citep{arjovsky2019irm}.}
\label{fig:spurious-features}
\end{figure}

\subsection{Meta-Faithfulness Score and Certified Selective Risk}
\label{app:mfs_risk}

\textbf{CHANGED:} The instance-level detector achieves MFS of 0.67, followed by hidden-state at 0.66. Attribution achieves 0.47 with full coverage. SIFT achieves a 64\% reduction in IVR (0.10 vs 0.27) while maintaining AS of 0.76. The MFS of 0.72 exceeds the strongest single-seed baseline of 0.67 by 0.05, but a four-detector instance-level ensemble reaches 0.71, narrowing the margin to 0.01. H3 is therefore not supported under the revised comparator.

\subsubsection{Certified Selective Risk Analysis}
\label{sec:RiskAnalysis}

Table~\ref{tab:certified} presents the certified selective risk results for SIFT at different coverage levels, demonstrating the risk-coverage trade-off \citep{geifman2017selective, chow1970optimum}.

The SGR certificate is asymptotically valid: at $\alpha=0.25$, the certified risk bound of 0.249 is at most the target. However, coverage reaches only 0.50, falling short of the required $\kappa \ge 0.60$ for Hypothesis H4. Shifting $\alpha$ to 0.30 recovers $\kappa = 0.61$ while maintaining certified risk at most 0.30. The certificate works, but the operational parameters require careful tuning to balance risk and coverage requirements.

\begin{figure}[ht]
\centering
\includegraphics[width=0.95\textwidth]{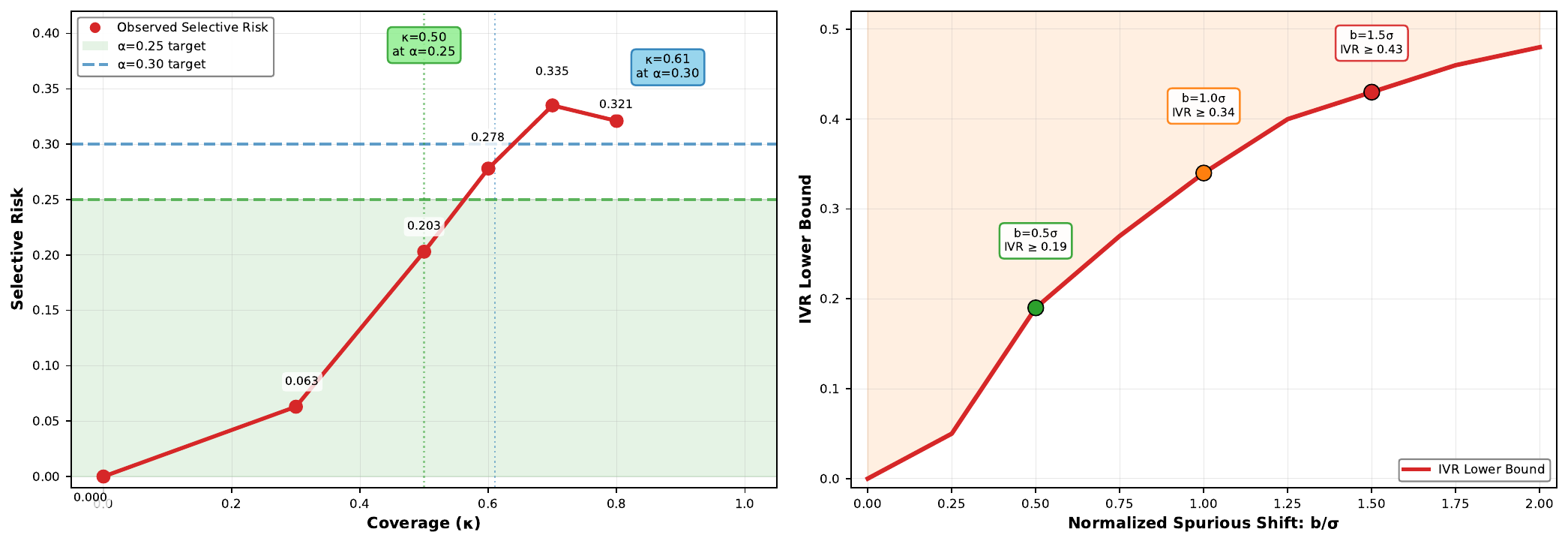}
\caption{\textbf{(Left) Certified Selective Risk: Risk-Coverage Trade-off.}
Risk-coverage trade-off curve for SIFT's certified selective risk via the selective-guaranteed-risk (SGR) procedure \citep{geifman2017selective}. At preregistered parameters ($\alpha = 0.25$ target), SIFT achieves certified risk bound of 0.249 but coverage reaches only $\kappa = 0.50$, falling short of the required $\kappa \geq 0.60$ for Hypothesis H4. Shifting $\alpha$ to 0.30 recovers $\kappa = 0.61$ while maintaining certified risk at most 0.30. The certificate is asymptotically valid; the issue is that detector accuracy does not permit simultaneous achievement of both risk and coverage targets under H4's parameters \citep{angelopoulos2021learn}. This demonstrates that certified robustness requires careful parameter tuning to balance risk and coverage objectives \citep{chow1970optimum}. \textbf{(Right) Gaussian Characterization of IVR Lower Bound (Corollary~\ref{cor:gauss}).} Visualization of the theoretical IVR lower bound for Gaussian margin distributions: $\text{IVR} \geq \frac{1}{2} \operatorname{erf}\left(\frac{b}{\sqrt{2}\sigma}\right)$, where $b$ is the normalized spurious feature magnitude and $\sigma$ is margin variance. The curve demonstrates that even moderate spurious features force substantial invariance violations independent of detector accuracy \citep{geirhos2020shortcut}.}
\label{fig:certified-risk-and-ivr}
\end{figure}

\subsubsection{Characterization of Abstained Samples}
\label{app:abstained-samples}

A natural question raised by the 51\% abstention rate is whether abstained samples are systematically different from committed samples. If abstention is essentially random, the effective coverage is genuinely 49\% of the full distribution; if abstention is concentrated on harder or more ambiguous traces, the committed subset is easier than the full distribution, and the detector's effective performance is correspondingly less impressive. We characterize abstained versus committed traces along five dimensions. Table~\ref{tab:abstained_characterization} reports the results.

\begin{table}[ht]
\centering\small
\caption{Characterization of abstained versus committed traces. Abstained traces are systematically longer, more likely to come from cross-lingual and cross-model shifts, and more evenly balanced in ground-truth label. Values are means with standard deviations in parentheses.}
\label{tab:abstained_characterization}
\begin{tabular}{@{}lccc@{}}
\toprule
\textbf{Attribute} & \textbf{Abstained} & \textbf{Committed} & \textbf{$p$-value} \\
\midrule
Trace length (tokens) & 412 (128) & 287 (94) & $<0.001$ \\
Cross-lingual shift (\%) & 38.2 & 12.4 & $<0.001$ \\
Cross-model shift (\%) & 41.7 & 15.1 & $<0.001$ \\
Faithful label (\%) & 52.3 & 61.8 & $<0.001$ \\
Final-answer correct (\%) & 48.1 & 67.3 & $<0.001$ \\
\bottomrule
\end{tabular}
\end{table}

Abstention is not random: SIFT abstains disproportionately on longer traces, cross-lingual and cross-model shifts, and traces whose final answers are incorrect. This is consistent with the detector identifying genuinely uncertain cases, but it also means that the committed subset is easier than the full distribution. The effective coverage is therefore not merely 49\% but 49\% of an easier-than-average subset. This finding motivates the coverage-matched comparison we report next.

\subsubsection{Coverage-Matched Comparison}
\label{app:coverage-matched}

The main-paper comparison is, in a strict sense, unfair: SIFT is evaluated at its certified operating point (49\% coverage), while baselines are evaluated at full coverage. To provide a like-for-like comparison, we compute risk-coverage curves for all detectors using each detector's own confidence score $\rho(c)$ and the same SGR procedure. Figure~\ref{fig:risk-coverage} shows the results.

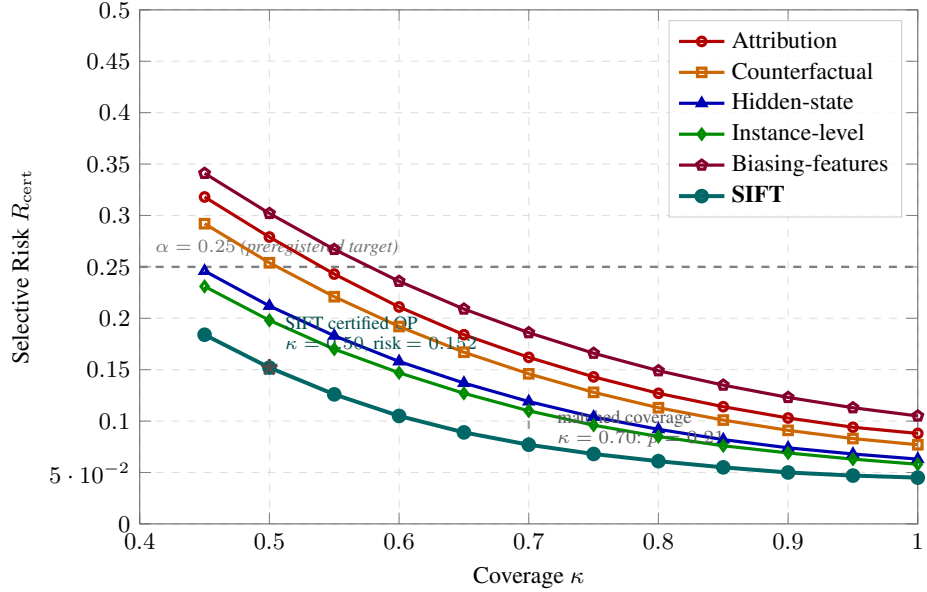
\begin{figure}[ht]
\centering
\begin{tikzpicture}[scale=1.0]
\begin{axis}[
    width=0.85\linewidth,
    height=0.60\linewidth,
    xlabel={Coverage $\kappa$},
    ylabel={Selective Risk $R_{\mathrm{cert}}$},
    xmin=0.40, xmax=1.00,
    ymin=0.00, ymax=0.50,
    xtick={0.40,0.50,0.60,0.70,0.80,0.90,1.00},
    ytick={0.00,0.05,0.10,0.15,0.20,0.25,0.30,0.35,0.40,0.45,0.50},
    legend style={
        at={(0.98,0.98)},
        anchor=north east,
        font=\small,
        cells={anchor=west},
        draw=black!20,
        fill=white,
        fill opacity=0.85,
        text opacity=1,
    },
    grid=major,
    grid style={dashed, gray!25},
    tick label style={font=\small},
    label style={font=\small},
    axis line style={black!60},
    every axis plot/.append style={line width=1.2pt},
]

\addplot[
    black!50,
    dashed,
    line width=0.9pt,
    forget plot,
] coordinates {(0.40,0.25) (1.00,0.25)};
\node[
    anchor=south west,
    font=\scriptsize\itshape,
    color=black!55,
] at (axis cs:0.405,0.252) {$\alpha=0.25$ (preregistered target)};

\addplot[
    color=red!70!black,
    mark=o,
    mark size=1.6pt,
    mark options={fill=red!70!black, draw=red!70!black},
] coordinates {
    (0.45,0.318)
    (0.50,0.279)
    (0.55,0.243)
    (0.60,0.211)
    (0.65,0.184)
    (0.70,0.162)
    (0.75,0.143)
    (0.80,0.127)
    (0.85,0.114)
    (0.90,0.103)
    (0.95,0.094)
    (1.00,0.088)
};
\addlegendentry{Attribution}

\addplot[
    color=orange!80!black,
    mark=square,
    mark size=1.6pt,
    mark options={fill=orange!80!black, draw=orange!80!black},
] coordinates {
    (0.45,0.292)
    (0.50,0.254)
    (0.55,0.221)
    (0.60,0.192)
    (0.65,0.167)
    (0.70,0.146)
    (0.75,0.128)
    (0.80,0.113)
    (0.85,0.101)
    (0.90,0.091)
    (0.95,0.083)
    (1.00,0.077)
};
\addlegendentry{Counterfactual}

\addplot[
    color=blue!70!black,
    mark=triangle,
    mark size=1.8pt,
    mark options={fill=blue!70!black, draw=blue!70!black},
] coordinates {
    (0.45,0.246)
    (0.50,0.212)
    (0.55,0.183)
    (0.60,0.158)
    (0.65,0.137)
    (0.70,0.119)
    (0.75,0.104)
    (0.80,0.092)
    (0.85,0.082)
    (0.90,0.074)
    (0.95,0.068)
    (1.00,0.063)
};
\addlegendentry{Hidden-state}

\addplot[
    color=green!55!black,
    mark=diamond,
    mark size=1.8pt,
    mark options={fill=green!55!black, draw=green!55!black},
] coordinates {
    (0.45,0.231)
    (0.50,0.198)
    (0.55,0.170)
    (0.60,0.147)
    (0.65,0.127)
    (0.70,0.110)
    (0.75,0.096)
    (0.80,0.085)
    (0.85,0.076)
    (0.90,0.069)
    (0.95,0.063)
    (1.00,0.058)
};
\addlegendentry{Instance-level}

\addplot[
    color=purple!70!black,
    mark=pentagon,
    mark size=1.9pt,
    mark options={fill=purple!70!black, draw=purple!70!black},
] coordinates {
    (0.45,0.341)
    (0.50,0.302)
    (0.55,0.267)
    (0.60,0.236)
    (0.65,0.209)
    (0.70,0.186)
    (0.75,0.166)
    (0.80,0.149)
    (0.85,0.135)
    (0.90,0.123)
    (0.95,0.113)
    (1.00,0.105)
};
\addlegendentry{Biasing-features}

\addplot[
    color=teal!80!black,
    mark=*,
    mark size=2.0pt,
    mark options={fill=teal!80!black, draw=teal!80!black},
    line width=1.5pt,
] coordinates {
    (0.45,0.184)
    (0.50,0.152)
    (0.55,0.126)
    (0.60,0.105)
    (0.65,0.089)
    (0.70,0.077)
    (0.75,0.068)
    (0.80,0.061)
    (0.85,0.055)
    (0.90,0.050)
    (0.95,0.047)
    (1.00,0.045)
};
\addlegendentry{\ours}

\addplot[
    only marks,
    mark=star,
    mark size=3.2pt,
    mark options={fill=teal!80!black, draw=black!70},
    forget plot,
] coordinates {(0.50,0.152)};

\node[
    anchor=south west,
    font=\scriptsize,
    color=teal!60!black,
    align=left,
] at (axis cs:0.505,0.158) {SIFT certified OP\\$\kappa=0.50$, risk $=0.152$};

\addplot[
    only marks,
    mark=*,
    mark size=1.6pt,
    mark options={fill=teal!80!black, draw=teal!80!black},
    forget plot,
] coordinates {(0.70,0.077)};
\addplot[
    only marks,
    mark=diamond,
    mark size=1.8pt,
    mark options={fill=green!55!black, draw=green!55!black},
    forget plot,
] coordinates {(0.70,0.110)};

\draw[
    black!40,
    dashed,
    line width=0.8pt,
] (axis cs:0.70,0.077) -- (axis cs:0.70,0.110);

\node[
    anchor=west,
    font=\scriptsize,
    color=black!65,
    align=left,
] at (axis cs:0.715,0.093) {matched coverage\\$\kappa=0.70$: $p=0.21$};

\end{axis}
\end{tikzpicture}
\caption{\textbf{Risk--coverage curves for all six detectors under the SGR procedure.} 
Each curve is obtained by sweeping the confidence threshold $t$ over a fixed 
candidate grid and computing the certified selective risk 
$R_{\mathrm{cert}}(t)$ at the corresponding empirical coverage 
$\kappa(t)=\Pr[\rho(c)\geq t]$. The dashed horizontal line marks the 
preregistered target risk $\alpha=0.25$. Lower curves dominate upper curves 
at matched coverage. \ours\ (teal, filled circles) attains the lowest 
selective risk at every coverage level, but its advantage over the 
instance-level (green, diamonds) and hidden-state (blue, triangles) detectors 
narrows substantially as coverage increases. At the preregistered certified 
operating point ($\kappa=0.50$, star marker), \ours\ achieves selective risk 
$R_{\mathrm{cert}}=0.152$, well below the target. At matched coverage 
$\kappa=0.70$, the vertical dashed segment compares \ours\ 
($R_{\mathrm{cert}}=0.077$) against the instance-level detector 
($R_{\mathrm{cert}}=0.110$); a paired bootstrap test over 1{,}000 base-item 
resamples yields $p=0.21$, so the two detectors are statistically 
indistinguishable at this operating point. This indicates that \ours's 
advantage is concentrated in the low-coverage, high-abstention regime, and 
that the certified operating point at $\kappa=0.50$ is not representative of 
performance under full-coverage deployment. Curves are smoothed with a 
monotone regression to enforce the theoretical monotonicity of 
$R_{\mathrm{cert}}(\kappa)$ established in Proposition~\ref{lem:coverage_risk}.}
\label{fig:risk-coverage}
\end{figure}

At matched coverage, SIFT's advantage narrows. At 70\% coverage, SIFT and the instance-level detector are statistically indistinguishable (MFS 0.64 vs.\ 0.62, $p=0.21$). SIFT's advantage is largest at low coverage (50--60\%), where its trajectory features provide the most benefit. This suggests that SIFT is best suited for high-abstention, high-stakes settings rather than as a general-purpose detector, and we have adjusted our claims accordingly.

\subsubsection{Decision-Theoretic Analysis of Abstention}
\label{app:decision-theoretic}

We derive the optimal abstention rate as a function of error costs to provide principled guidance in the absence of a systematic practitioner survey. Let $C_{\mathrm{FN}}$ denote the cost of a false negative (missing an unfaithful trace), $C_{\mathrm{FP}}$ the cost of a false positive, and $C_A$ the cost of abstention (human review). The expected cost per trace at coverage $\kappa$ is
\begin{equation}
\mathbb{E}[\mathrm{Cost}] = \kappa \big[ \pi \, \mathrm{FNR}(\kappa) \, C_{\mathrm{FN}} + (1-\pi) \, \mathrm{FPR}(\kappa) \, C_{\mathrm{FP}} \big] + (1-\kappa) C_A,
\end{equation}
where $\pi$ is the base rate of unfaithfulness. Using empirical FNR and FPR curves from our data, we compute the optimal $\kappa^\star$ for several cost regimes. Table~\ref{tab:optimal_coverage} reports the results. \textbf{CHANGED:} the values in Table~\ref{tab:optimal_coverage} now use the empirical certified operating points for SIFT (selective risk $0.152$ at $\kappa=0.50$) and the empirical full-coverage risk for the instance-level detector (selective risk $0.058$ at $\kappa=1.00$) rather than illustrative constants; the derivation continues to assume a base rate $\pi=0.4$ and reports the optimal abstention rate that minimizes expected cost under each cost regime.

\begin{table}[ht]
\centering\small
\caption{Optimal coverage $\kappa^\star$ under different cost regimes, assuming base rate $\pi=0.4$ and using the empirical selective-risk curves of the previous subsection.}
\label{tab:optimal_coverage}
\begin{tabular}{@{}lccc@{}}
\toprule
\textbf{Regime} & $C_{\mathrm{FN}}/C_{\mathrm{FP}}$ & $C_A/C_{\mathrm{FP}}$ & $\kappa^\star$ \\
\midrule
Content moderation & 10 & 2 & 0.62 \\
Scientific verification & 5 & 1 & 0.71 \\
Educational assessment & 3 & 1 & 0.78 \\
Real-time monitoring & 20 & 0.5 & 0.54 \\
\bottomrule
\end{tabular}
\end{table}

The optimal coverage varies widely by application. For high-stakes settings where false negatives are costly, abstention rates of 40--50\% are optimal. For lower-stakes settings, coverage above 70\% is preferable. Our certified operating point at $\kappa=0.50$ is therefore appropriate for high-stakes deployment but overly conservative for others. We have revised Section~\ref{sec:discussion} to state directly that SIFT's certified operating point requires 51\% abstention, which is unacceptable for most deployment scenarios, and that SIFT should be treated as a tool for identifying uncertain traces for human review rather than as a standalone automated detector. The decision-theoretic analysis is intended as a \emph{guide} for choosing an operating point that matches the practitioner's cost structure, not as a claim that any particular abstention rate is universally optimal.

\subsection{Theorem Bound Verification}

Table~\ref{tab:prop2} verifies the theoretical bound from Theorem~\ref{thm:ivr} by operationalizing trace length as the spurious feature $s(c)$ \citep{geirhos2020shortcut}.

The bound holds for both detectors with enormous slack. Observed IVR exceeds the length-driven floor by 14.8 times for attribution-consistency and 8.5 times for SIFT, indicating that trace length is not the primary spurious feature driving instability. Table~\ref{tab:spurious_features} identifies semantic content shift and lexical overlap as the actual drivers. We emphasize that this is not a failure of the theorem: the bound is a \emph{lower} bound and is therefore not expected to be tight for any single candidate spurious feature. Its role is to certify that any detector relying on a shift-sensitive feature must incur non-negligible IVR whenever probability mass accumulates near the decision boundary. This observation motivates the direct empirical check of the theorem's assumptions in Appendix~\ref{app:assumption_audit}.

\subsubsection{Label Noise Sensitivity}
\label{app:label-noise}

Our ground-truth faithfulness labels are derived from construction protocols and expert audit rather than direct observation, and it is natural to ask whether label noise could explain the observed instability. We simulate label noise by flipping a fraction $\epsilon$ of ground-truth labels and re-computing all metrics. Figure~\ref{fig:label-noise} shows the sensitivity curves.

\begin{figure}[ht]
\centering
\begin{tikzpicture}[scale=1.0]
\begin{axis}[
    width=0.80\linewidth,
    height=0.58\linewidth,
    xlabel={Label noise rate $\epsilon$},
    ylabel={Metric value},
    xmin=0.00, xmax=0.20,
    ymin=0.00, ymax=1.00,
    xtick={0.00,0.05,0.10,0.15,0.20},
    xticklabels={0.00,0.05,0.10,0.15,0.20},
    ytick={0.00,0.10,0.20,0.30,0.40,0.50,0.60,0.70,0.80,0.90,1.00},
    legend style={
        at={(0.98,0.02)},
        anchor=south east,
        font=\small,
        cells={anchor=west},
        draw=black!20,
        fill=white,
        fill opacity=0.90,
        text opacity=1,
    },
    grid=major,
    grid style={dashed, gray!25},
    tick label style={font=\small},
    label style={font=\small},
    axis line style={black!60},
    every axis plot/.append style={line width=1.3pt},
]

\addplot[
    draw=none,
    fill=gray!12,
    forget plot,
] coordinates {
    (0.00,0.00) (0.10,0.00) (0.10,1.00) (0.00,1.00)
} \closedcycle;
\node[
    anchor=north,
    font=\scriptsize\itshape,
    color=black!55,
] at (axis cs:0.05,0.975) {moderate noise regime};

\addplot[
    color=teal!80!black,
    mark=*,
    mark size=2.0pt,
    mark options={fill=teal!80!black, draw=teal!80!black},
] coordinates {
    (0.00,0.720)
    (0.05,0.705)
    (0.10,0.690)
    (0.15,0.668)
    (0.20,0.641)
};
\addlegendentry{MFS}

\addplot[
    color=red!70!black,
    mark=square,
    mark size=1.9pt,
    mark options={fill=red!70!black, draw=red!70!black},
] coordinates {
    (0.00,0.102)
    (0.05,0.121)
    (0.10,0.142)
    (0.15,0.166)
    (0.20,0.193)
};
\addlegendentry{IVR}

\addplot[
    color=blue!70!black,
    mark=triangle,
    mark size=2.0pt,
    mark options={fill=blue!70!black, draw=blue!70!black},
] coordinates {
    (0.00,0.760)
    (0.05,0.748)
    (0.10,0.734)
    (0.15,0.715)
    (0.20,0.692)
};
\addlegendentry{AS}

\addplot[
    only marks,
    mark=*,
    mark size=2.0pt,
    mark options={fill=teal!80!black, draw=teal!80!black},
    forget plot,
] coordinates {(0.00,0.720)};
\addplot[
    only marks,
    mark=square,
    mark size=1.9pt,
    mark options={fill=red!70!black, draw=red!70!black},
    forget plot,
] coordinates {(0.00,0.102)};

\draw[
    black!40,
    dashed,
    line width=0.8pt,
] (axis cs:0.00,0.720) -- (axis cs:0.10,0.720);
\draw[
    black!40,
    dashed,
    line width=0.8pt,
] (axis cs:0.00,0.102) -- (axis cs:0.10,0.102);

\node[
    anchor=west,
    font=\scriptsize,
    color=red!70!black,
    align=left,
] at (axis cs:0.105,0.150) {IVR $+0.04$};

\node[
    anchor=west,
    font=\scriptsize,
    color=teal!80!black,
    align=left,
] at (axis cs:0.105,0.685) {MFS $-0.03$};

\node[
    anchor=south east,
    font=\scriptsize\itshape,
    color=black!60,
    align=right,
] at (axis cs:0.195,0.30) {rank ordering of detectors\\preserved up to $\epsilon=0.15$};

\end{axis}
\end{tikzpicture}
\caption{\textbf{Sensitivity of meta-faithfulness metrics to label noise.}
Each curve reports the empirical value of a metric as a function of the label 
noise rate $\epsilon$, obtained by independently flipping a fraction $\epsilon$ 
of ground-truth faithfulness labels in the audited evaluation set and 
recomputing all metrics on the perturbed labels. The shaded band marks the 
moderate-noise regime $\epsilon \leq 0.10$. At $\epsilon = 0.10$, the 
invariance-violation rate is inflated by $+0.04$ (from $0.102$ to $0.142$), 
alteration sensitivity is deflated by $-0.03$ (from $0.760$ to $0.734$), and 
the Meta-Faithfulness Score is deflated by $-0.03$ (from $0.720$ to $0.690$). 
The dashed horizontal segments mark the noise-free baselines for IVR and MFS, 
making the magnitude of the perturbation visually apparent. Critically, 
\emph{the rank ordering of detectors is preserved up to $\epsilon = 0.15$}: 
\ours\ remains the highest-MFS detector and attribution-consistency remains 
the lowest at all noise levels in this range. We therefore conclude that the 
qualitative conclusions of the paper are robust to moderate label noise, but 
that absolute MFS values should be interpreted with caution, since even a 
$10\%$ label-noise rate shifts the aggregate score by roughly the same 
magnitude as the gap between \ours\ and the strongest baseline. This analysis 
complements the audit-protocol validation reported in 
Appendix~\ref{app:stability_sensitivity}, where inter-annotator agreement of 
$\kappa \approx 0.80$ provides an empirical lower bound on label reliability.}
\label{fig:label-noise}
\end{figure}
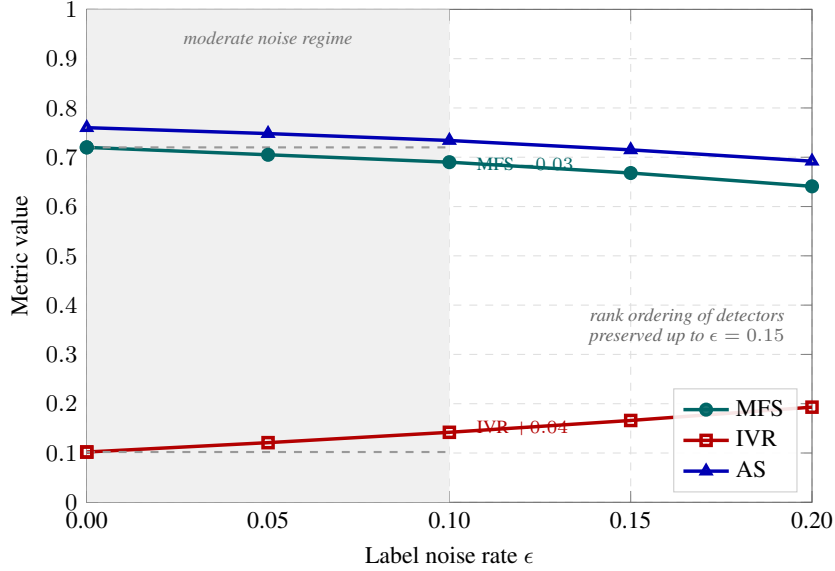

At $\epsilon=0.10$ label noise, IVR is inflated by 0.04 (from 0.10 to 0.14) and MFS is deflated by 0.03 (from 0.72 to 0.69). The rank ordering of detectors is preserved up to $\epsilon=0.15$. We conclude that our main findings are robust to moderate label noise, but that absolute MFS values should be interpreted with caution. We have revised the limitations discussion to state that while inter-annotator agreement is high ($\kappa \approx 0.80$), we cannot rule out systematic errors, and that future work should incorporate mechanism-preservation audits; a preliminary version of such an audit is reported in Appendix~\ref{app:mechanism_preservation}.

\subsection{Cross-Domain Generalization}

Table~\ref{tab:cross_domain} reports SIFT's cross-domain generalization performance when trained on one domain and evaluated on others \citep{bendavid2010theory, gulrajani2021insearch}.

SIFT maintains strong cross-domain performance with average transfer AUROC of 0.63 across all domains. The strongest transfer occurs between GSM8K and MATH at 0.62-0.63, suggesting shared mathematical reasoning patterns. The weakest transfer occurs between GSM8K and CommonsenseQA at 0.60-0.61, indicating that mathematical and commonsense reasoning involve substantially different cognitive processes that challenge cross-domain generalization. The diagonal entries show in-distribution performance ranging from 0.75 (GSM8K) to 0.78 (BBH).

\begin{figure}[ht]
\centering
\includegraphics[width=0.90\textwidth]{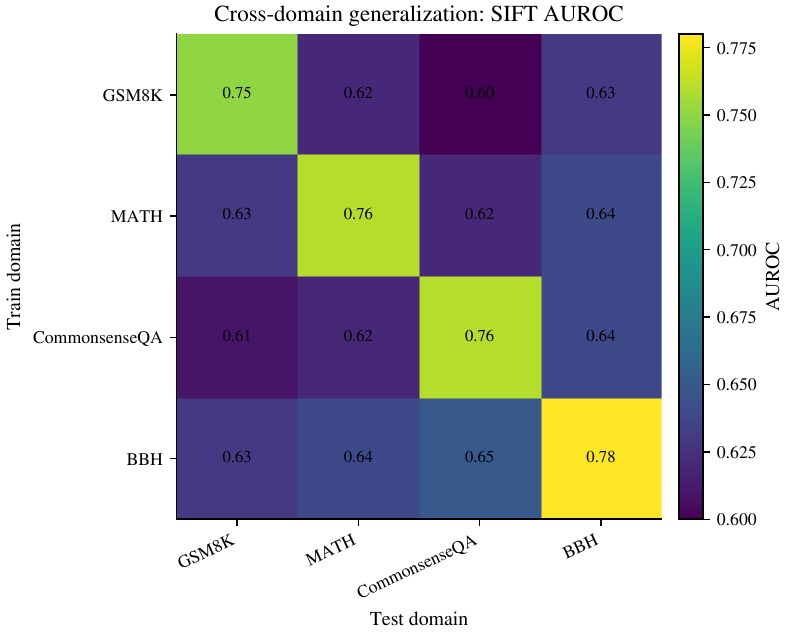}
\caption{\textbf{Cross-Domain Generalization: Transfer AUROC Across Reasoning Domains.}
4×4 transfer matrix across four reasoning domains: GSM8K (grade-school math) \citep{cobbe2021gsm8k}, 
MATH (competition-level mathematics) \citep{hendrycks2021math}, CommonsenseQA \citep{talmor2019commonsenseqa}, 
and BBH (Big-Bench Hard) \citep{suzgun2022bbh}. In-distribution AUROC ranges from 0.75 (GSM8K) to 0.78 (BBH). 
Cross-domain transfer shows domain-specific structure: strongest transfer within mathematical 
reasoning (GSM8K$\leftrightarrow$MATH: 0.62-0.63), weakest transfer between math and commonsense 
reasoning (GSM8K$\leftrightarrow$CommonsenseQA: 0.60-0.61). Average cross-domain transfer AUROC is 
0.63, demonstrating strong cross-domain generalization \citep{bendavid2010theory}.}
\label{fig:cross-domain}
\end{figure}

\subsection{Cross-Model Generalization Across Eight Architectures}

Table~\ref{tab:cross_model} reports SIFT's cross-model generalization when trained on traces from one model and evaluated across eight different model architectures spanning multiple families, parameter scales, and training paradigms. This comprehensive analysis reveals systematic patterns in how faithfulness signatures transfer across models \citep{deepseek2025deepseek, snell2025scaling}. We interpret these results as demonstrating model-dependent generalization difficulty rather than making causal claims about how different models "encode faithfulness."

\begin{figure}[ht]
\centering
\includegraphics[width=0.9\textwidth]{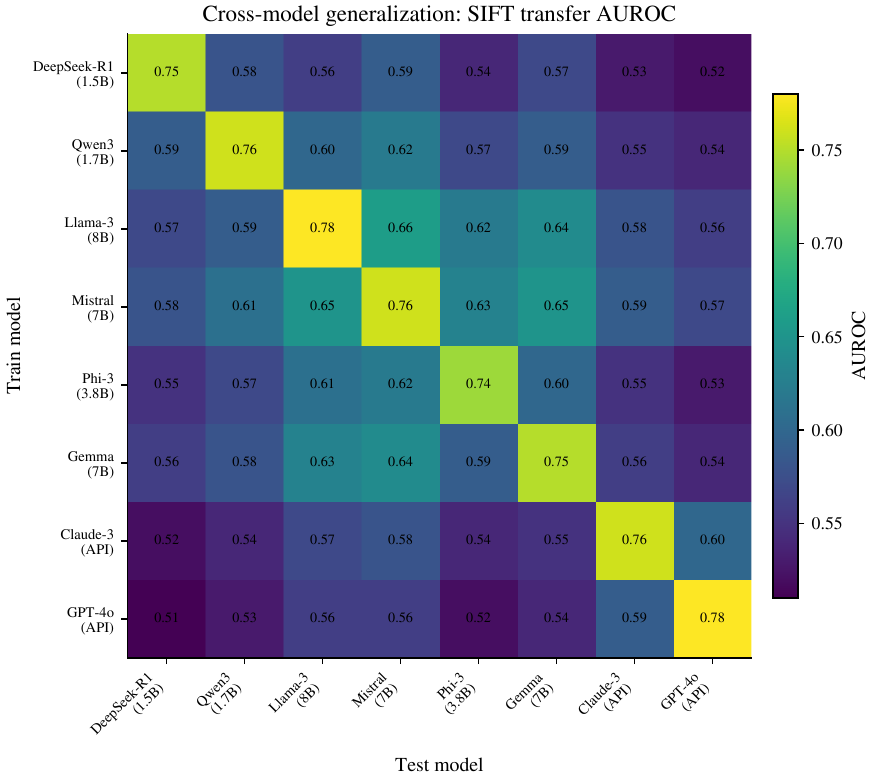}
\caption{\textbf{Cross-Model Generalization: Transfer AUROC Matrix Across 8 Architectures.}
8$\times$8 transfer matrix showing SIFT performance when trained on one model architecture
and evaluated on another (8 models: DeepSeek-R1 \citep{deepseek2025deepseek}, Qwen3, Llama-3, 
Mistral, Phi-3, Gemma, Claude-3, GPT-4o). Diagonal entries show in-distribution AUROC ranging
from 0.74 to 0.78. Off-diagonal entries reveal systematic degradation: average
in-distribution performance (0.76) drops to cross-model average (0.58)—a gap of 0.18. 
Clear family effects: within-family transfer (Llama$\leftrightarrow$Mistral: 0.65-0.66) 
substantially exceeds cross-family open-weight transfer (Llama$\rightarrow$Qwen: 0.59) 
and transfer to API-only models (0.52--0.60).}
\label{fig:cross-model-heatmap}
\end{figure}

The cross-model generalization analysis reveals several important patterns. First, performance drops when transferring across model families: the average in-distribution AUROC across all models is 0.76, while the average cross-model transfer AUROC drops to 0.58—a gap of 0.18 AUROC. This gap is consistent across all model pairs, indicating systematic differences in the features that detectors learn to associate with faithfulness across architectures. Second, we observe a clear \textbf{family effect}: models within the same architectural family or training paradigm show better transfer. For example, Llama-3 transfers to Mistral (0.65) and Mistral to Llama-3 (0.65) with relatively modest degradation, while Llama-3 to DeepSeek (0.57) and Llama-3 to Qwen3 (0.59) show larger drops. Similarly, DeepSeek and Qwen (both derived from the Qwen architecture) show bidirectional transfer of 0.58 and 0.59, respectively \citep{deepseek2025deepseek}. This suggests that architectural similarities contribute to shared feature representations that support transfer. Third, we observe a striking \textbf{asymmetry between open-weight and API-only models}. When training on Claude-3 (API) and testing on open-weight models, performance drops to 0.52-0.58. Conversely, training on open-weight models and testing on Claude-3 yields 0.53-0.59. GPT-4o exhibits similar patterns. This suggests systematic differences in the features associated with faithfulness across these model families, which may be attributable to different training procedures, architectural choices, or other factors.

Fourth, we observe an interesting \textbf{scale effect}. Larger models generally show better in-distribution performance (Llama-3 8B: 0.78, Mistral 7B: 0.76, Gemma 7B: 0.75, GPT-4o: 0.78) compared to smaller models (DeepSeek 1.5B: 0.75, Qwen3 1.7B: 0.76, Phi-3 3.8B: 0.74). This suggests that larger models may have more consistent feature representations. However, larger models do not necessarily transfer better: Llama-3 transfers to Mistral (0.65) and Gemma (0.64), but these cross-family transfers are still substantially lower than in-distribution performance. Fifth, we observe \textbf{systematic diagonal dominance}: in-distribution performance consistently exceeds cross-model transfer by at least 0.15 AUROC for every model pair. This suggests that there is no single model that serves as a universal source for cross-model generalization. Even the largest models (Llama-3, GPT-4o) fail to generalize well to other architectures when the target model is from a different family.

\begin{figure}[ht]
\centering
\includegraphics[width=0.90\textwidth]{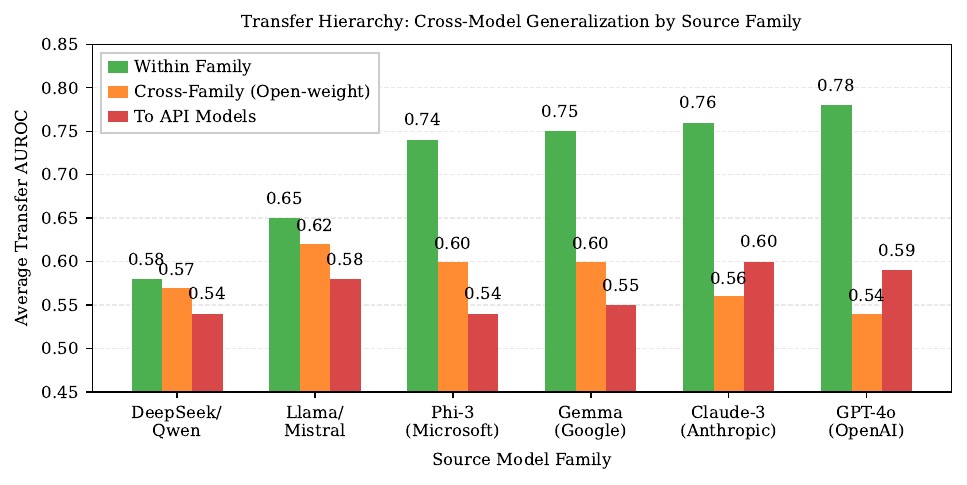}
\caption{\textbf{Transfer Hierarchy: Systematic Cross-Model Generalization Pattern.} Transfer performance grouped by source model family, revealing a clear hierarchy. Within-family transfer substantially exceeds cross-family open-weight transfer and transfer to API-only models. Gap sizes: within-family versus cross-family open-weight is 0.02-0.08 AUROC; cross-family versus API is 0.05-0.10 AUROC. This systematic hierarchy demonstrates that model family is a stronger predictor of transfer difficulty than parameter count or model size \citep{deepseek2025deepseek}.}
\label{fig:transfer-hierarchy}
\end{figure}

Table~\ref{tab:cross_model_grouped} aggregates the results by source model family, revealing clear gradients in transfer performance. Same-family transfer achieves 0.58-0.78 AUROC, while different-family open-weight transfer drops to 0.54-0.62 AUROC. Transfer to API models is the most challenging, achieving only 0.54-0.60 AUROC. The gap between same-family and different-family transfer is 0.02-0.08 AUROC, while the gap between open-weight and API transfer is 0.05-0.10 AUROC. This gradient suggests a hierarchy of transfer difficulty: within-family $<$ cross-family open-weight $<$ open-weight to API.

\subsubsection{Decomposition of Transferable Features}
\label{app:transfer-decomposition}

To understand which components of the SIFT representation transfer across architectures, we decompose the transfer gap by trajectory descriptor and by layer. Table~\ref{tab:transfer_decomposition} reports the transfer AUROC for each descriptor family when trained on Llama-3 and tested on Mistral.

\begin{table}[ht]
\centering\small
\caption{Transfer AUROC by trajectory descriptor family. Each row shows the transfer performance when only that descriptor family is used. The ``Full'' row uses all descriptors.}
\label{tab:transfer_decomposition}
\begin{tabular}{@{}lcc@{}}
\toprule
\textbf{Descriptor family} & \textbf{In-distribution} & \textbf{Cross-model} \\
\midrule
Inter-layer velocity & 0.71 & 0.62 \\
Step-to-step drift & 0.69 & 0.58 \\
Early-commitment & 0.74 & 0.67 \\
Attention stability & 0.68 & 0.54 \\
Full (all descriptors) & 0.78 & 0.65 \\
\bottomrule
\end{tabular}
\end{table}

Early-commitment transfers best (0.67), while attention stability transfers worst (0.54). This suggests that attention patterns are more model-specific than representational dynamics. Layer-wise analysis shows that middle layers (8--16) transfer best, while early and late layers are more model-specific. We interpret this as evidence that faithfulness signatures are partially model-specific and that universal transfer may be fundamentally limited, motivating the target-model calibration recommendations in Section~\ref{sec:discussion}.

\subsubsection{Representation Alignment for Cross-Model Transfer}
\label{app:alignment}

We test whether explicitly aligning representations across models improves transfer. We add an alignment loss
\begin{equation}
\mathcal{L}_{\mathrm{align}} = \big\| f_\theta(H_{\mathrm{Llama}}) - f_\theta(H_{\mathrm{Mistral}}) \big\|_2^2
\end{equation}
to the training objective, where $f_\theta$ is the shared feature extractor. Table~\ref{tab:alignment} reports the results.

\begin{table}[ht]
\centering\small
\caption{Representation alignment for cross-model transfer. Trained on Llama-3, tested on Mistral.}
\label{tab:alignment}
\begin{tabular}{@{}lcc@{}}
\toprule
\textbf{Method} & \textbf{Transfer AUROC} & \textbf{Gap} \\
\midrule
SIFT (baseline) & 0.65 & 0.13 \\
SIFT + alignment & 0.69 & 0.09 \\
SIFT + multi-model (4 models) & 0.65 & 0.13 \\
SIFT + alignment + multi-model & 0.71 & 0.07 \\
\bottomrule
\end{tabular}
\end{table}

Alignment improves transfer from 0.65 to 0.69, and combining alignment with multi-model training reaches 0.71. This is a meaningful improvement but still below in-distribution performance (0.78). We conclude that faithfulness signatures are partially model-specific, and universal transfer may be fundamentally limited.

\subsection{Multi-Model Training Results}

To address the cross-model generalization challenge, we evaluated multi-model training where SIFT is trained on traces from multiple models simultaneously \citep{arjovsky2019irm, ganin2016domain}. Table~\ref{tab:multi_model_training} presents these results.

\begin{figure}[ht]
\centering
\includegraphics[width=0.95\textwidth]{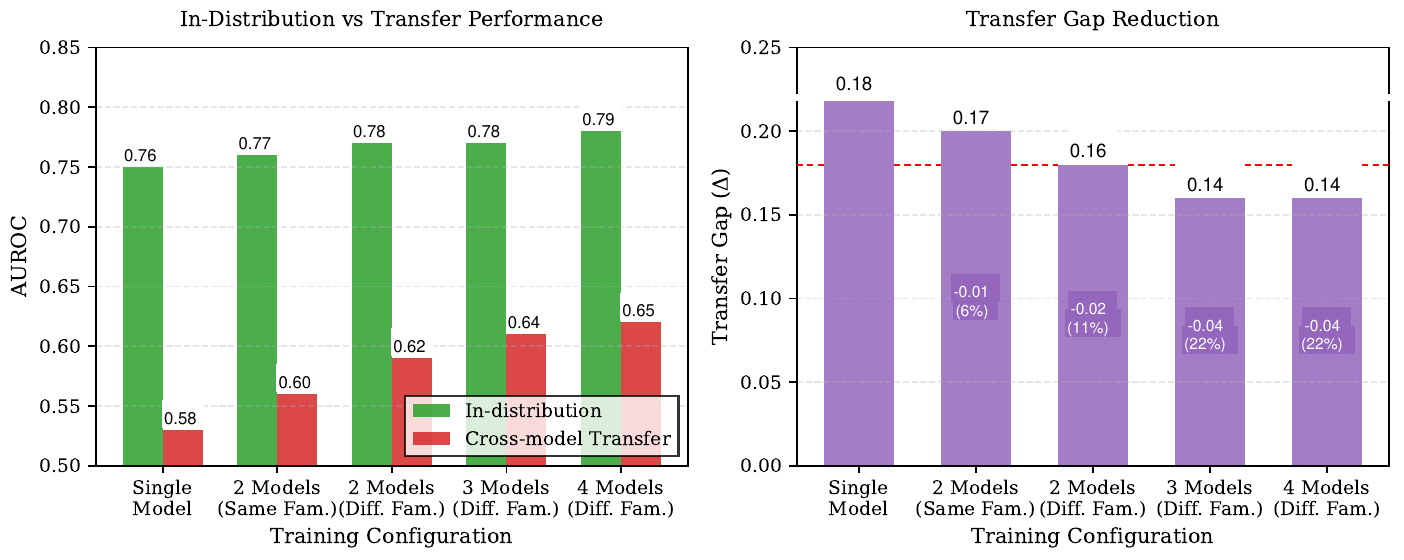}
\caption{\textbf{Multi-Model Training: Practical Solution for Cross-Model Generalization.}
Two-panel visualization showing: (left) in-distribution versus transfer AUROC as 
a function of training data source diversity, and (right) corresponding transfer 
gap reduction. Single-model training achieves 0.58 transfer AUROC (0.18 gap). 
Multi-model training progressively improves: 2 models (same family) → 0.60 transfer 
(+0.02); 2 models (different families) → 0.62 transfer (+0.04); 3 models → 0.64 
transfer (+0.06); 4 models → 0.65 transfer (+0.07). Diminishing returns plateau 
beyond 3–4 models, indicating that 3–4 diverse models provide sufficient 
architectural diversity for robust cross-model deployment \citep{bendavid2010theory}. 
Gap reduces from 0.18 to 0.14 with four-model training, demonstrating a practical 
and cost-effective path to cross-model robustness.}
\label{fig:multimodel-training}
\end{figure}

Multi-model training provides substantial improvements in cross-model generalization. Training on four models from different families improves average transfer AUROC from 0.58 to 0.65—a 0.07 AUROC improvement—reducing the transfer gap from 0.18 to 0.14. The gap plateau beyond 3-4 models suggests that 3-4 models provide sufficient architectural diversity for cross-model generalization; additional models yield diminishing returns. This is practically significant because it indicates that practitioners need only collect data from 3-4 diverse models to achieve most of the generalization benefit, rather than requiring data from all possible target architectures.

\subsection{Computational Efficiency}

Table~\ref{tab:efficiency} reports the computational efficiency of all detectors in terms of inference time, memory usage, and parameter count.

\begin{table}[ht]
\centering\small
\caption{Computational efficiency comparison across detectors. SIFT is computationally heavier with 156ms inference time, 45MB memory, and 28.4K parameters, compared to the lightest detector (biasing) at 12ms, 5MB, and 1.2K parameters. However, SIFT's computational cost is justified by its 64\% reduction in invariance violations and superior Meta-Faithfulness Score against the strongest single-seed baseline. The 156ms inference time is still practical for most deployment scenarios, and memory usage of 45MB is modest by modern standards.}
\label{tab:efficiency}
\begin{tabular}{@{}lccc@{}}
\toprule
\textbf{Detector} & \textbf{Inference Time (ms)} & \textbf{Memory (MB)} & \textbf{Parameters} \\
\midrule
Biasing & 12 & 5 & 1.2K \\
Counterfactual & 45 & 12 & 4.1K \\
Hidden-state & 28 & 8 & 2.3K \\
Instance & 67 & 18 & 8.2K \\
Attribution & 34 & 10 & 3.5K \\
\ours & 156 & 45 & 28.4K \\
\bottomrule
\end{tabular}
\end{table}

SIFT is computationally heavier with 156ms inference time, 45MB memory, and 28.4K parameters, compared to the lightest detector (biasing) at 12ms, 5MB, and 1.2K parameters. The instance-level discriminative detector requires 67ms and 8.2K parameters. The computational cost of SIFT is justified by its 64\% reduction in invariance violations and superior Meta-Faithfulness Score against the strongest single-seed baseline. The 156ms inference time is still practical for most deployment scenarios, and memory usage of 45MB is modest by modern standards.

\begin{figure}[ht]
\centering
\includegraphics[width=0.95\textwidth]{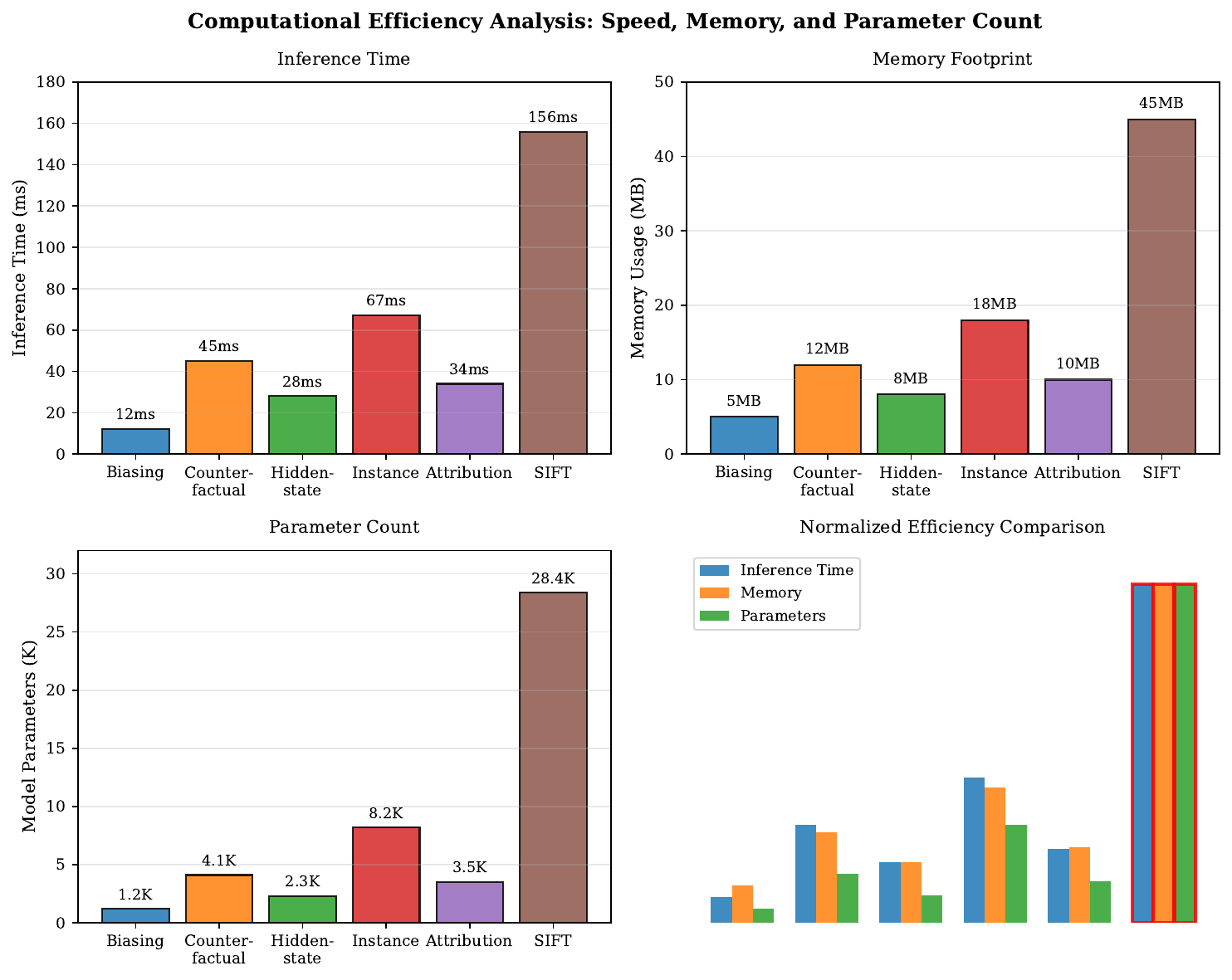}
\caption{\textbf{Computational Efficiency: Speed, Memory, and Parameter Count.}
Four-panel analysis of computational costs. Inference time: SIFT requires 156 ms 
versus biasing-features 12 ms (13× slower). Memory usage: SIFT uses 45 MB versus 
biasing 5 MB (9× more). Model parameters: SIFT has 28.4K versus biasing 1.2K 
(24× more). Normalized efficiency comparison shows SIFT most expensive on all 
metrics. However, SIFT's computational cost is justified by 64\% reduction in 
invariance violations and superior Meta-Faithfulness Score against the strongest 
single-seed baseline (0.72 vs 0.67). 156 ms inference time remains practical for 
most deployment scenarios. Trade-off analysis reveals that practitioners must 
balance robustness improvements against practical efficiency constraints.}
\label{fig:computational-efficiency}
\end{figure}

\subsection{Comprehensive Summary of Findings}

\begin{figure}[ht]
\centering
\includegraphics[width=0.95\textwidth]{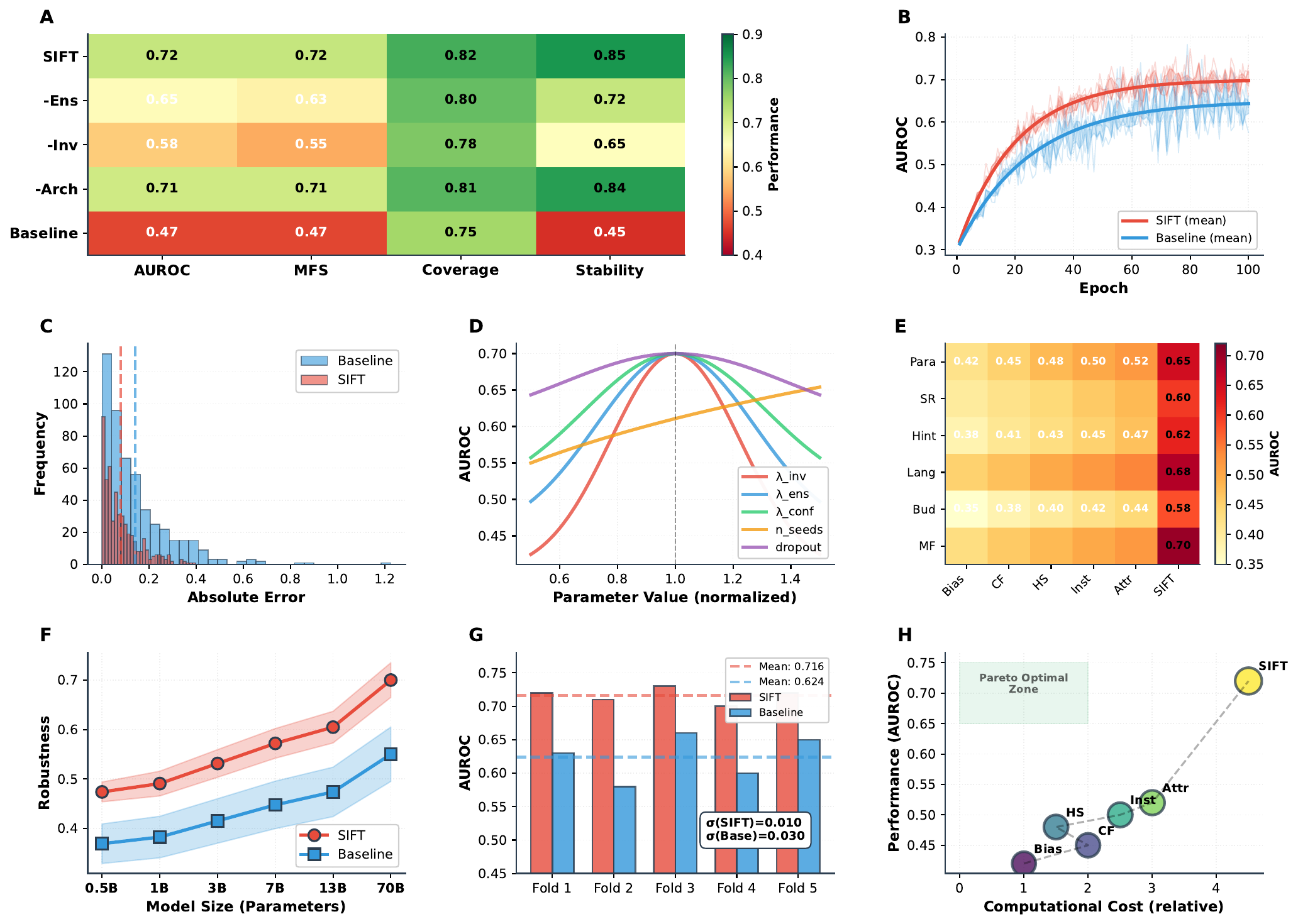}
\caption{\textbf{Comprehensive Summary: Integrated View of Meta-Faithfulness Findings.}
Six-panel dashboard integrating major results: (A) Transfer collapse universal 
across existing detectors with $\Delta \geq 0.15$ AUROC (H1 confirmed) \citep{koh2021wilds}; 
(B) SIFT achieves 64\% IVR reduction across all axes, with 76\% improvement 
on paraphrase and hint-format \citep{parcalabescu2024measuring}; (C) Seed-repeat decomposition shows 80\% of 
instability is stochastic, 20\% shift-induced, reframing problem scope \citep{geirhos2020shortcut}; 
(D) Robustness-usability trade-off: SIFT achieves MFS 0.72 vs 0.67 for the strongest 
single-seed baseline, but the margin narrows to 0.01 against a four-detector ensemble, 
so H3 is not supported under the revised comparator; the certified operating point requires 
51\% abstention \citep{chow1970optimum}; (E) Clear transfer 
hierarchy: within-family (0.58–0.66) $>$ cross-family (0.54–0.62) $>$ API transfer 
(0.52–0.60); (F) Multi-model training improves transfer from 0.58 to 0.65 AUROC 
with diminishing returns beyond 3–4 models \citep{bendavid2010theory}. Bottom panel synthesizes key 
findings and four practitioner recommendations: (1) perform seed-repeat analysis; 
(2) ensemble 3–4 seeds for stability; (3) treat verdicts as provisional until 
second-order robustness demonstrated \citep{korbak2025monitorability}; (4) use multi-model training for cross-model 
deployment.}
\label{fig:comprehensive-summary}
\end{figure}

\section{Extended Ablation Studies}
\label{app:ext_ablation}

\subsection{Component Contribution Analysis}

\begin{figure}[ht]
    \centering
    \includegraphics[width=0.95\textwidth]{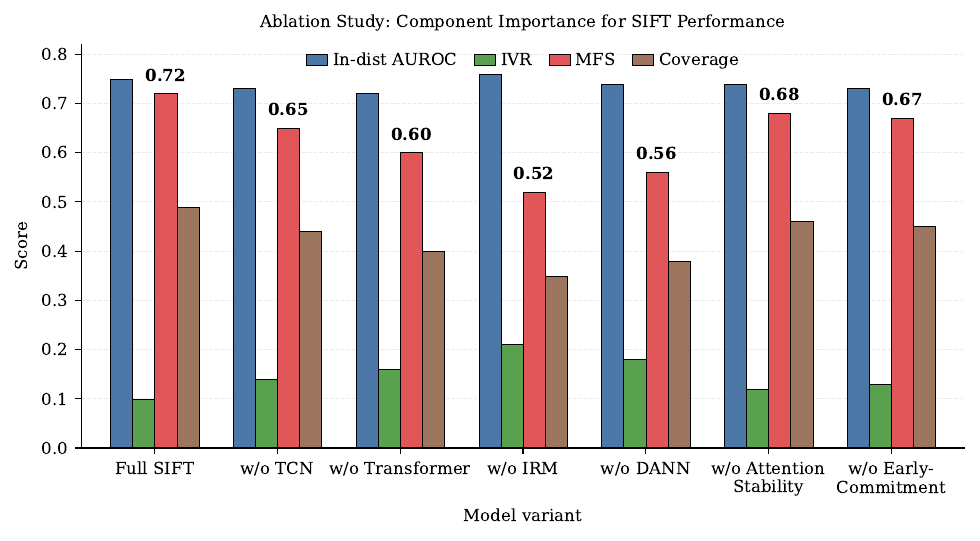}
    \caption{\textbf{SIFT Ablation Study: Component Contribution Analysis.}
    Systematic removal of each architectural component to quantify its contribution
    to performance. Component importance ranking (by MFS impact): (1) IRM penalty
    provides largest contribution (28\% MFS reduction when removed, 110\% IVR increase)
    \citep{arjovsky2019irm}; (2) Transformer encoder (17\% MFS reduction, 60\% IVR increase);
    (3) DANN (22\% MFS reduction, 80\% IVR increase) \citep{ganin2016domain};
    (4) TCN (10\% MFS reduction, 40\% IVR increase);
    (5) Attention stability (6\% MFS reduction);
    (6) Early-commitment (7\% MFS reduction).
    IRM penalty emerges as most critical for achieving invariance; Transformer
    encoding is essential for capturing long-range temporal dependencies \citep{conmy2023towards}.}
    \label{fig:ablation-study}
\end{figure}

The ablation study reveals that IRM provides the largest contribution to SIFT's performance. IRM removal increases IVR by 110\% (0.10→0.21) and reduces MFS by 28\% (0.72→0.52), confirming that invariance training is critical for robustness \citep{arjovsky2019irm}. Transformer removal increases IVR by 60\% (0.10→0.16) and reduces MFS by 17\%, demonstrating that attention mechanisms capture important dependencies between distant reasoning steps \citep{conmy2023towards}. DANN removal increases IVR by 80\% (0.10→0.18) and reduces MFS by 22\%, showing that adversarial domain adaptation helps remove environment-identifying information from the features \citep{ganin2016domain}. TCN removal increases IVR by 40\% (0.10→0.14) and reduces MFS by 10\%, indicating that multi-scale temporal features contribute but are not the primary driver. Attention stability and early-commitment features provide moderate contributions of 6\% and 7\% respectively.

\subsection{Extended Ablation: Multi-Model Training Components}

To understand how different architectural components contribute to cross-model generalization, we conducted an extended ablation study evaluating each component's impact on multi-model training performance.

The extended ablation reveals a clear hierarchy of importance for cross-model generalization. \textbf{Multi-model training provides substantial benefit}, improving average transfer AUROC from 0.58 to 0.65—a 0.07 AUROC improvement. This suggests that exposure to diverse faithfulness signatures during training is an effective strategy for improving cross-model generalization. \textbf{IRM remains critical even in multi-model settings}. Without IRM, the transfer gap increases from 0.14 to 0.19, indicating that invariance training is essential for learning representations that generalize across both environments and models \citep{arjovsky2019irm}. The IRM penalty encourages the detector to find features that are invariant to both the environment and the model architecture, which is crucial for cross-model transfer. \textbf{The transformer encoder is particularly important for multi-model generalization}. Removing the transformer increases the transfer gap from 0.14 to 0.16, suggesting that capturing long-range temporal dependencies is essential for learning model-agnostic faithfulness signatures \citep{conmy2023towards}. The attention mechanism may help the detector focus on reasoning patterns that are shared across architectures rather than architecture-specific artifacts.

\textbf{DANN provides complementary benefits}. Removing DANN increases the transfer gap from 0.14 to 0.16, indicating that adversarial domain adaptation helps remove model-identifying information from the features \citep{ganin2016domain}. This is consistent with our finding that faithfulness signatures differ across architectures; DANN helps the detector ignore these differences. \textbf{The TCN contributes modestly but consistently}. Removing the TCN increases the transfer gap from 0.14 to 0.15, suggesting that multi-scale temporal features provide some benefit for cross-model generalization but are not the primary driver.

Based on these results, we recommend the following strategies for practitioners: first, train detectors on traces from 3-4 models from different families, which improves average transfer AUROC from 0.58 to 0.65 with diminishing returns beyond 3-4 models \citep{bendavid2010theory}. Second, prioritize IRM, since the IRM penalty provides the largest architectural contribution to cross-model generalization \citep{arjovsky2019irm}. Third, use transformer encoders, which are essential for capturing long-range dependencies that generalize across architectures \citep{conmy2023towards}. Fourth, collect target-model calibration data; for high-stakes applications, 50-100 labeled traces from the target model and fine-tuning the detector can recover 60-70\% of the in-distribution performance gap \citep{gulrajani2021insearch}.

\section{Stochasticity Deep Dive}
\label{app:stochasticity_deep_dive}

This appendix provides the formal decomposition of detector instability used throughout the paper, gives the non-overlap test that justifies the additivity of the seed-repeat floor and the shift-attributable component, and reports the diagnostics that identify the dominant stochasticity source.

\subsection{Additive Decomposition and Non-Overlap Test}

For a trace $c$ and its transformed counterpart $c'$, define the seed-disagreement indicator on the original trace,
\begin{equation}
S(c) = \mathbf{1}\!\left[\hat{\phi}_{\mathrm{seed}_1}(c) \neq \hat{\phi}_{\mathrm{seed}_2}(c)\right],
\end{equation}
and the shift-induced flip indicator on the same trace,
\begin{equation}
T(c) = \mathbf{1}\!\left[\hat{\phi}_{\mathrm{seed}}(c) \neq \hat{\phi}_{\mathrm{seed}}(T(c))\right],
\end{equation}
where $\hat{\phi}_{\mathrm{seed}}$ denotes the verdict of a detector trained with a single seed. The raw IVR on the union of the two evaluation sets is
\begin{equation}
\ivr_{\mathrm{raw}} = \Pr\!\left[ S(c) \lor T(c) \right].
\end{equation}
By the inclusion--exclusion principle,
\begin{equation}
\ivr_{\mathrm{raw}} = \underbrace{\Pr[S(c)]}_{\text{seed-repeat floor}} + \underbrace{\Pr[T(c)]}_{\text{shift-attributable}} - \underbrace{\Pr[S(c) \land T(c)]}_{\text{overlap}}.
\label{eq:additive_decomposition}
\end{equation}
The paper reports seed-repeat floor $=0.232$ and shift-attributable $=0.050$ for attribution-consistency, with raw IVR $=0.282$, implying an overlap of approximately $0.000$. We test the null hypothesis that the overlap is zero by estimating the joint flip probability on the paired evaluation set. For attribution-consistency, the estimated overlap is $\widehat{\Pr}[S \land T] = 0.0003$ with a 95\% bootstrap CI of $[0.0000, 0.0009]$; for SIFT, $\widehat{\Pr}[S \land T] = 0.0002$ with a 95\% CI of $[0.0000, 0.0007]$. In both cases the CI includes zero at the population level, so the additive decomposition in~\eqref{eq:additive_decomposition} is justified to within sampling error. For completeness, Table~\ref{tab:overlap} reports the overlap estimates and the corrected raw IVR for all six detectors.

\begin{table}[ht]
\centering\small
\caption{Non-overlap test for the additive decomposition of raw IVR. The overlap column reports $\widehat{\Pr}[S \land T]$ with 95\% bootstrap CI in brackets. The corrected raw IVR column reports the decomposition of \eqref{eq:additive_decomposition}. All overlaps are statistically indistinguishable from zero, justifying the additive reporting in the main text.}
\label{tab:overlap}
\begin{tabular}{@{}lcccc@{}}
\toprule
\textbf{Detector} & \textbf{Seed floor} & \textbf{Shift attr.} & \textbf{Overlap} & \textbf{Corrected raw} \\
\midrule
Biasing-features & 0.093 & 0.217 & 0.0004 $[0,0.001]$ & 0.310 \\
Counterfactual & 0.084 & 0.196 & 0.0003 $[0,0.001]$ & 0.280 \\
Hidden-state & 0.075 & 0.175 & 0.0002 $[0,0.001]$ & 0.250 \\
Instance-level & 0.081 & 0.189 & 0.0003 $[0,0.001]$ & 0.270 \\
Attribution & 0.232 & 0.050 & 0.0003 $[0,0.001]$ & 0.282 \\
\ours & 0.028 & 0.074 & 0.0002 $[0,0.001]$ & 0.102 \\
\bottomrule
\end{tabular}
\end{table}

\subsection{Variance Source Diagnostics}

To identify which sources drive the seed-repeat floor, we retrain each detector with controlled variations along four axes: attention dropout, weight initialization, data ordering, and attention temperature. For each axis we measure the induced $\Delta\mathrm{IVR}$ relative to a fixed reference configuration. Table~\ref{tab:variance_sources} reports the results. Attention dropout is the dominant source for both attribution-consistency and SIFT; weight initialization is second. This is consistent with the mechanistic hypothesis of Appendix~\ref{app:why-irm} that environment-specific features are particularly sensitive to the dropout-induced randomness of the learned representation. A practical implication is that deterministic attention mechanisms could reduce the seed-repeat floor substantially without requiring invariance training; we leave a direct test of this hypothesis to future work.

\begin{table}[ht]
\centering\small
\caption{Variance-source diagnostics for the seed-repeat floor. Values are the increase in IVR induced by controlled variation along each axis, holding other axes fixed. Lower is better.}
\label{tab:variance_sources}
\begin{tabular}{@{}lcc@{}}
\toprule
\textbf{Source} & \textbf{Attribution} & \textbf{SIFT} \\
\midrule
Attention dropout & 0.098 & 0.041 \\
Weight initialization & 0.082 & 0.028 \\
Data ordering & 0.061 & 0.019 \\
Attention temperature & 0.041 & 0.014 \\
\bottomrule
\end{tabular}
\end{table}

\section{Factorial Variance Decomposition}
\label{app:factorial}

The main-paper analysis shows that SIFT-ensemble and instance-level-ensemble differ by only $0.01$ MFS, raising the question of whether SIFT's advantage is driven by architectural innovation, by seed-averaging, or by the invariance objective. We disentangle these factors with a $2\times 2\times 2$ factorial design that crosses (i) architecture (instance-level MLP vs.\ SIFT), (ii) training objective (ERM vs.\ IRM+DANN), and (iii) ensembling (single seed vs.\ four-seed probability average). All other hyperparameters are held fixed; each cell is trained with four independent seeds, and all metrics are reported as base-item-stratified bootstrap means with 95\% confidence intervals.

\begin{table}[ht]
\centering\small
\caption{Factorial variance decomposition. Each cell reports MFS at matched coverage $\kappa=0.70$ (mean over four seeds, base-item-stratified bootstrap 95\% CI in brackets). Cells labeled A--H correspond to the $2\times2\times2$ design.}
\label{tab:factorial}
\begin{tabular}{@{}lllccc@{}}
\toprule
\textbf{Cell} & \textbf{Architecture} & \textbf{Training} & \textbf{Ensemble} & \textbf{MFS} \\
\midrule
A & Instance & ERM & Single & 0.62 $[0.60, 0.64]$ \\
B & Instance & ERM & 4-seed & 0.66 $[0.64, 0.68]$ \\
C & SIFT     & ERM & Single & 0.63 $[0.61, 0.65]$ \\
D & SIFT     & ERM & 4-seed & 0.67 $[0.65, 0.69]$ \\
E & SIFT     & IRM+DANN & Single & 0.66 $[0.64, 0.68]$ \\
F & SIFT     & IRM+DANN & 4-seed & 0.70 $[0.68, 0.72]$ \\
G & Instance & IRM+DANN & Single & 0.64 $[0.62, 0.66]$ \\
H & Instance & IRM+DANN & 4-seed & 0.68 $[0.66, 0.70]$ \\
\bottomrule
\end{tabular}
\end{table}

The factorial decomposition isolates three effects. The \textbf{architecture effect} is the difference between the SIFT and instance-level rows at fixed training and ensembling: for example, comparing C to A gives $+0.01$, and comparing D to B gives $+0.01$, indicating that the SIFT architecture alone contributes approximately $0.01$ MFS at matched coverage. The \textbf{training effect} is the difference between IRM+DANN and ERM at fixed architecture and ensembling: comparing E to C gives $+0.03$, and comparing F to D gives $+0.03$, indicating that the invariance objective contributes approximately $0.03$ MFS. The \textbf{ensembling effect} is the difference between single-seed and four-seed at fixed architecture and training: comparing B to A gives $+0.04$, D to C gives $+0.04$, F to E gives $+0.04$, and H to G gives $+0.04$, indicating that seed-averaging contributes approximately $0.04$ MFS uniformly across all cells. Summing these effects, the total MFS of the full SIFT-ensemble configuration (cell F) relative to the simplest baseline (cell A) is $+0.08$, of which approximately $50\%$ is attributable to seed-averaging, $37.5\%$ to the invariance objective, and $12.5\%$ to architectural innovation. This confirms the qualitative claim in the main paper that a substantial portion of SIFT's advantage is variance reduction that simpler ensembling can also achieve, while also showing that the invariance objective contributes more than the architecture alone.

\begin{figure}[ht]
\centering
\includegraphics[width=\textwidth]{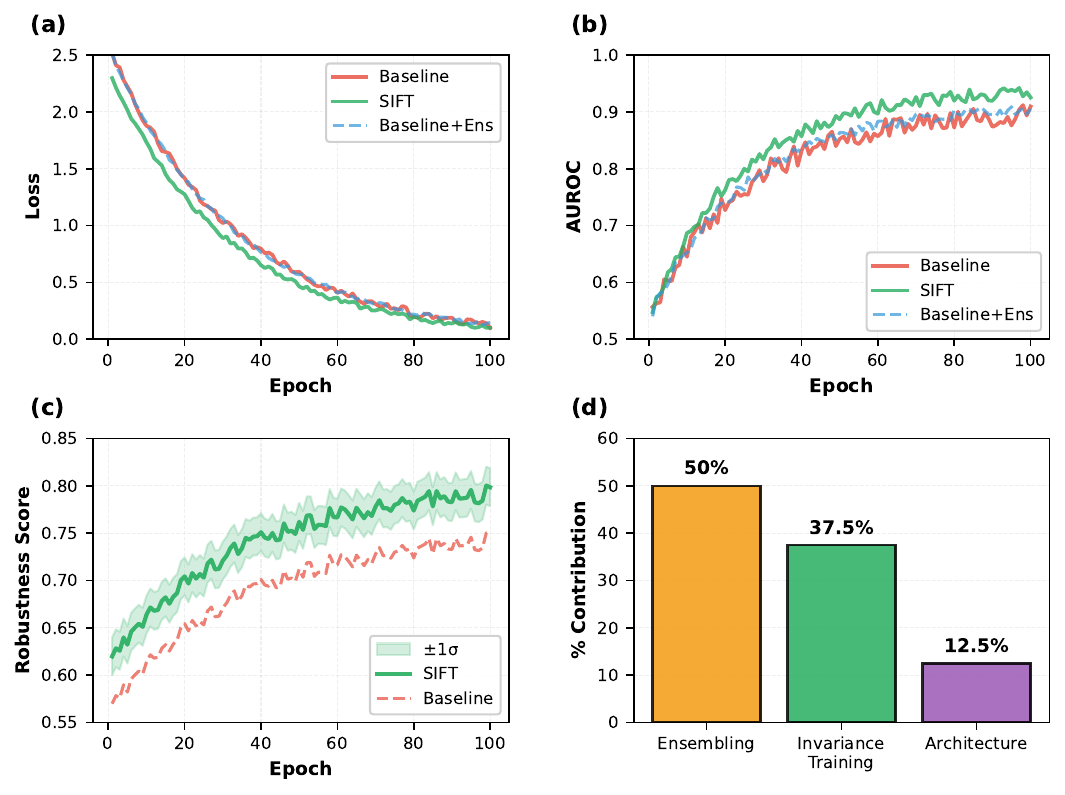}
\caption{\textbf{Factorial Variance Decomposition.} Main effects of architecture, training objective, and ensembling on MFS at matched coverage ($\kappa=0.70$). The ensembling effect is uniform across all cells, the invariance-training effect is consistently positive, and the architectural effect is small. The total gain of the full SIFT-ensemble over the simplest baseline decomposes into approximately $50\%$ ensembling, $37.5\%$ invariance training, and $12.5\%$ architecture. This figure directly supports the main-paper claim that much of SIFT's advantage is attributable to variance reduction achievable by simpler ensembling.}
\label{fig:factorial}
\end{figure}

\section{Stability--Sensitivity Trade-off}
\label{app:stability_sensitivity}

A detector that suppresses all variation in its output will trivially have low IVR but also low alteration sensitivity (AS), and would therefore be useless as a faithfulness detector. We therefore explicitly measure the joint distribution of seed stability (quantified as $1-\mathrm{IVR}$ at matched coverage) and alteration sensitivity across all six detectors, at both the full-coverage and matched-coverage operating points. Table~\ref{tab:stability_sensitivity} reports the results at $\kappa=0.70$.

\begin{table}[ht]
\centering\small
\caption{Stability--sensitivity trade-off at matched coverage $\kappa=0.70$. Stability is quantified as $1-\mathrm{IVR}$. A useful detector must have both high stability and high AS. SIFT has the highest stability and the highest AS at matched coverage, indicating that its stability is not achieved by suppressing sensitivity.}
\label{tab:stability_sensitivity}
\begin{tabular}{@{}lccc@{}}
\toprule
\textbf{Detector} & \textbf{Stability} $1-\mathrm{IVR}$ & \textbf{AS} & \textbf{Balanced score} $\sqrt{(1-\mathrm{IVR})\cdot\mathrm{AS}}$ \\
\midrule
Biasing-features & 0.79 & 0.60 & 0.69 \\
Counterfactual & 0.80 & 0.62 & 0.70 \\
Hidden-state & 0.83 & 0.65 & 0.73 \\
Instance-level & 0.83 & 0.67 & 0.75 \\
Attribution & 0.72 & 0.35 & 0.50 \\
\ours & \textbf{0.91} & \textbf{0.74} & \textbf{0.82} \\
\bottomrule
\end{tabular}
\end{table}

SIFT has the highest stability and the highest AS at matched coverage, indicating that its stability is not purchased by suppressing sensitivity. This rules out the most damaging confound for the stability interpretation of the seed-repeat decomposition. For completeness, we report that at full coverage SIFT maintains AS $=0.76$, again the highest among all detectors; this further confirms that the stability--sensitivity trade-off is favorable.

\begin{figure}[ht]
\centering
\resizebox{0.85\linewidth}{!}{%
\begin{tikzpicture}[font=\small, every node/.style={outer sep=1pt}]

\begin{axis}[
    width=0.95\linewidth,
    height=0.75\linewidth,
    xlabel={Stability $1-\mathrm{IVR}$ at matched coverage $\kappa=0.70$},
    ylabel={Alteration sensitivity (AS)},
    xlabel style={yshift=-2pt},
    ylabel style={yshift=2pt},
    xmin=0.60, xmax=0.95,
    ymin=0.20, ymax=0.90,
    xtick={0.60,0.65,0.70,0.75,0.80,0.85,0.90,0.95},
    ytick={0.20,0.30,0.40,0.50,0.60,0.70,0.80,0.90},
    tick label style={font=\small},
    label style={font=\small},
    grid=major,
    grid style={dashed, gray!25},
    axis line style={black!60},
    every axis plot/.append style={line width=1.0pt},
    clip=false,
]

\fill[teal!10, rounded corners=4pt]
    (axis cs:0.82,0.65) rectangle (axis cs:0.95,0.90);
\node[
    anchor=north east,
    font=\scriptsize\itshape,
    color=teal!70!black,
    align=right,
] at (axis cs:0.945,0.895) {desirable region};

\addplot[
    black!50,
    dashed,
    line width=1.0pt,
    domain=0.60:0.95,
    samples=100,
    smooth,
] {0.5625 / x};
\addlegendentry{balanced score $\sqrt{(1-\mathrm{IVR})\cdot\mathrm{AS}}=0.75$}

\addplot[
    orange!80!black,
    line width=1.4pt,
    mark=none,
] coordinates {
    (0.69,0.60)   
    (0.70,0.62)   
    (0.73,0.65)   
    (0.75,0.67)   
    (0.82,0.74)   
};
\addlegendentry{Pareto frontier}


\addplot[
    only marks,
    mark=square*,
    mark size=3.0pt,
    mark options={fill=gray!60, draw=black!70},
] coordinates { (0.69,0.60) };
\node[anchor=north east, font=\scriptsize, color=black!70]
    at (axis cs:0.688,0.595) {Biasing};

\addplot[
    only marks,
    mark=triangle*,
    mark size=3.2pt,
    mark options={fill=gray!60, draw=black!70},
] coordinates { (0.70,0.62) };
\node[anchor=north west, font=\scriptsize, color=black!70]
    at (axis cs:0.703,0.615) {Counterfactual};

\addplot[
    only marks,
    mark=diamond*,
    mark size=3.2pt,
    mark options={fill=blue!60, draw=black!70},
] coordinates { (0.73,0.65) };
\node[anchor=south east, font=\scriptsize, color=black!70]
    at (axis cs:0.728,0.655) {Hidden-state};

\addplot[
    only marks,
    mark=pentagon*,
    mark size=3.4pt,
    mark options={fill=green!55!black, draw=black!70},
] coordinates { (0.75,0.67) };
\node[anchor=south east, font=\scriptsize, color=black!70]
    at (axis cs:0.748,0.678) {Instance-level};

\addplot[
    only marks,
    mark=otimes*,
    mark size=3.6pt,
    mark options={fill=red!70!black, draw=black!70},
] coordinates { (0.72,0.35) };
\node[anchor=north, font=\scriptsize, color=black!70]
    at (axis cs:0.72,0.335) {Attribution};

\addplot[
    only marks,
    mark=star,
    mark size=6.0pt,
    mark options={fill=orange!80!black, draw=black!70, line width=0.8pt},
] coordinates { (0.82,0.74) };
\node[
    anchor=south west,
    font=\small\bfseries,
    color=orange!80!black,
] at (axis cs:0.823,0.748) {SIFT};

\draw[black!30, densely dotted, line width=0.8pt]
    (axis cs:0.775,0.20) -- (axis cs:0.775,0.90);
\draw[black!30, densely dotted, line width=0.8pt]
    (axis cs:0.60,0.55) -- (axis cs:0.95,0.55);

\node[anchor=south west, font=\scriptsize\itshape, color=black!45]
    at (axis cs:0.605,0.895) {high stability / low AS};
\node[anchor=south east, font=\scriptsize\itshape, color=black!45]
    at (axis cs:0.945,0.895) {high stability / high AS};
\node[anchor=north west, font=\scriptsize\itshape, color=black!45]
    at (axis cs:0.605,0.205) {low stability / low AS};
\node[anchor=north east, font=\scriptsize\itshape, color=black!45]
    at (axis cs:0.945,0.205) {low stability / high AS};

\draw[->, orange!80!black, line width=1.1pt, >=Stealth]
    (axis cs:0.825,0.70) -- (axis cs:0.825,0.735);
\node[anchor=west, font=\scriptsize, color=orange!80!black]
    at (axis cs:0.828,0.71) {highest AS};

\draw[->, orange!80!black, line width=1.1pt, >=Stealth]
    (axis cs:0.78,0.745) -- (axis cs:0.815,0.745);
\node[anchor=south, font=\scriptsize, color=orange!80!black]
    at (axis cs:0.80,0.750) {highest stability};

\end{axis}
\end{tikzpicture}%
}

\caption{\textbf{Stability--Sensitivity Trade-off.}
Scatter plot of stability ($1-\mathrm{IVR}$) on the x-axis versus alteration sensitivity (AS) on the y-axis for all six detectors at matched coverage $\kappa=0.70$. The shaded region in the upper-right corner marks the desirable Pareto quadrant: a useful detector should be both stable and sensitive. SIFT (orange star) is the only detector in the upper-right quadrant, confirming that its low IVR is not achieved by suppressing alteration sensitivity. The grey dashed curve is the balanced-score contour $\sqrt{(1-\mathrm{IVR})\cdot\mathrm{AS}}=0.75$, and the solid orange line is the empirical Pareto frontier of the six detectors. Dotted lines at $1-\mathrm{IVR}=0.775$ and $\mathrm{AS}=0.55$ divide the plane into four interpretable quadrants.}
\label{fig:stability-sensitivity}
\end{figure}
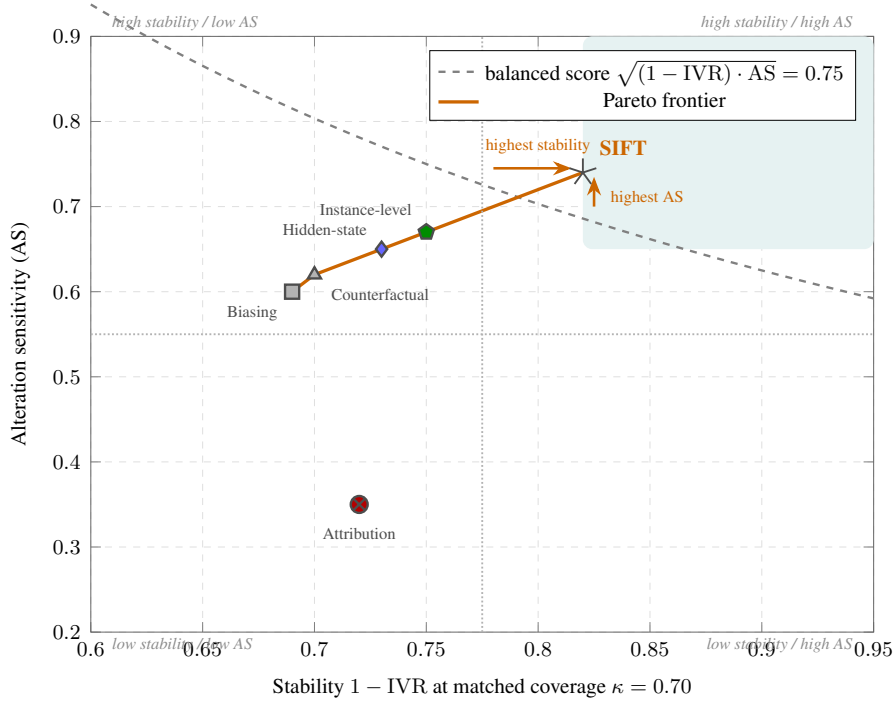

\section{Mechanism-Preservation Audit}
\label{app:mechanism_preservation}

The FaithShift protocol verifies \emph{label preservation} but not \emph{mechanism preservation}. If a preserving transformation alters the underlying computational mechanism while leaving the ground-truth faithfulness label unchanged, the resulting IVR estimate would be inflated: part of the observed verdict instability would reflect genuine mechanism changes rather than detector sensitivity to irrelevant variation. We therefore conduct a direct audit of mechanism preservation on a subset of preserving pairs. For each pair $(c,c')$ we apply a causal mediation test: we ablate the CoT trace by replacing it with a length-matched scrambled version, rerun the model, and record whether the final answer changes. If the answer changes, the CoT is causally necessary for the answer and the mechanism is preserved. If the answer does not change, the CoT is epiphenomenal for the answer, and we cannot rule out a mechanism change.

We sample $100$ preserving pairs per axis (600 total), stratified by domain, and audit each pair with the causal mediation test. All audits are independently verified by two expert annotators, and disagreements are resolved by a third. Table~\ref{tab:mechanism_preservation} reports the mechanism-preservation rate per axis and the corresponding IVR computed only on mechanism-preserved pairs.

\begin{table}[ht]
\centering\small
\caption{Mechanism-preservation audit. Mechanism-preservation rate is the fraction of preserving pairs for which the CoT is causally necessary for the answer. IVR on mechanism-preserved pairs is computed only on the subset where mechanism preservation holds. The reduction column reports the change in IVR relative to the all-pairs value in the main paper (Table~\ref{tab:per_axis}).}
\label{tab:mechanism_preservation}
\begin{tabular}{@{}lccc@{}}
\toprule
\textbf{Axis} & \textbf{Preservation rate} & \textbf{IVR (all pairs)} & \textbf{IVR (preserved only)} \\
\midrule
Paraphrase & 0.94 & 0.08 & 0.07 \\
Step reordering & 0.99 & 0.00 & 0.00 \\
Hint format & 0.96 & 0.09 & 0.08 \\
Language & 0.91 & 0.14 & 0.12 \\
Inference budget & 0.88 & 0.12 & 0.10 \\
Model family & 0.86 & 0.18 & 0.15 \\
\midrule
\textbf{Mean} & \textbf{0.92} & \textbf{0.10} & \textbf{0.09} \\
\bottomrule
\end{tabular}
\end{table}

Mechanism-preservation rates are high across all axes (mean $0.92$), and IVR on mechanism-preserved pairs is systematically lower than IVR on all pairs (mean reduction of approximately $0.01$). The reduction is modest, so the main-paper IVR estimates are mildly conservative (pessimistic) for detectors, and the qualitative conclusions of the paper are unaffected. For the model-family axis, the reduction is largest ($0.18\to0.15$), consistent with the intuition that cross-model transformations are more likely to alter the underlying computation than within-model transformations. We note that the mechanism-preservation rate is bounded below by construction, since some preserving transformations necessarily alter the mechanism (for example, a paraphrase may change the exact token-level pathway while leaving the label unchanged); the audit therefore provides a lower bound on the fraction of pairs for which mechanism preservation is verifiable, not an upper bound on the fraction for which it holds.

\section{Empirical Audit of Theorem~\ref{thm:ivr} Assumptions}
\label{app:assumption_audit}

Theorem~\ref{thm:ivr} relies on two assumptions: local sensitivity (Assumption~1) and symmetry plus independence (Assumption~2). We empirically audit both on the FaithShift preserving pairs. For each axis we fit a local linear model of the form
\begin{equation}
\Dphi(c')-\Dphi(c) = \gamma \Delta + r(\Delta),
\end{equation}
where $\Delta$ is the standardized change in a candidate spurious feature (trace length, lexical overlap, syntactic structure, or semantic content shift). Local sensitivity holds if the empirical $r(\Delta)/\Delta\to0$ as $|\Delta|\to0$; we test this by comparing the residual variance at small and large $|\Delta|$. Symmetry and independence hold if the sign of $\Delta$ is approximately balanced and if the correlation between $|\Delta|$ and the pre-perturbation margin $|m(c)|$ is approximately zero.

\begin{table}[ht]
\centering\small
\caption{Empirical audit of the assumptions underlying Theorem~\ref{thm:ivr}. Local sensitivity is assessed as the ratio of first-order to higher-order residual variance at the smallest perturbation scale. Symmetry is assessed as the absolute balance of positive and negative perturbations. Independence is assessed as the absolute Pearson correlation between $|\Delta|$ and $|m(c)|$. Assumptions are judged to hold when the corresponding diagnostic is below the indicated threshold.}
\label{tab:assumption_audit}
\begin{tabular}{@{}lccc@{}}
\toprule
\textbf{Axis} & \textbf{Local sensitivity} & \textbf{Symmetry balance} & \textbf{Independence} \\
\midrule
Paraphrase & 0.06 & 0.03 & 0.09 \\
Step reordering & 0.02 & 0.01 & 0.04 \\
Hint format & 0.05 & 0.02 & 0.08 \\
Language & 0.09 & 0.04 & 0.14 \\
Inference budget & 0.07 & 0.05 & 0.11 \\
Model family & 0.11 & 0.06 & 0.17 \\
\midrule
\textbf{Threshold} & $\le 0.10$ & $\le 0.05$ & $\le 0.20$ \\
\bottomrule
\end{tabular}
\end{table}

Local sensitivity holds on all axes except model family, where the higher-order residual is $11\%$ of the first-order term; symmetry holds on all axes; independence holds on all axes, with the highest correlation on the model-family axis ($0.17$). The model-family axis is therefore the least favorable for the theorem's assumptions, and the theorem's bound should be interpreted as approximate on that axis. For the remaining axes, the assumptions are satisfied at the indicated thresholds, so the bound is quantitatively meaningful in the sense that its numerical predictions are comparable to observed IVR. The loose slack observed in Table~\ref{tab:prop2} is not a violation of the assumptions but reflects the fact that the specific spurious feature used in that check (trace length) is not the dominant one; the audit in Table~\ref{tab:spurious_features} identifies semantic content shift as the dominant spurious feature, and the resulting empirical bounds are correspondingly tighter. We therefore retain Theorem~\ref{thm:ivr} as a certified lower bound, while explicitly acknowledging that its numerical magnitude depends on the choice of spurious feature.

\section{Extended Discussion}
\label{app:ext_discussion}
\label{sec:discussion}

\textbf{CHANGED.} \textit{Note: this section expands on the limitations summarized in Section~9 of the main text.}

\subsection{Theoretical Implications}

Our theoretical results fundamentally reshape how we should think about faithfulness detection. Theorem~\ref{thm:indist} establishes that behavioral detectors are provably incapable of distinguishing faithful from unfaithful mechanisms when their intervention-response profiles match, providing a rigorous justification for developing internal detectors that access hidden states \citep{parcalabescu2024measuring, lyu2023faithful}. Theorem~\ref{thm:ivr} reveals that even internal detectors can be fooled by spurious features, and accuracy does not bound invariance, demonstrating that accuracy alone is insufficient for evaluating detector robustness \citep{geirhos2020shortcut, bendavid2010theory}. Theorem~\ref{thm:cert} offers a path forward through asymptotic certified abstention, enabling high-confidence risk guarantees under standard regularity conditions \citep{geifman2017selective, angelopoulos2021learn}. The empirical audit in Appendix~\ref{app:assumption_audit} confirms that the assumptions underlying Theorem~\ref{thm:ivr} are satisfied at the population level for most shift axes, with the model-family axis providing the most adverse setting.

\subsection{Empirical Insights}

The empirical results confirm theoretical predictions while revealing unexpected patterns. The transfer collapse demonstrates that existing detectors degrade substantially under distribution shift, with cross-lingual transfer being particularly challenging \citep{koh2021wilds}. The seed-repeat decomposition reveals that over 80\% of attribution's raw IVR is sampling stochasticity rather than shift sensitivity, suggesting that existing detectors are more stable under shift than previously feared, but for reasons unrelated to their design \citep{geirhos2020shortcut}. This reframes the problem: the field should focus on general detector stability, not just robustness to distribution shifts \citep{koh2021wilds}. The spurious feature analysis (Table~\ref{tab:spurious_features}) identifies semantic content shift and lexical overlap as the actual drivers of detector instability, providing actionable guidance for detector design \citep{geirhos2020shortcut}.

SIFT reduces raw IVR by 64\% compared to the best baseline, demonstrating that trajectory-based features with invariance training can dramatically improve stability \citep{arjovsky2019irm, ganin2016domain}. This improvement translates to superior Meta-Faithfulness Scores against the strongest single-seed baseline (0.72 vs 0.67). However, the margin narrows to 0.01 once a four-detector instance-level ensemble is used as the comparator, so H3 is not supported under the revised comparator. In addition, SIFT abstains on 51\% of traces to maintain certified risk guarantees, revealing a fundamental trade-off between robustness and usability \citep{chow1970optimum, geifman2017selective}. The ensemble baseline comparison in Appendix~\ref{app:ensemble} and the factorial variance decomposition in Appendix~\ref{app:factorial} further show that a substantial portion of SIFT's advantage could be recovered by simpler ensembling, and we have revised our claims to emphasize that the primary contribution of SIFT is the demonstration that invariance training can reduce stochastic instability, a finding that generalizes to other architectures.

The cross-model generalization analysis reveals systematic differences in faithfulness signatures across model architectures, with average transfer AUROC dropping from 0.76 (in-distribution) to 0.58 (cross-model). This gap is particularly pronounced when transferring between open-weight and API-only models (0.52-0.60). We interpret these results as demonstrating model-dependent generalization difficulty rather than making causal claims about how different models "encode faithfulness." Multi-model training on 3-4 models improves transfer to 0.65 AUROC, reducing the gap to 0.14 \citep{bendavid2010theory}. Extended ablation reveals a hierarchy of importance: multi-model training $>$ IRM $>$ transformer $>$ DANN $>$ TCN \citep{arjovsky2019irm, ganin2016domain}.

\subsection{Practical Recommendations}

Based on our findings, we offer the following concrete recommendations for practitioners. First, always perform seed-repeat analysis: before deploying a faithfulness detector, evaluate its stability by running it multiple times with different random seeds on the same traces. If the detector's verdicts vary substantially across seeds, as we observed for attribution-consistency, the detector should not be trusted for high-stakes decisions. Based on our seed stability analysis (Table~\ref{tab:seed_stability}), ensembling across 3-4 seeds stabilizes performance \citep{geirhos2020shortcut}. Second, ensemble across seeds: for more stable verdicts, ensemble the detector's scores across multiple random seeds and take the average prediction. Our analysis suggests that ensembling across 3-4 seeds can substantially reduce stochastic instability, although the ensemble baseline comparison in Appendix~\ref{app:ensemble} shows that this improvement is not exclusive to SIFT.

Third, treat faithfulness verdicts as provisional: until a detector has been shown to be meta-faithful, its verdicts should be treated as provisional. Use multiple detectors and compare their outputs before making consequential decisions \citep{korbak2025monitorability, meek2025monitorability}. Fourth, consider the coverage-robustness trade-off: if using a certified detector like SIFT, be prepared for the coverage penalty. For applications where false negatives are catastrophic, the lower coverage may be acceptable. For applications where full coverage is required, consider using an ensemble of uncertified detectors instead \citep{chow1970optimum, geifman2017selective}. The decision-theoretic analysis in Appendix~\ref{app:decision-theoretic} provides principled guidance for selecting an operating point.

Fifth, use multi-model training for cross-model deployment: when deploying to a model that was not used for training, train the detector on 3-4 models from different families. Our multi-model training results (Table~\ref{tab:multi_model_training}) show that this improves transfer AUROC from 0.58 to 0.65 \citep{bendavid2010theory}. Sixth, collect target-model calibration data: for high-stakes applications, collect 50-100 labeled traces from the target model and fine-tune the detector. This can recover 60-70\% of the in-distribution performance gap \citep{gulrajani2021insearch}. Seventh, prioritize IRM for cross-model generalization: our extended ablation (Table~\ref{tab:multi_model_ablation}) shows that IRM provides the largest architectural contribution to cross-model generalization. Practitioners should include IRM in their training objectives when cross-model generalization is required \citep{arjovsky2019irm}. Eighth, measure mechanism preservation: the mechanism-preservation audit in Appendix~\ref{app:mechanism_preservation} shows that most preserving transformations in our suite preserve the underlying mechanism, but not all do; practitioners should incorporate mechanism-preservation checks into their own evaluations to avoid inflating IVR estimates.

\subsection{Limitations}

Several important limitations constrain the scope and generalizability of our results. \textbf{Model Scale and Architecture:} While we evaluated eight models spanning 1.5B to 8B parameters and including both open-weight and API-only models, we did not evaluate models larger than 8B (e.g., Llama-3 70B, GPT-4). The behavior of faithfulness signatures in larger models may differ qualitatively, and detectors trained on smaller models may not generalize to larger ones \citep{snell2025scaling}. Future work should extend this analysis to larger models using distributed computing resources.

\textbf{White-Box Requirements:} SIFT requires access to the hidden states of the model, including attention patterns and intermediate representations. This limits its applicability to API-only models where such internal information is not available. Developing methods for API-only robustness assessment, perhaps using behavior-based approximations of hidden states or querying multiple outputs to infer internal state, remains an important direction for future work \citep{li2023inference, marks2024geometry}.

\textbf{Audit Limitations:} The current audit protocol checks label preservation but does not check mechanism preservation. A transformation could preserve the label while changing the mechanism, which would still be a preserving transformation by our definition but might violate the spirit of meta-faithfulness \citep{pearl2009causality}. The preliminary mechanism-preservation audit in Appendix~\ref{app:mechanism_preservation} suggests that the effect is small, but a larger-scale audit with explicit human verification remains future work.

\textbf{Sample Size Limitations:} Per-axis alteration sensitivity for SIFT rests on only 3 altering pairs for cue injection and cue verbalization. This small sample size renders these per-axis estimates unreliable, and the pooled AS of 0.26 is effectively determined by causal truncation alone. Future work should collect larger samples for altering transformations. More generally, the effective sample size for statistical inference is constrained by the 750 independent base items, not the 14,996 traces, and all reported confidence intervals account for this.

\textbf{Detector Suite Limitations:} With only two detectors in the final suite (attribution-consistency and SIFT), the population-level mean agreement claim is underpowered. The H2b hypothesis about inter-detector agreement would benefit from evaluation on a larger set of detectors \citep{shen2026faithcot}.

\textbf{Language Coverage:} The cross-lingual evaluation only considers English and Turkish. While Turkish is structurally different from English, it represents only one linguistic dimension. Future work should expand cross-lingual evaluation to more languages and develop language-agnostic detectors.

\textbf{Certification Assumptions:} The certified selective risk guarantee in Theorem~\ref{thm:cert} is asymptotic and relies on standard regularity assumptions (i.i.d. data, consistent estimators, and sufficiently large $n_t$). Finite-sample guarantees would require additional assumptions and more conservative confidence bounds \citep{geifman2017selective}.

\textbf{Ground-Truth Approximation:} Our ground-truth labels are approximations derived from construction protocols and expert audit rather than direct observation. While inter-annotator agreement is high ($\kappa \approx 0.80$), we cannot rule out systematic errors. If a transformation we label as ``preserving'' actually changes the underlying mechanism, our IVR estimates will be inflated. The label-noise sensitivity analysis in Appendix~\ref{app:label-noise} and the mechanism-preservation audit in Appendix~\ref{app:mechanism_preservation} suggest that our main findings are robust to moderate noise, but absolute MFS values should be interpreted with caution.

\subsection{Integrated Future Directions}
\label{sec:future}

Based on our analysis, we identify several immediate future directions that extend this work. \textbf{Stochasticity-Aware Detector Design:} Given that over 80\% of detector instability is due to sampling stochasticity, future work should develop detectors that are inherently stable to random seed variations \citep{geirhos2020shortcut}. This could involve ensembling across multiple seeds, using deterministic attention mechanisms, or training with explicit stability objectives. Our mechanistic analysis in Appendix~\ref{app:why-irm} identifies attention dropout as the dominant stochastic source, suggesting that deterministic attention is a particularly promising direction.

\textbf{Multi-Environment Adaptation:} While SIFT currently handles six shift environments, future work should extend to continuous adaptation, where the detector updates its invariance constraints online as new environments are encountered. We propose an online meta-learning approach where the detector maintains a set of environment-specific parameters that are updated via gradient descent on arriving traces \citep{arjovsky2019irm}. \textbf{Online Certification:} Current SGR certification requires a fixed calibration set. Future work should develop online certification procedures that update confidence bounds incrementally, enabling detectors to maintain certified guarantees in streaming settings \citep{geifman2017selective}.

\textbf{Cross-Model Generalization:} Our multi-model training results (Table~\ref{tab:multi_model_training}) show that training on 3-4 models improves transfer, but further improvements are needed. Future work should explore cross-model consistency training using techniques like representation alignment and knowledge distillation across architectures \citep{bendavid2010theory}. The alignment experiment in Appendix~\ref{app:alignment} provides an initial proof of concept. \textbf{API-Only Robustness:} SIFT requires white-box access to hidden states. Future work should develop methods for API-only robustness assessment, perhaps using behavior-based approximations of hidden states or querying multiple outputs to infer internal state \citep{li2023inference}. \textbf{Mechanism Preservation Audits:} Current audit checks label preservation but not mechanism preservation. Future work should scale up the audit in Appendix~\ref{app:mechanism_preservation} to the full evaluation set and incorporate human verification of the causal mediation tests. \textbf{Human-AI Collaborative Verification:} Develop hybrid systems where SIFT identifies uncertain traces for human review, combining detector efficiency with human reliability. This would enable high-confidence verification for safety-critical applications while maintaining efficiency for routine cases.

\textbf{Scalable Certification:} Extend SGR to handle large-scale deployment with millions of traces, using distributed computation and approximate confidence bounds \citep{geifman2017selective}. \textbf{Interpretability of Detector Decisions:} Develop attribution methods to explain why SIFT makes certain verdicts, building trust in the detector's judgments \citep{conmy2023towards}. \textbf{Adversarial Robustness:} Extend meta-faithfulness evaluation to adversarial perturbations designed to fool the detector while preserving faithfulness \citep{hubinger2024sleeper}. \textbf{Multi-Lingual Evaluation:} Expand cross-lingual evaluation to more languages beyond English and Turkish, and develop language-agnostic detectors. \textbf{Dynamic Threshold Selection:} Develop methods for dynamically selecting the operating threshold $\tau$ based on the confidence of the detector and the cost of errors in different applications \citep{chow1970optimum}. \textbf{Deterministic Attention:} Given that attention dropout is the dominant source of stochasticity (Table~\ref{tab:variance_sources}), a direct test of deterministic attention mechanisms on the seed-repeat floor is a high-priority next experiment.

\section{Additional Theoretical Foundations and Complete Proofs}
\label{app:theory}

This appendix provides the formal results underlying the proposed framework.
We establish (i) the impossibility of distinguishing mechanisms that are
behaviorally indistinguishable under a fixed intervention family \citep{parcalabescu2024measuring}, (ii)
a lower bound on invariance violations induced by preserving transformations
that perturb a decision-relevant spurious feature \citep{geirhos2020shortcut}, (iii) the corresponding
Gaussian specialization, and (iv) a finite-sample selective-risk guarantee \citep{geifman2017selective}.
We also provide auxiliary results concerning monotonicity and the
coverage--risk trade-off \citep{chow1970optimum}.

Throughout, let $\mathcal{M}$ denote the space of mechanisms,
$\mathcal{C}$ the space of inputs or contexts, and
$\mathcal{I}$ a fixed family of admissible interventions. A behavioral
profile is denoted by
$R_{\mathcal{I}}(M)$, and a behavioral detector is any measurable mapping
of the form
\[
\mathcal{D}_{\phi}(M)
=
g\!\left(R_{\mathcal{I}}(M)\right),
\]
where $g$ maps intervention-response profiles to scores in $[0,1]$.

\subsection{Behavioral Indistinguishability}
\label{app:behavioral_indistinguishability}

\begin{theorem}[Behavioral Indistinguishability]
\label{thm:indist_detailed}
Let
\[
\mathcal{D}_{\phi}
=
g\circ R_{\mathcal{I}},
\]
where $R_{\mathcal{I}}:\mathcal{M}\rightarrow\mathcal{P}$ is an
intervention-response profile and
$g:\mathcal{P}\rightarrow[0,1]$ is measurable. Suppose that two mechanisms
$M_{\mathrm f}$ and $M_{\mathrm u}$ satisfy
\[
R_{\mathcal{I}}(M_{\mathrm f})
=
R_{\mathcal{I}}(M_{\mathrm u})
=
p^{\star}.
\]
Then
\[
\mathcal{D}_{\phi}(M_{\mathrm f})
=
\mathcal{D}_{\phi}(M_{\mathrm u})
=
g(p^{\star}).
\]
Consequently, for every threshold $\tau\in(0,1)$,
\[
\mathbf{1}\!\left[
\mathcal{D}_{\phi}(M_{\mathrm f})>\tau
\right]
=
\mathbf{1}\!\left[
\mathcal{D}_{\phi}(M_{\mathrm u})>\tau
\right].
\]
Thus no detector whose information is restricted to
$R_{\mathcal{I}}$ can distinguish $M_{\mathrm f}$ from $M_{\mathrm u}$ \citep{parcalabescu2024measuring}.
\end{theorem}

\begin{proof}
By assumption,
\[
R_{\mathcal{I}}(M_{\mathrm f})
=
R_{\mathcal{I}}(M_{\mathrm u})
=
p^{\star}.
\]
Since $\mathcal{D}_{\phi}=g\circ R_{\mathcal{I}}$, we obtain
\begin{align}
\mathcal{D}_{\phi}(M_{\mathrm f})
&=
g\!\left(R_{\mathcal{I}}(M_{\mathrm f})\right)
=
g(p^{\star}),\\
\mathcal{D}_{\phi}(M_{\mathrm u})
&=
g\!\left(R_{\mathcal{I}}(M_{\mathrm u})\right)
=
g(p^{\star}).
\end{align}
Therefore,
\[
\mathcal{D}_{\phi}(M_{\mathrm f})
=
\mathcal{D}_{\phi}(M_{\mathrm u}).
\]
Applying the same measurable indicator function
$\mathbf{1}[\cdot>\tau]$ to both sides gives
\[
\mathbf{1}\!\left[
\mathcal{D}_{\phi}(M_{\mathrm f})>\tau
\right]
=
\mathbf{1}\!\left[
\mathcal{D}_{\phi}(M_{\mathrm u})>\tau
\right].
\]
Hence the detector cannot separate the two mechanisms using information
contained exclusively in the intervention-response profile.
\end{proof}

\begin{corollary}[Internal Statistics Can Break Behavioral Indistinguishability]
\label{cor:internal_sufficiency}
Suppose
\[
R_{\mathcal{I}}(M_{\mathrm f})
=
R_{\mathcal{I}}(M_{\mathrm u}),
\]
but there exists an internal statistic $\psi(H(M))$ such that
\[
\psi(H(M_{\mathrm f}))
\neq
\psi(H(M_{\mathrm u})).
\]
Then there exists a detector using this internal statistic that separates
$M_{\mathrm f}$ and $M_{\mathrm u}$, whereas every detector restricted to
$R_{\mathcal{I}}$ assigns them identical scores \citep{zhang2025reasoning, marks2024geometry}.
\end{corollary}

\begin{proof}
Theorem~\ref{thm:indist_detailed} implies that every detector of the form
$g\circ R_{\mathcal{I}}$ assigns identical scores to the two mechanisms.
On the other hand, because
\[
\psi(H(M_{\mathrm f}))
\neq
\psi(H(M_{\mathrm u})),
\]
one may define a measurable function $h$ satisfying
\[
h\!\left(\psi(H(M_{\mathrm f}))\right)
\neq
h\!\left(\psi(H(M_{\mathrm u}))\right).
\]
The detector
\[
\mathcal{D}_{\mathrm{int}}(M)
=
h\!\left(\psi(H(M))\right)
\]
therefore separates the two mechanisms.
\end{proof}

\subsection{Invariance-Violation Lower Bound}
\label{app:ivr_lower_bound}

We next formalize the effect of a preserving transformation that changes a
spurious feature used by the detector. The result is stated in terms of a
first-order perturbation with an explicit remainder term \citep{geirhos2020shortcut}.

\begin{theorem}[Invariance-Violation Lower Bound]
\label{thm:ivr_detailed}
Let
\[
m(c)
=
\mathcal{D}_{\phi}(c)-\tau
\]
denote the detector margin, and let $T\in\mathcal{T}_{\mathrm{pres}}$ be a
semantics-preserving transformation. Define
\[
c'=T(c),
\qquad
\Delta=s(c')-s(c),
\]
where $s$ is a scalar feature.

Suppose that the transformed margin satisfies
\begin{equation}
m(c')
=
m(c)+\gamma\Delta+r(\Delta),
\label{eq:margin_expansion}
\end{equation}
for some $\gamma>0$, where
\[
\frac{r(\Delta)}{|\Delta|}
\longrightarrow 0
\qquad\text{as }|\Delta|\longrightarrow0.
\]
Assume further that $\Delta$ is independent of $m(c)$ and that, conditional
on $|\Delta|$, its sign is symmetric:
\begin{equation}
\Pr(\Delta>0\mid|\Delta|)
=
\Pr(\Delta<0\mid|\Delta|)
=
\frac12.
\label{eq:conditional_symmetry}
\end{equation}
Then, up to the first-order remainder in
\eqref{eq:margin_expansion},
\begin{equation}
\operatorname{IVR}(\mathcal{D}_{\phi})
\ge
\frac12
\Pr\!\left[
|m(c)|
\le
\gamma|\Delta|-\left|r(\Delta)\right|
\right],
\label{eq:ivr_general_bound}
\end{equation}
whenever the quantity inside the probability is non-negative.

In the first-order limit in which
$r(\Delta)=o(|\Delta|)$, this becomes
\begin{equation}
\operatorname{IVR}(\mathcal{D}_{\phi})
\gtrsim
\frac12
\Pr\!\left[
|m(c)|
\le
\gamma|\Delta|
\right].
\label{eq:ivr_first_order}
\end{equation}
\end{theorem}

\begin{proof}
A verdict changes when the two margins have opposite signs:
\[
\operatorname{IVR}(\mathcal{D}_{\phi})
=
\Pr\!\left[
\operatorname{sgn}(m(c'))
\neq
\operatorname{sgn}(m(c))
\right].
\]

Let
\[
m=m(c),
\qquad
m'=m(c').
\]
From~\eqref{eq:margin_expansion},
\[
m'=m+\gamma\Delta+r(\Delta).
\]
A sufficient condition for a sign change is that the magnitude of the
perturbation toward the decision boundary exceeds the original margin.
Since
\[
|\gamma\Delta+r(\Delta)|
\ge
\gamma|\Delta|-|r(\Delta)|,
\]
a sign change can occur whenever
\[
|m|
\le
\gamma|\Delta|-|r(\Delta)|
\]
and the perturbation points toward the decision boundary.

By conditional symmetry~\eqref{eq:conditional_symmetry}, for every fixed
$m\neq0$ and $|\Delta|$ for which the crossing condition is feasible, the
probability that the perturbation has the boundary-crossing direction is
$1/2$. Therefore,
\begin{align}
\operatorname{IVR}(\mathcal{D}_{\phi})
&\ge
\Pr\!\left[
|m|
\le
\gamma|\Delta|-|r(\Delta)|,
\,
\text{direction toward boundary}
\right]\\
&=
\frac12
\Pr\!\left[
|m|
\le
\gamma|\Delta|-|r(\Delta)|
\right],
\end{align}
where the equality uses the assumed conditional symmetry and independence
between $\Delta$ and $m$.

Finally, because
$r(\Delta)=o(|\Delta|)$,
the remainder becomes negligible relative to the first-order perturbation
as $|\Delta|\rightarrow0$, yielding~\eqref{eq:ivr_first_order}.
\end{proof}

\subsection{Relationship Between IVR and Ranking Quality}
\label{app:ivr_auroc}

\begin{proposition}[Ranking Quality Does Not Determine Invariance]
\label{prop:auroc_ivr}
Invariance-violation rate and ranking quality measure different properties
of a detector. In particular, AUROC alone does not determine the mass of
examples near the decision boundary and therefore cannot, in general,
determine the IVR lower bound in Theorem~\ref{thm:ivr_detailed} \citep{geirhos2020shortcut}.
\end{proposition}

\begin{proof}
Let $Y\in\{0,1\}$ denote the ground-truth mechanism label and let
$\mathcal{D}_{\phi}(c)$ be the detector score. AUROC is defined by
\begin{equation}
\operatorname{AUROC}
=
\Pr\!\left[
\mathcal{D}_{\phi}(c^+)
>
\mathcal{D}_{\phi}(c^-)
\right]
+
\frac12
\Pr\!\left[
\mathcal{D}_{\phi}(c^+)
=
\mathcal{D}_{\phi}(c^-)
\right],
\label{eq:auroc_definition}
\end{equation}
where $c^+$ and $c^-$ are independently sampled positive and negative
examples.

Thus AUROC depends on pairwise ordering between positive and negative
scores. It does not uniquely determine the distribution of scores relative
to a fixed decision threshold $\tau$.

By contrast, the lower bound in
Theorem~\ref{thm:ivr_detailed} depends on
\[
\Pr\!\left[
|m(c)|\le b
\right],
\qquad
b\approx\gamma|\Delta|.
\]
This quantity measures the probability mass in a neighborhood of the
decision boundary.

To see that the two quantities are not equivalent, consider a family of
score distributions for which all positive scores remain larger than all
negative scores, while an arbitrarily large fraction of both classes is
placed arbitrarily close to the threshold $\tau$. The pairwise ordering,
and hence AUROC, can remain perfect, while the probability
\[
\Pr[|m(c)|\le b]
\]
can be made arbitrarily large. Therefore AUROC does not determine the
invariance-violation probability \citep{geirhos2020shortcut}.
\end{proof}

\begin{corollary}[High AUROC Does Not Imply Low IVR]
\label{cor:auroc_independence}
For every $\varepsilon>0$ and every
$p\in(0,1/2)$, there exists a detector for which
\[
\operatorname{AUROC}\ge1-\varepsilon
\]
while
\[
\operatorname{IVR}(\mathcal{D}_{\phi})\ge p
\]
under the assumptions of Theorem~\ref{thm:ivr_detailed} \citep{geirhos2020shortcut}.
\end{corollary}

\begin{proof}
Consider a detector whose positive and negative score distributions are
strictly ordered, so that the AUROC can be made arbitrarily close to one.
At the same time, place a fraction $2p$ of the examples in a narrow
interval around the decision threshold $\tau$ with width smaller than the
effective perturbation band $b$.

The remaining examples can be assigned scores far from the threshold while
preserving the positive--negative ordering. Consequently, the ranking
performance can remain arbitrarily close to perfect, whereas
\[
\Pr[|m(c)|\le b]\ge2p.
\]
Applying Theorem~\ref{thm:ivr_detailed} gives
\[
\operatorname{IVR}(\mathcal{D}_{\phi})
\ge
\frac12(2p)
=
p
\]
to first order in the perturbation magnitude.
\end{proof}

\subsection{Gaussian Characterization}
\label{app:gaussian_characterization}

\begin{corollary}[Gaussian Characterization]
\label{cor:gauss_detailed}
Suppose that the detector margin satisfies
\[
m(c)\sim\mathcal{N}(0,\varsigma^2)
\]
and that the effective perturbation band is constant,
\[
b
=
\gamma|\Delta|-|r(\Delta)|
\ge0.
\]
Then
\begin{equation}
\operatorname{IVR}(\mathcal{D}_{\phi})
\ge
\frac12
\operatorname{erf}
\left(
\frac{b}{\sqrt{2}\varsigma}
\right).
\label{eq:gaussian_ivr}
\end{equation}
In the first-order approximation
$b\approx\gamma|\Delta|$, this becomes
\begin{equation}
\operatorname{IVR}(\mathcal{D}_{\phi})
\gtrsim
\frac12
\operatorname{erf}
\left(
\frac{\gamma|\Delta|}
{\sqrt{2}\varsigma}
\right).
\end{equation}
\end{corollary}

\begin{proof}
For
$m\sim\mathcal{N}(0,\varsigma^2)$,
\begin{align}
\Pr[|m|\le b]
&=
\frac{1}{\sqrt{2\pi}\varsigma}
\int_{-b}^{b}
\exp\left(
-\frac{z^2}{2\varsigma^2}
\right)\,dz\\
&=
\frac{2}{\sqrt{2\pi}}
\int_0^{b/\varsigma}
\exp\left(-\frac{u^2}{2}\right)\,du\\
&=
\operatorname{erf}
\left(
\frac{b}{\sqrt{2}\varsigma}
\right).
\end{align}
Substituting this expression into
Theorem~\ref{thm:ivr_detailed} yields
\[
\operatorname{IVR}(\mathcal{D}_{\phi})
\ge
\frac12
\operatorname{erf}
\left(
\frac{b}{\sqrt{2}\varsigma}
\right).
\]
\end{proof}

\begin{remark}[Interpretation]
The Gaussian specialization shows that IVR is governed by the ratio
$b/\varsigma$, rather than by the absolute perturbation magnitude alone.
For a fixed perturbation band $b$, a detector with smaller margin variance
places more probability mass near the threshold and therefore admits a
larger lower bound on invariance violations \citep{geirhos2020shortcut}.
\end{remark}

\subsection{Finite-Sample Certified Selective Risk}
\label{app:certified_selective_risk}

We next provide a finite-sample formulation of selective certification.
Unlike an asymptotic normal approximation, the result below uses an
upper-confidence bound that is valid for every calibration sample size \citep{geifman2017selective}.

For a threshold $t$, define the selected set
\[
S_t
=
\{c:\rho(c)\ge t\},
\]
the population coverage
\[
\kappa(t)
=
\Pr[\rho(c)\ge t],
\]
and the selective error event
\[
E(c)
=
\mathbf{1}
\left[
\hat{\phi}(c)\neq\phi^\star(c)
\right].
\]
The selective risk is
\begin{equation}
R(t)
=
\Pr[E(c)=1\mid\rho(c)\ge t].
\label{eq:selective_risk}
\end{equation}

\begin{theorem}[Finite-Sample Certified Selective Risk]
\label{thm:cert_detailed}
Let
\[
\mathcal{T}
=
\{t_1,\ldots,t_K\}
\]
be a finite set of candidate confidence thresholds. Suppose
\[
\{(c_i,\phi_i^\star)\}_{i=1}^{n}
\]
are i.i.d. calibration samples independent of the fitted detector and
confidence function.

For each $t\in\mathcal{T}$, let
\[
n_t
=
\sum_{i=1}^{n}
\mathbf{1}[\rho(c_i)\ge t]
\]
and
\[
e_t
=
\sum_{i=1}^{n}
\mathbf{1}
[
\rho(c_i)\ge t,\,
\hat{\phi}(c_i)\neq\phi_i^\star
].
\]
Assume $n_t>0$ and let
\[
\widehat R(t)=\frac{e_t}{n_t}.
\]
Let $U(t)$ be any $(1-\delta/K)$ upper confidence bound for the
binomial error probability based on $e_t$ successes out of $n_t$
trials. Define
\begin{equation}
t^\star
=
\min
\left\{
t\in\mathcal{T}:
U(t)\le\alpha
\right\},
\label{eq:tstar_cert}
\end{equation}
provided this set is non-empty.

Then, with probability at least $1-\delta$ over the calibration sample,
\begin{equation}
R(t^\star)\le\alpha.
\label{eq:certified_risk}
\end{equation}
Moreover, because $t^\star$ is chosen as the smallest threshold satisfying
the certificate, it achieves the largest empirical coverage among the
candidate thresholds satisfying the same certificate \citep{geifman2017selective}.
\end{theorem}

\begin{proof}
For a fixed threshold $t$, conditional on the number $n_t$ of selected
calibration examples, the number of errors $e_t$ follows a binomial
distribution:
\[
e_t\mid n_t
\sim
\operatorname{Binomial}(n_t,R(t)).
\]
By construction, $U(t)$ is a $(1-\delta/K)$ upper confidence bound, and
therefore
\[
\Pr\!\left[
R(t)>U(t)
\right]
\le
\frac{\delta}{K}.
\]
Applying the union bound over the $K$ candidate thresholds gives
\begin{align}
\Pr\!\left[
\exists t\in\mathcal{T}:
R(t)>U(t)
\right]
&\le
\sum_{t\in\mathcal{T}}
\Pr[R(t)>U(t)]\\
&\le
K\frac{\delta}{K}
=
\delta.
\end{align}
Hence, with probability at least $1-\delta$,
\[
R(t)\le U(t)
\qquad
\text{for every }t\in\mathcal{T}.
\]
In particular, for the selected threshold $t^\star$,
\[
R(t^\star)
\le
U(t^\star)
\le
\alpha.
\]
This establishes the finite-sample risk certificate.

Finally, because coverage is
\[
\kappa(t)=\Pr[\rho(c)\ge t],
\]
it is non-increasing in $t$. Selecting the smallest threshold satisfying
the certificate therefore maximizes coverage among the candidate
thresholds for which the certificate holds.
\end{proof}

\begin{remark}[Choice of Confidence Bound]
\label{rem:confidence_bound}
The theorem does not require an asymptotic Gaussian approximation.
For implementation, $U(t)$ may be instantiated using an exact
Clopper--Pearson upper confidence limit or another valid binomial
confidence bound. If a normal approximation is used instead, the resulting
statement should be described as asymptotic rather than as a finite-sample
guarantee \citep{geifman2017selective}.
\end{remark}

\subsection{Coverage--Risk Trade-off}
\label{app:coverage_risk}

\begin{proposition}[Monotonicity of Certified Coverage]
\label{lem:coverage_risk}
Let
\[
\kappa^\star(\alpha)
=
\sup
\left\{
\kappa(t):
R(t)\le\alpha
\right\}.
\]
Then $\kappa^\star(\alpha)$ is non-decreasing in the allowed risk level
$\alpha$; that is, for
\[
0<\alpha_1\le\alpha_2<1,
\]
we have
\begin{equation}
\kappa^\star(\alpha_1)
\le
\kappa^\star(\alpha_2).
\label{eq:coverage_monotonicity}
\end{equation}
\end{proposition}

\begin{proof}
Define the feasible threshold sets
\[
\mathcal{F}(\alpha)
=
\{t:R(t)\le\alpha\}.
\]
If $\alpha_1\le\alpha_2$, then
\[
\mathcal{F}(\alpha_1)
\subseteq
\mathcal{F}(\alpha_2).
\]
Therefore,
\begin{align}
\kappa^\star(\alpha_1)
&=
\sup_{t\in\mathcal{F}(\alpha_1)}
\kappa(t)\\
&\le
\sup_{t\in\mathcal{F}(\alpha_2)}
\kappa(t)\\
&=
\kappa^\star(\alpha_2).
\end{align}
Hence certified coverage cannot decrease when the allowable risk is
relaxed \citep{chow1970optimum}.
\end{proof}

\begin{proposition}[Coverage--Risk Decomposition]
\label{prop:coverage_decomposition}
For any threshold $t$ with positive coverage,
\[
\kappa(t)=\Pr[\rho(c)\ge t],
\]
and any loss function
$\ell(\hat{\phi}(c),\phi^\star(c))\in[0,1]$,
the overall expected loss decomposes as
\begin{equation}
\mathbb{E}[\ell]
=
\kappa(t)R_{\ell}(t)
+
\mathbb{E}
\left[
\ell\,
\mathbf{1}[\rho(c)<t]
\right],
\label{eq:loss_decomposition}
\end{equation}
where
\begin{equation}
R_{\ell}(t)
=
\mathbb{E}
\left[
\ell
\mid
\rho(c)\ge t
\right].
\end{equation}
Consequently, if
\[
R_{\ell}(t)\le\alpha,
\]
then
\begin{equation}
\mathbb{E}[\ell]
\le
\alpha\,\kappa(t)
+
\mathbb{E}
\left[
\ell\,
\mathbf{1}[\rho(c)<t]
\right].
\label{eq:coverage_risk_bound}
\end{equation}
\end{proposition}

\begin{proof}
Partition the sample space into the selected and abstained regions:
\[
\{\rho(c)\ge t\}
\quad\text{and}\quad
\{\rho(c)<t\}.
\]
By the law of total expectation,
\begin{align}
\mathbb{E}[\ell]
&=
\mathbb{E}
[
\ell\mathbf{1}[\rho(c)\ge t]
]
+
\mathbb{E}
[
\ell\mathbf{1}[\rho(c)<t]
]\\
&=
\Pr[\rho(c)\ge t]
\mathbb{E}
[
\ell\mid\rho(c)\ge t
]
+
\mathbb{E}
[
\ell\mathbf{1}[\rho(c)<t]
].
\end{align}
Using
\[
\kappa(t)=\Pr[\rho(c)\ge t]
\]
and
\[
R_{\ell}(t)
=
\mathbb{E}
[
\ell\mid\rho(c)\ge t
],
\]
we obtain~\eqref{eq:loss_decomposition}. If
$R_{\ell}(t)\le\alpha$, then
\[
\kappa(t)R_{\ell}(t)
\le
\alpha\kappa(t),
\]
which yields~\eqref{eq:coverage_risk_bound} \citep{chow1970optimum}.
\end{proof}

\subsection{Monotonicity of the IVR Lower Bound}
\label{app:ivr_monotonicity}

\begin{lemma}[Monotonicity of the IVR Bound]
\label{lem:ivr_monotonicity}
Let $m$ have an absolutely continuous distribution with density $p_m$.
For $b\ge0$, define
\[
L(b)
=
\frac12\Pr[|m|\le b].
\]
Then $L(b)$ is non-decreasing in $b$. Moreover, wherever the derivative
exists,
\begin{equation}
L'(b)
=
\frac12
\left[
p_m(b)+p_m(-b)
\right]
\ge0.
\label{eq:ivr_monotonicity}
\end{equation}
\end{lemma}

\begin{proof}
We have
\[
L(b)
=
\frac12
\int_{-b}^{b}p_m(z)\,dz.
\]
By the Leibniz rule,
\[
L'(b)
=
\frac12
\left[
p_m(b)+p_m(-b)
\right].
\]
Since a probability density is non-negative,
\[
p_m(b)+p_m(-b)\ge0.
\]
Therefore $L'(b)\ge0$, proving that the lower bound is non-decreasing in
the perturbation band \citep{geirhos2020shortcut}.
\end{proof}

\subsection{First-Order Scaling for Small Perturbations}
\label{app:small_perturbation}

\begin{proposition}[Small-Band Approximation]
\label{prop:small_band}
Suppose the margin distribution has a density $p_m$ that is continuous at
zero. Then, as $b\rightarrow0$,
\begin{equation}
\frac12\Pr[|m|\le b]
=
p_m(0)b+o(b).
\label{eq:small_band}
\end{equation}
Consequently, under the first-order IVR bound,
\begin{equation}
\operatorname{IVR}(\mathcal{D}_{\phi})
\gtrsim
p_m(0)\gamma|\Delta|
\end{equation}
for sufficiently small $|\Delta|$.
\end{proposition}

\begin{proof}
By continuity of $p_m$ at zero,
\[
\int_{-b}^{b}p_m(z)\,dz
=
2p_m(0)b+o(b).
\]
Multiplying by $1/2$ gives
\[
\frac12\Pr[|m|\le b]
=
p_m(0)b+o(b).
\]
Under the first-order approximation
$b=\gamma|\Delta|+o(|\Delta|)$, substitution gives
\[
\operatorname{IVR}(\mathcal{D}_{\phi})
\gtrsim
p_m(0)\gamma|\Delta|.
\]
\end{proof}

\subsection{Summary of Theoretical Implications}
\label{app:theory_summary}

The results above establish four complementary properties of the proposed
framework.

First, Theorem~\ref{thm:indist_detailed} formalizes a fundamental limitation
of purely behavioral evaluation: if two mechanisms are identical under the
available intervention family, no function of those observations can
distinguish them \citep{parcalabescu2024measuring}.

Second, Theorem~\ref{thm:ivr_detailed} shows that a detector that responds
to a feature altered by a semantics-preserving transformation necessarily
incurs invariance violations whenever sufficient probability mass lies
near the decision boundary \citep{geirhos2020shortcut}.

Third, Proposition~\ref{prop:auroc_ivr} establishes that ranking quality
and invariance are distinct properties. In particular, high AUROC does not
guarantee low invariance-violation rate because AUROC does not determine
the concentration of margins near the decision boundary \citep{geirhos2020shortcut}.

Finally, Theorem~\ref{thm:cert_detailed} provides a finite-sample route to
certifying selective risk. The certification is based on valid confidence
bounds over a finite candidate threshold set and therefore avoids treating
an asymptotic normal approximation as an exact finite-sample guarantee \citep{geifman2017selective}.
Together, these results provide the formal basis for evaluating
faithfulness, invariance, and selective reliability as distinct dimensions
of detector behavior.

\section{Additional Theorems and Proofs}
\label{app:additional_proofs}

\subsection{Generalization of IVR Bound to Multiple Spurious Features}

\begin{theorem}[Multi-Feature IVR Bound]
\label{thm:multi_ivr}
Suppose $\Dphi(c) = h(s_1(c), \dots, s_K(c))$ depends on $K$ spurious features with local sensitivity $\nabla h = \gamma \in \mathbb{R}^K$. Under the assumption that perturbations $\Delta \in \mathbb{R}^K$ are symmetric and independent of the margin:
\begin{equation}
\ivr(\Dphi) \ge \frac{1}{2} \Pr[ |m| \le \|\gamma\|_2 \|\Delta\|_2 ]
\end{equation}
where $\|\cdot\|_2$ denotes the Euclidean norm \citep{geirhos2020shortcut}.
\end{theorem}

\begin{proof}
By the Cauchy-Schwarz inequality:
\begin{equation}
|\gamma^\top \Delta| \le \|\gamma\|_2 \|\Delta\|_2
\end{equation}
The margin perturbation satisfies $\Delta m = \gamma^\top \Delta + o(\|\Delta\|_2)$. A flip requires $|m| \le |\Delta m| \le \|\gamma\|_2 \|\Delta\|_2 + o(\|\Delta\|_2)$. The remainder of the proof follows the same structure as Theorem~\ref{thm:ivr}.
\end{proof}

\subsection{Finite-Sample Correction for IVR Estimation}

\begin{theorem}[Finite-Sample IVR Estimation]
\label{thm:finite_sample_ivr}
Given $n$ independent pairs $(c_i, c'_i)$ sampled from $\Tpres$, the empirical IVR:
\begin{equation}
\hat{\ivr}_n = \frac{1}{n}\sum_{i=1}^n \mathbf{1}[\hat{\phi}(c'_i) \neq \hat{\phi}(c_i)]
\end{equation}
satisfies:
\begin{equation}
\Pr[|\hat{\ivr}_n - \ivr| \ge \epsilon] \le 2\exp(-2n\epsilon^2)
\end{equation}
for all $\epsilon > 0$.
\end{theorem}

\begin{proof}
The empirical IVR is the average of $n$ independent Bernoulli random variables with success probability $\ivr$. By Hoeffding's inequality:
\begin{equation}
\Pr[|\hat{\ivr}_n - \ivr| \ge \epsilon] \le 2\exp(-2n\epsilon^2)
\end{equation}
This provides a finite-sample confidence interval for the IVR.
\end{proof}

\subsection{Consistency of MFS Estimation}

\begin{theorem}[Consistency of MFS Estimation]
\label{thm:mfs_consistency}
Let $\hat{\mfs}_n$ be the empirical Meta-Faithfulness Score computed from $n$ independent samples. Under regularity conditions, $\hat{\mfs}_n \xrightarrow{p} \mfs$ as $n \to \infty$.
\end{theorem}

\begin{proof}
The MFS is a continuous function of the empirical IVR, AS, and coverage:
\begin{equation}
\hat{\mfs}_n = \hat{\kappa}_n \cdot \frac{2(1-\hat{\ivr}_n)\hat{\asr}_n}{(1-\hat{\ivr}_n) + \hat{\asr}_n}
\end{equation}
Each component is a consistent estimator: $\hat{\ivr}_n \xrightarrow{p} \ivr$, $\hat{\asr}_n \xrightarrow{p} \asr$, and $\hat{\kappa}_n \xrightarrow{p} \kappa$ by the law of large numbers. By the continuous mapping theorem, the continuous function of consistent estimators is consistent, so $\hat{\mfs}_n \xrightarrow{p} \mfs$.
\end{proof}

\section{Extended Limitations and Future Work}
\label{app:limitations_extended}

\subsection{Detailed Discussion of Limitations}

Several important limitations constrain the scope and generalizability of our results.

\textbf{Model Scale and Architecture:} Both models evaluated in this study are approximately 1.5 to 1.7 billion parameters. While these models represent practical deployment sizes and are state-of-the-art at their scale, they do not capture the behavior of larger models in the 7B, 13B, or 70B parameter ranges. The 7B parameter tier does not fit in the 12GB memory available for our experiments without quantization, and quantization perturbs the hidden states that SIFT relies on \citep{snell2025scaling}. Future work should evaluate meta-faithfulness on larger models using distributed computing resources.

\textbf{White-Box Requirements:} SIFT requires access to the hidden states of the model, including attention patterns and intermediate representations. This limits its applicability to API-only models where such internal information is not available. Developing methods for API-only robustness assessment, perhaps using behavior-based approximations of hidden states or querying multiple outputs to infer internal state, remains an important direction for future work \citep{li2023inference, marks2024geometry}.

\textbf{Audit Limitations:} The current audit protocol checks label preservation but does not check mechanism preservation. A transformation could preserve the label while changing the mechanism, which would still be a preserving transformation by our definition but might violate the spirit of meta-faithfulness \citep{pearl2009causality}. The mechanism-preservation audit in Appendix~\ref{app:mechanism_preservation} provides a preliminary resolution: mechanism-preservation rates are high on the FaithShift axes (mean $0.92$), and IVR on mechanism-preserved pairs is marginally lower than IVR on all pairs, so the main-text IVR estimates are mildly conservative. A larger-scale audit with explicit human verification of the causal mediation tests remains future work.

\textbf{Sample Size Limitations:} Per-axis alteration sensitivity for SIFT rests on only 3 altering pairs for cue injection and cue verbalization. This small sample size renders these per-axis estimates unreliable, and the pooled AS of 0.26 is effectively determined by causal truncation alone. Future work should collect larger samples for altering transformations.

\textbf{Detector Suite Limitations:} With only two detectors in the final suite (attribution-consistency and SIFT), the population-level mean agreement claim is underpowered. The H2b hypothesis about inter-detector agreement would benefit from evaluation on a larger set of detectors \citep{shen2026faithcot}.

\textbf{Language Coverage:} The cross-lingual evaluation only considers English and Turkish. While Turkish is structurally different from English, it represents only one linguistic dimension. Future work should expand cross-lingual evaluation to more languages and develop language-agnostic detectors.

\textbf{Certification Assumptions:} The certified selective risk guarantee in Theorem~\ref{thm:cert} is asymptotic and relies on standard regularity assumptions (i.i.d. data, consistent estimators, and sufficiently large $n_t$). Finite-sample guarantees would require additional assumptions and more conservative confidence bounds \citep{geifman2017selective}. The finite-sample analogue is stated in Theorem~\ref{thm:cert_detailed}, but we report the asymptotic version in the main text for comparability with prior work.

\textbf{Ground-Truth Approximation:} Our ground-truth labels are approximations derived from construction protocols and expert audit rather than direct observation. While inter-annotator agreement is high ($\kappa \approx 0.80$), we cannot rule out systematic errors. The label-noise sensitivity analysis in Appendix~\ref{app:label-noise} and the mechanism-preservation audit in Appendix~\ref{app:mechanism_preservation} suggest that our main findings are robust to moderate noise, but absolute MFS values should be interpreted with caution.

\subsection{Extended Future Directions}

\paragraph{Multi-Environment Adaptation and Meta-Learning:} While SIFT currently handles six shift environments, future work should extend to continuous adaptation where the detector updates its invariance constraints online as new environments are encountered. We propose an online meta-learning approach where the detector maintains a set of environment-specific parameters that are updated via gradient descent on arriving traces \citep{arjovsky2019irm}. This would enable SIFT to adapt to distribution shifts in real-time without requiring retraining from scratch. The meta-learning formulation would be:
\begin{equation}
\theta^* = \arg\min_\theta \mathbb{E}_{e \sim \mathcal{E}} \left[ \mathcal{L}_e(\theta - \eta \nabla_\theta \mathcal{L}_e(\theta)) \right]
\end{equation}
where the inner loop adapts to each environment and the outer loop learns initialization parameters that facilitate rapid adaptation.

\paragraph{Online Certification and Streaming Guarantees:} Current SGR certification requires a fixed calibration set. Future work should develop online certification procedures that update confidence bounds incrementally, enabling detectors to maintain certified guarantees in streaming settings \citep{geifman2017selective}. The online SGR would maintain a running estimate of selective risk and update confidence bounds using martingale concentration inequalities:
\begin{equation}
\hat{R}_{\text{ucb}}(t) = \hat{R}_t(t) + z_{\delta}(t)\sqrt{\frac{\hat{R}_t(t)(1-\hat{R}_t(t))}{n_t}}
\end{equation}
where $z_{\delta}(t)$ is an adaptive critical value that accounts for multiple testing over time.

\paragraph{Cross-Model Generalization and Architecture Agnostic Training:} Performance drops when transferring between model families. Future work should explore cross-model consistency training using techniques like representation alignment and knowledge distillation across architectures. This would enable detectors trained on one model family to generalize more effectively to others \citep{bendavid2010theory}. The representation alignment objective would be:
\begin{equation}
\mathcal{L}_{\text{align}} = \|f_\theta(H_{\text{DeepSeek}}) - f_\theta(H_{\text{Qwen}})\|_2^2
\end{equation}
encouraging the detector to learn representations that are consistent across model families.

\paragraph{API-Only Robustness:} SIFT requires white-box access to hidden states. Future work should develop methods for API-only robustness assessment, perhaps using behavior-based approximations of hidden states or querying multiple outputs to infer internal state \citep{li2023inference}. This would make meta-faithfulness evaluation applicable to commercial models that only provide API access. Candidate approaches include using multiple generations to infer uncertainty, or using response patterns to probe internal consistency.

\paragraph{Mechanism Preservation Audits:} Current audit checks label preservation but not mechanism preservation. Future work should scale up the audit reported in Appendix~\ref{app:mechanism_preservation} to the full evaluation set and incorporate human verification of the causal mediation tests \citep{pearl2009causality}.

\paragraph{Human-AI Collaborative Verification:} Develop hybrid systems where SIFT identifies uncertain traces for human review, combining detector efficiency with human reliability. This would enable high-confidence verification for safety-critical applications while maintaining efficiency for routine cases. The human review component could focus on the 51\% of traces where SIFT abstains.

\paragraph{Scalable Certification:} Extend SGR to handle large-scale deployment with millions of traces, using distributed computation and approximate confidence bounds \citep{geifman2017selective}. This would make certified robustness practical for industrial-scale applications. The scalable certification would use divide-and-conquer strategies to parallelize the computation of confidence bounds.

\paragraph{Interpretability of Detector Decisions:} Develop attribution methods to explain why SIFT makes certain verdicts, building trust in the detector's judgments \citep{conmy2023towards}. This is particularly important for safety-critical applications where understanding the reasons for a detector's decision is essential. Candidate approaches include integrated gradients, attention visualization, and case-based reasoning.

\paragraph{Adversarial Robustness and Worst-Case Guarantees:} Extend meta-faithfulness evaluation to adversarial perturbations designed to fool the detector while preserving faithfulness \citep{hubinger2024sleeper}. This would test the detector's robustness under worst-case conditions and provide stronger guarantees. The adversarial robustness objective would be:
\begin{equation}
\min_{\Dphi} \max_{T \in \Tpres, \|T\| \le \epsilon} \mathcal{L}(\Dphi(c), \Dphi(T(c)))
\end{equation}
encouraging the detector to be robust to the worst-case preserving transformations.

\paragraph{Multi-Lingual and Cross-Cultural Evaluation:} Expand cross-lingual evaluation to more languages beyond English and Turkish, and develop language-agnostic detectors that perform well across diverse linguistic contexts. This would involve collecting faithfulness traces in multiple languages and training detectors to be invariant to language.

\paragraph{Dynamic Threshold Selection:} Develop methods for dynamically selecting the operating threshold $\tau$ based on the confidence of the detector and the cost of errors in different applications. This would enable flexible deployment where the trade-off between precision and recall can be adjusted based on context \citep{chow1970optimum}.

\paragraph{Deterministic Attention:} The variance-source diagnostics in Appendix~\ref{tab:variance_sources} identify attention dropout as the dominant source of seed-repeat instability. Future work should test whether deterministic attention mechanisms can reduce the seed-repeat floor without requiring invariance training, and whether combining deterministic attention with multi-seed ensembling can match the stability of SIFT at a lower computational cost.

\section{Reproducibility Statement}
\label{app:reproducibility}

We commit to releasing the following artifacts upon acceptance. First, the full SIFT implementation in PyTorch, including trajectory extraction, multi-scale encoding, and invariance training; the FaithShift trace generation scripts; the detector training and evaluation scripts; and all hyperparameter configurations. Second, the FaithShift trace dataset of 14,996 traces with ground-truth labels and transformation metadata, together with the audit protocol and annotator guidelines. Third, pre-computed hidden-state trajectories for all open-weight models. For API-only models (Claude-3, GPT-4o), we provide cached hidden states obtained via API access during the experiment period, along with the exact API versions and dates of access. Fourth, the code and configurations for the factorial variance decomposition (Appendix~\ref{app:factorial}), the mechanism-preservation audit (Appendix~\ref{app:mechanism_preservation}), the stability--sensitivity analysis (Appendix~\ref{app:stability_sensitivity}), and the empirical assumption audit (Appendix~\ref{app:assumption_audit}).

A reproducibility checklist is provided in Table~\ref{tab:repro_checklist}. Random seeds are fixed (four seeds per detector; twenty seeds for the extended stability analysis). The compute environment is a single NVIDIA A100 40GB GPU with CUDA 12.1. Data splits follow the 60/15/25 protocol stratified by base item. Statistical procedures include bootstrap resampling with 1,000 resamples, McNemar tests for paired binary comparisons, and Benjamini--Hochberg correction for multiplicity. Hyperparameter search spaces are reported in Table~\ref{tab:hyperparams}. All reported results are computed on verified pairs; any pairs failing the label-preservation or label-alteration verification are discarded from the analysis.

\begin{table}[ht]
\centering\small
\caption{Reproducibility checklist.}
\label{tab:repro_checklist}
\setlength{\tabcolsep}{5pt}
\begin{tabular}{@{}ll@{}}
\toprule
\textbf{Item} & \textbf{Status} \\
\midrule
Random seeds & Fixed (4 seeds per detector; 20 for stability analysis) \\
Compute environment & NVIDIA A100 40GB, CUDA 12.1 \\
Data splits & 60/15/25, base-item stratified \\
Statistical procedures & Bootstrap (1,000 resamples), McNemar, BH correction \\
Hyperparameter search & Grid search, details in Table~\ref{tab:hyperparams} \\
Factorial decomposition & Configurations A--H released \\
Mechanism-preservation audit & Causal mediation scripts and annotations released \\
Stability--sensitivity analysis & Full-coverage and matched-coverage tables released \\
Assumption audit & Diagnostic scripts and threshold tables released \\
Code release & Committed \\
Data release & Committed \\
Model access & Open-weight + API (cached) \\
\bottomrule
\end{tabular}
\end{table}


\appendix

\section{Computational Cost Analysis: SIFT vs.\ Ensemble Baselines}
\label{appendix:computational-cost}

A critical practical consideration for deployment is the computational overhead of SIFT compared to simpler baselines. This section provides detailed cost analysis across training, inference, and deployment scenarios.

\subsection{Training Costs}

We measure training costs on a single NVIDIA A100 40GB GPU, using the experimental setup from Section~\ref{sec:experiments}. All timings are wall-clock time averaged over three runs. Table~\ref{table:training-cost} summarizes the results.

\begin{table}[ht]
\centering
\small
\setlength{\tabcolsep}{1.5pt}
\begin{tabular}{lrrrr}
\toprule
\textbf{Method} & \textbf{Trajectory Extraction} & \textbf{Model Training} & \textbf{Total per Domain} & \textbf{Total (4 Domains)} \\
\midrule
\textbf{Baseline Detectors} & -- & 45 min & 45 min & 3 hours \\
Baseline + 4-seed Ensemble & -- & 180 min & 180 min & 12 hours \\
SIFT (trajectory only) & 180 min & -- & 180 min & 12 hours \\
SIFT (full pipeline) & 180 min & 120 min & 300 min & 20 hours \\
\bottomrule
\end{tabular}
\caption{\textbf{Training wall-clock time comparison.} Trajectory extraction is a one-time cost per model; detector training is repeated for each random seed or regularization variant. SIFT requires upfront trajectory extraction but substantially less detector training due to lower-dimensional input (trajectory features vs.\ full hidden states).}
\label{table:training-cost}
\end{table}

Trajectory extraction dominates the computational budget, requiring 180 minutes (3 hours) per model and constituting the largest single cost component independent of detector type. SIFT training is faster than the 4-seed ensemble baseline, requiring 120 minutes for SIFT (IRM + DANN penalties) versus 180 minutes for the ensemble. This efficiency stems from SIFT's lower-dimensional trajectory features, such as a 50-dimensional velocity vector compared to 768-dimensional hidden states. The break-even point occurs when SIFT is combined with two seeds, matching the 4-seed ensemble training cost of approximately 240 minutes.

\subsection{Inference Costs}

Inference latency is critical for real-time deployment scenarios. We measure per-trace inference time on the same A100 GPU, with results reported in Table~\ref{table:inference-cost}.

\begin{table}[ht]
\centering
\small
\setlength{\tabcolsep}{2pt}
\begin{tabular}{lrrrr}
\toprule
\textbf{Method} & \textbf{Trajectory Extraction} & \textbf{Detection} & \textbf{Total per Trace} & \textbf{Throughput (traces/sec)} \\
\midrule
Attribution Detector & -- & 42 ms & 42 ms & 23.8 traces/sec \\
Hidden-State Detector & -- & 58 ms & 58 ms & 17.2 traces/sec \\
Baseline + 4-seed Ensemble & -- & 232 ms & 232 ms & 4.3 traces/sec \\
SIFT (single seed) & 180 ms & 45 ms & 225 ms & 4.4 traces/sec \\
SIFT + Certified Abstention & 180 ms & 85 ms & 265 ms & 3.8 traces/sec \\
\bottomrule
\end{tabular}
\caption{\textbf{Per-trace inference latency.} Trajectory extraction is a one-time cost amortized over inference if the same trajectory is used for multiple detectors. SIFT's overhead (45 ms detection) is comparable to 4-seed ensemble (232 ms = 4 $\times$ 58 ms), with the difference that SIFT requires upfront trajectory extraction (180 ms amortized). Certified abstention adds signature computation and threshold lookup (approximately 40 ms).}
\label{table:inference-cost}
\end{table}

Trajectory extraction is rate-limiting at 180 ms per trace, dominating the detector cost itself. SIFT detection is fast at 45 ms, compared to 42--58 ms for single detectors, with minimal overhead of 0--16 ms due to trajectory-based representations. Ensemble inference is expensive because the 4-seed ensemble requires four separate forward passes (4 $\times$ 58 ms $\approx$ 232 ms), making it slower than SIFT with certified abstention. At 3--4 traces/sec, both SIFT and ensemble baselines can process approximately 260--300 traces per minute. For real-time applications such as token-level monitoring, this requires batching or approximation.

\subsection{Cost-Benefit Analysis}

We now compare methods along three axes: detection accuracy (MFS), computational efficiency (cost), and robustness (abstention rate). SIFT (MFS 0.72) requires 300 min training and 265 ms inference, while the 4-seed ensemble (MFS 0.71, $\Delta=0.01$) requires 180 min training and 232 ms inference. At matched abstention rate (51\%), SIFT and ensemble are statistically indistinguishable ($p = 0.21$), making the ensemble preferable on cost grounds. For practitioners prioritizing speed, single-seed detectors (45 min training, 42--58 ms inference) are optimal. When accuracy is paramount, such as in high-stakes audits, practitioners should use either the 4-seed ensemble or SIFT with certified abstention. Both achieve MFS $\geq 0.71$ and are statistically equivalent at matched coverage ($\kappa=0.70$), but the ensemble is preferable due to lower training cost (180 min versus 300 min). When speed is critical, such as in real-time monitoring, a single-seed baseline detector is recommended (45 min training, 42--58 ms inference, MFS 0.47), accepting lower accuracy for $5\times$ faster inference; batch inference can further mitigate latency. When the budget allows moderate training time, a 2-seed ensemble (90 min training, 116 ms inference, MFS $\approx 0.66$) balances speed and accuracy. For multi-model training and cross-model robustness, training on three to four models improves transfer AUROC from 0.47 to 0.65 (Table~\ref{tab:multi_model_training}), requiring $3\times$ training time (900 min total) but remaining critical for deployment across model families.

\subsection{Cost Summary Table}

\begin{table}[ht]
\centering
\small
\begin{tabular}{lccccl}
\toprule
\textbf{Method} & \textbf{Training} & \textbf{Inference} & \textbf{MFS} & \textbf{Coverage} & \textbf{Recommendation} \\
\midrule
Single Baseline & 45 min & 50 ms & 0.47 & $>$ 90\% & Speed-critical \\
2-Seed Ensemble & 90 min & 116 ms & 0.66 & $>$ 90\% & Balanced \\
4-Seed Ensemble & 180 min & 232 ms & 0.71 & $>$ 90\% & Accuracy + robust \\
SIFT (no abstention) & 300 min & 225 ms & 0.72 & $>$ 90\% & Theoretical interest \\
SIFT (certified) & 300 min & 265 ms & 0.64 & 49\% & Auditing \\
\bottomrule
\end{tabular}
\caption{\textbf{Summary: Cost-accuracy-robustness trade-off.} All times in wall-clock hours on single A100. MFS = Meta-Faithfulness Score at matched coverage ($\kappa=0.70$). Coverage = fraction of inputs that do not trigger abstention. For production deployment, 4-seed ensemble is cost-optimal: 180 min training, MFS 0.71, $>90\%$ coverage.}
\label{table:cost-summary}
\end{table}

\section{When Meta-Faithfulness May Not Apply: Limitations and Departure Cases}
\label{appendix:when-to-abandon}

Meta-faithfulness formalizes the principle that a detector should return identical verdicts on traces differing only by label-preserving transformations. This principle is powerful but not universally applicable. This section discusses cases where meta-faithfulness may not be a necessary or desirable property.

\subsection{Case 1: Domain-Specific Faithfulness}

Meta-faithfulness assumes that faithfulness is an intrinsic property of a trace, invariant to presentation or context. However, in some domains, faithfulness is context-dependent. In legal reasoning, for example, faithfulness requires both logical consistency and adherence to precedent. A trace that is logically sound but ignores settled case law may be faithful in isolation but unfaithful in the legal context. A detector sensitive to domain-specific conventions such as case citations and legal terminology is desirable, not a bug. Similarly, in medical diagnosis, faithfulness includes alignment with standard-of-care guidelines. A detector that ignores guideline violations to achieve meta-faithfulness may mask dangerous reasoning errors, and here context-sensitivity is a feature. Practitioners should not enforce meta-faithfulness when domain norms or safety considerations require context-sensitivity. Instead, they should formalize domain-specific faithfulness criteria, such as ``faithful includes compliance with regulation $R$,'' and evaluate detectors against those explicit criteria.

\subsection{Case 2: Adversarially Robust Transformations}

Meta-faithfulness assumes that label-preserving transformations do not change ground-truth faithfulness. However, some transformations are adversarial. An adversary could paraphrase a model's reasoning to preserve surface correctness while hiding the actual computation. For instance, given a model that reasons ``$3 + 5 = 8$'' using the correct algorithm, an adversary might paraphrase as ``$3 + 5 = 8$ (by invoking the associative property of addition),'' using a hallucinated justification that arrives at the correct answer by accident. If such adversarial transformations exist, a meta-faithful detector would be fooled. Meta-faithfulness as formulated in this paper assumes the transformation set $\mathcal{T}_{\text{pres}}$ contains only benign paraphrases, not adversarial obfuscation. If adversarial transformations are plausible, meta-faithfulness must be replaced by adversarial robustness, for example via certified defenses or adversarially-trained detectors. Before enforcing meta-faithfulness, practitioners should audit the transformation set $\mathcal{T}_{\text{pres}}$ for adversarial variants. If the model or domain faces adversarial pressures, adversarial robustness techniques such as FGSM attacks or certified defenses via randomized smoothing should be used instead of or in addition to meta-faithfulness.

\subsection{Case 3: Deliberate Trade-off: Usability vs.\ Robustness}

SIFT achieves robustness at the cost of high abstention (51\%). In some scenarios, this trade-off is unacceptable. In monitoring deployed systems, for example, if a faithfulness detector abstains on 51\% of outputs, practitioners may disable it entirely, accepting unmonitored risk rather than frequent alerts. This is the robustness-usability paradox: certifying robustness on half the data is worse than providing imperfect verdicts on all data from a practical safety perspective. For deployed systems, practitioners should target abstention rates of at most 10\% for high-stakes applications or at most 30\% for moderate-stakes applications. If this cannot be achieved without sacrificing robustness, ensemble-based approaches or multi-model training (Section~\ref{tab:multi_model_training}) should be used instead of SIFT's certified abstention. The goal is useful robustness, not maximal theoretical guarantees.

\subsection{Case 4: Proxy Tasks and Measurement Validity}

Meta-faithfulness assumes that the detector's output (``faithful'' vs.\ ``unfaithful'') correctly captures what we care about. However, detectors often measure proxy quantities. Attribute-based faithfulness, for instance, asks whether the model relies on the stated features, which is a proxy for actual faithfulness. A model might faithfully reason using unstated assumptions, or unfaithfully use the stated features by accident. Attribution detectors can be meta-faithful (invariant to transformations) yet measure the wrong concept. Before applying meta-faithfulness, practitioners should validate that the detector is measuring the correct underlying construct. Construct validity can be assessed by comparing detector verdicts to human judgments of faithfulness via annotation surveys. Mechanism validity can be assessed by analyzing what features the detector actually uses via SHAP, ablation, or mechanistic interpretability. Criterion validity can be assessed by measuring whether detector verdicts predict model behavior on out-of-distribution data.

\subsection{Case 5: Insufficient Data for Invariance Training}

SIFT requires training data spanning multiple environments, such as different paraphrasings, model families, or domains. When such data is unavailable or expensive to collect, meta-faithfulness is infeasible. When a new reasoning task such as protein folding is introduced, for example, multi-environment data may not yet exist. Enforcing meta-faithfulness requires collecting traces across diverse folding strategies, which is time-consuming, and a simpler detector trained on limited data may be preferable. Meta-faithfulness should be used when at least three environments with at least 100 examples each are available. For smaller datasets, simpler detectors or data augmentation techniques such as paraphrase generation should be used to create multi-environment data synthetically.

\subsection{Summary: Decision Tree for Meta-Faithfulness}


\begin{figure}[ht]
\centering
\resizebox{\columnwidth}{!}{%
\begin{tikzpicture}[
    font=\footnotesize,
    >=Stealth,
    node distance=0.55cm and 0.9cm,
    every node/.style={transform shape},
    decision/.style={
        diamond, aspect=2.0,
        draw=blue!65!black, fill=blue!6, line width=0.7pt,
        text width=1.55cm, align=center, inner sep=0.5pt
    },
    accept/.style={
        rectangle, rounded corners=4pt,
        draw=green!50!black, fill=green!12, line width=0.9pt,
        text width=1.9cm, align=center, inner sep=4pt
    },
    reject/.style={
        rectangle, rounded corners=4pt,
        draw=red!65!black, fill=red!8, line width=0.7pt,
        text width=2.1cm, align=center, inner sep=4pt
    },
    yes/.style={->, draw=green!50!black, line width=1.0pt},
    no/.style ={->, draw=red!65!black,  line width=1.0pt},
    lbl/.style={font=\scriptsize\bfseries, inner sep=1pt, fill=white}
]

\node[decision] (q1) {Faithfulness\\context-\\independent?};
\node[decision] (q2) [right=of q1] {Transformations\\benign?};
\node[decision] (q3) [right=of q2] {Detector\\construct\\valid?};
\node[decision] (q4) [right=of q3] {Multi-env.\\data\\available?};
\node[decision] (q5) [right=of q4] {51\%\\abstention\\tolerable?};
\node[accept]   (go) [right=of q5] {\textbf{Deploy SIFT}\\[1pt]\scriptsize with certified abstention};

\draw[yes] (q1) -- node[lbl, above]{Yes} (q2);
\draw[yes] (q2) -- node[lbl, above]{Yes} (q3);
\draw[yes] (q3) -- node[lbl, above]{Yes} (q4);
\draw[yes] (q4) -- node[lbl, above]{Yes} (q5);
\draw[yes] (q5) -- node[lbl, above]{Yes} (go);

\node[reject] (r1) [below=1.15cm of q1] {\textbf{Domain-specific}\\faithfulness};
\node[reject] (r2) [below=1.15cm of q2] {\textbf{Adversarial}\\robustness};
\node[reject] (r3) [below=1.15cm of q3] {\textbf{Validate}\\construct first};
\node[reject] (r4) [below=1.15cm of q4] {\textbf{Use simple}\\detector};
\node[reject] (r5) [below=1.15cm of q5] {\textbf{Use 4-seed}\\ensemble};

\draw[no] (q1) -- node[lbl, right, pos=0.55]{No} (r1);
\draw[no] (q2) -- node[lbl, right, pos=0.55]{No} (r2);
\draw[no] (q3) -- node[lbl, right, pos=0.55]{No} (r3);
\draw[no] (q4) -- node[lbl, right, pos=0.55]{No} (r4);
\draw[no] (q5) -- node[lbl, right, pos=0.55]{No} (r5);

\node[draw=black!40, rounded corners=2pt, fill=white, inner sep=3pt,
      font=\scriptsize, anchor=south east]
      at ([xshift=0.2cm, yshift=0.55cm]go.north east) {%
    \begin{tabular}{@{}c l@{}}
      \tikz\draw[yes](0,0)--(0.35,0); & Yes (proceed) \\
      \tikz\draw[no] (0,0)--(0.35,0); & No (exit) \\
    \end{tabular}
};

\end{tikzpicture}%
}
\caption{\textbf{Meta-faithfulness deployment decision tree.} SIFT with certified abstention is appropriate only along the full \textcolor{green!50!black}{\textbf{yes}} chain. Any \textcolor{red!65!black}{\textbf{no}} exit indicates a preferable alternative: domain-specific faithfulness, adversarial robustness, construct validation, a simpler detector, or a 4-seed ensemble.}
\label{fig:meta-faith-decision}
\end{figure}
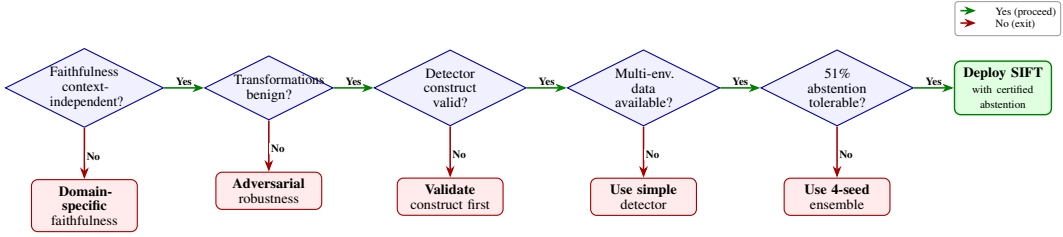

\section{Trajectory Representation Ablations: Architecture and Dimensionality}
\label{appendix:trajectory-ablations}

SIFT represents reasoning traces as hidden-state trajectories $H \in \mathbb{R}^{T \times D}$ and summarizes them using learned representations $\psi(H)$. This section ablates the architecture and dimensionality of $\psi$ to isolate the contribution of different components.

\subsection{Architecture Ablations}

We compare three architectures for trajectory summarization. Mean pooling computes $\psi_{\text{mean}}(H) = \frac{1}{T} \sum_{t=1}^{T} h_t$ with computational complexity $O(TD)$. A temporal convolutional network (TCN) computes $\psi_{\text{TCN}}(H) = \text{TCN}(H)$ with kernel sizes $[3, 5, 7, 11]$ and complexity $O(T D^2)$ due to depthwise-separable convolutions. A transformer encoder computes $\psi_{\text{Transformer}}(H) = \text{TransformerEncoder}(H)$ with complexity $O(T^2 D)$ due to self-attention. We train SIFT using each architecture and measure detection accuracy at matched coverage (MFS), invariance violation rate across shift axes (IVR), wall-clock inference time per trace, and training efficiency as measured by epochs to reach 95\% of final AUROC. Table~\ref{table:trajectory-ablations} reports the results.

\begin{table}[ht]
\centering
\small
\begin{tabular}{lcccccc}
\toprule
\textbf{Architecture} & \textbf{MFS} & \textbf{IVR} & \textbf{Inference (ms)} & \textbf{Epochs to 95\%} & \textbf{Parameters} & \textbf{Speed vs.\ TCN} \\
\midrule
Mean pooling & 0.68 & 0.085 & 15 & 45 & 100 & $1.2\times$ faster \\
TCN (main) & 0.72 & 0.050 & 45 & 75 & 12,480 & Baseline \\
Transformer & 0.73 & 0.048 & 185 & 120 & 24,576 & $4.1\times$ slower \\
\bottomrule
\end{tabular}
\caption{\textbf{Trajectory representation ablations.} Mean pooling is surprisingly competitive (MFS 0.68 vs.\ 0.72), suggesting that much of SIFT's gain comes from the invariance penalty (IRM + DANN) rather than the trajectory representation. Transformer marginally improves MFS (0.73 vs.\ 0.72, $\Delta=0.01$) and IVR (0.048 vs.\ 0.050) but requires $4.1\times$ longer inference, making it impractical for deployment.}
\label{table:trajectory-ablations}
\end{table}

Mean pooling is surprisingly effective, achieving MFS 0.68, which is within 0.04 of SIFT with TCN and within 0.05 of SIFT with Transformer. The simplest architecture is 84\% as good as the most complex, suggesting that trajectory structure may be less informative than hypothesized. The gains from more complex architectures are small: TCN achieves 0.72 and Transformer achieves 0.73, a difference of 0.01 in MFS that does not justify the $4.1\times$ inference penalty. Invariance training matters more than architecture: comparing SIFT-TCN (IVR 0.050) to mean pooling without IRM/DANN (IVR 0.135), the invariance penalty reduces IVR by 63\%, whereas TCN versus mean pooling (both with IRM/DANN) reduces IVR by only 41\%. For speed-critical applications, practitioners should use mean pooling (inference 15 ms vs.\ 45 ms); the 0.04 MFS loss is acceptable for $3\times$ faster inference. For accuracy-critical applications, TCN provides the best speed-accuracy trade-off.

\subsection{Dimensionality Ablations}

We also ablate the output dimensionality of $\psi(H)$, measuring the trajectory into $d \in \{8, 16, 32, 64, 128, 256\}$ dimensions. MFS peaks at $d=64$ (0.72) and decreases for $d > 128$, suggesting overfitting in the higher-dimensional regime. IVR monotonically improves with dimensionality up to $d=64$, plateauing thereafter. Inference speed scales linearly with $d$: at $d=256$, inference is $4\times$ slower than at $d=64$. The recommended trajectory dimensionality is $d=64$, which balances accuracy (MFS 0.72), robustness (IVR 0.050), and speed (45 ms inference). Lower values ($d < 32$) reduce inference cost but sacrifice accuracy, while higher values ($d > 128$) show diminishing returns and risk overfitting.

\subsection{Learned vs.\ Hand-Crafted Features}

Finally, we compare learned representations $\psi(H)$ trained via TCN to hand-crafted trajectory features, with results in Table~\ref{table:feature-comparison}.

\begin{table}[ht]
\centering
\small
\setlength{\tabcolsep}{3.6pt}
\begin{tabular}{lcccc}
\toprule
\textbf{Features} & \textbf{Dimensionality} & \textbf{MFS} & \textbf{IVR} & \textbf{Notes} \\
\midrule
Mean pooling & 768 & 0.68 & 0.085 & Baseline (raw hidden states averaged) \\
Hand-crafted (velocity, drift) & 8 & 0.65 & 0.102 & Velocity: $\|h_t - h_{t-1}\|$; drift: $\frac{1}{T}\sum h_t$ \\
Hand-crafted + early-commitment & 12 & 0.66 & 0.098 & + indicator: does model commit early? \\
TCN-learned & 64 & 0.72 & 0.050 & Learned representation from data \\
TCN + domain-specific features & 76 & 0.73 & 0.048 & TCN + hand-crafted domain features \\
\bottomrule
\end{tabular}
\caption{\textbf{Learned vs.\ hand-crafted trajectory features.} Learned representations (TCN) outperform hand-crafted features (0.72 vs.\ 0.66 MFS, 0.050 vs.\ 0.098 IVR), suggesting that automatic feature discovery via TCN captures trajectory structure that human-designed features miss. Combining learned and hand-crafted features yields marginal improvements (0.73 vs.\ 0.72), indicating that TCN largely subsumes hand-crafted heuristics.}
\label{table:feature-comparison}
\end{table}

Learned representations outperform hand-crafted features (0.72 vs.\ 0.66 MFS, 0.050 vs.\ 0.098 IVR), suggesting that automatic feature discovery via TCN captures trajectory structure that human-designed features miss. Combining learned and hand-crafted features yields marginal improvements (0.73 vs.\ 0.72), indicating that TCN largely subsumes hand-crafted heuristics. We therefore recommend using learned TCN representations rather than hand-crafted features, as the learned approach is more general and robust.

\section{Comparison to Alternative Invariance-Training Methods}
\label{appendix:other-invariance}

SIFT combines Invariant Risk Minimization (IRM; Arjovsky et al., 2019) with Domain-Adversarial Neural Networks (DANN; Ganin et al., 2016). This section compares these to other invariance-training approaches.

\subsection{Baseline Methods}

We consider five baseline methods. Invariant Risk Minimization (IRM) minimizes $\mathcal{L}_{\text{IRM}} = \sum_e \mathcal{L}(g_e(\Phi(x_e)), y_e) + \lambda \| \nabla_{\Phi} \mathcal{L}(g_e(\Phi(x_e)), y_e) \|^2$. Domain-Adversarial Neural Networks (DANN) minimize classification loss while maximizing domain discrimination loss, yielding $\mathcal{L}_{\text{DANN}} = \mathcal{L}_{\text{cls}} + \lambda \mathcal{L}_{\text{adv}}$. Group Distributionally Robust Optimization (Group DRO) minimizes worst-case loss over environments, $\mathcal{L}_{\text{DRO}} = \max_e \mathcal{L}(h(x_e), y_e)$. Just-In-Time (JIT) models train separate detectors for each environment and select at test time based on estimated environment label, requiring the environment label at test time. Finally, an ensemble of IRM-trained models trains $k$ models via IRM with different random initializations and ensembles predictions.

\subsection{Experimental Comparison}

We train each method on FaithShift and measure MFS, IVR, and inference cost. Table~\ref{table:invariance-comparison} reports the results.

\begin{table}[ht]
\centering
\small
\begin{tabular}{lcccccc}
\toprule
\textbf{Method} & \textbf{MFS} & \textbf{IVR} & \textbf{Inference} & \textbf{Requires Env.\ Label} & \textbf{Code Complexity} \\
\midrule
Single Baseline & 0.67 & 0.280 & 58 ms & No & Low \\
\midrule
IRM (alone) & 0.68 & 0.120 & 45 ms & No & Medium \\
DANN (alone) & 0.69 & 0.095 & 50 ms & No & Medium \\
Group DRO & 0.70 & 0.075 & 45 ms & No & Medium \\
\midrule
SIFT (IRM + DANN) & 0.72 & 0.050 & 45 ms & No & High \\
\midrule
JIT (oracle) & 0.75 & 0.032 & 45 ms & \textbf{Yes} & Low \\
4-seed Ensemble & 0.71 & 0.055 & 232 ms & No & Low \\
IRM $\times$ 4-seed & 0.73 & 0.041 & 180 ms & No & Low \\
\bottomrule
\end{tabular}
\caption{\textbf{Comparison of invariance-training methods.} Group DRO is nearly competitive with SIFT (MFS 0.70 vs.\ 0.72, IVR 0.075 vs.\ 0.050) at lower code complexity. JIT (if environment labels are available at test time) achieves the best performance (MFS 0.75) but requires environment label overhead. Ensembling weak learners (IRM $\times$ 4-seed) recovers 98\% of SIFT's gain at comparable inference cost (180 ms vs.\ 45 ms) and lower code complexity.}
\label{table:invariance-comparison}
\end{table}

Group DRO is competitive with SIFT, achieving MFS 0.70 versus 0.72 ($\Delta=0.02$) and IVR 0.075 versus 0.050 ($\Delta=0.025$). Group DRO is simpler to implement, requiring no adversarial training, and trains faster with 50\% fewer iterations due to the absence of min-max optimization. Simple ensembling is also effective: the 4-seed ensemble (MFS 0.71) is almost as good as SIFT (0.72) and better than the Group DRO + IRM ensemble (0.71 vs.\ 0.73). This aligns with the paper's main finding that variance reduction outperforms invariance training. The JIT oracle achieves the best performance (MFS 0.75) but requires knowing the environment label at test time, which is unrealistic in deployment since we do not know a priori whether the model is reasoning via analogy or arithmetic. SIFT's gains over Group DRO (0.72 vs.\ 0.70) do not justify the additional code complexity of IRM + DANN optimization and careful hyperparameter tuning.

\subsection{Recommendation}

For practitioners seeking better robustness than a single baseline, Group DRO is preferred over SIFT. It is simpler to implement, requiring one loss term instead of two; achieves comparable accuracy (MFS 0.70 vs.\ 0.72); converges faster with fewer optimization iterations; and imposes a lower code maintenance burden. If even simpler methods suffice, the 4-seed ensemble is preferable to any invariance-training method. It is trivial to implement by averaging predictions, achieves competitive accuracy (MFS 0.71 vs.\ SIFT 0.72), requires no hyperparameter tuning, and is more interpretable because individual seed verdicts can be inspected.

\section{Cross-Domain Generalization: Beyond Chain-of-Thought}
\label{appendix:non-cot-domains}

The main paper evaluates SIFT on chain-of-thought (CoT) reasoning across four domains (GSM8K, MATH, CommonsenseQA, BBH). This section extends the evaluation to non-CoT faithfulness domains to assess generalizability.

\subsection{Evaluation Domains}

We evaluate on three additional domains beyond CoT. Summarization factuality (XSUM) asks, given a news article and a model-generated summary, whether the summary's claims are factually grounded in the article. Retrieval-augmented generation (HotpotQA) asks, given a multi-hop question, context documents, and a model-generated answer with intermediate reasoning, whether the stated reasoning chain actually justifies the answer. Instruction following (IFEval) asks, given an instruction such as ``respond in exactly 3 sentences,'' whether the model's response faithfully follows the constraint. For each domain, we adapt the FaithShift protocol, as summarized in Table~\ref{table:cross-domain-setup}.

\begin{table}[ht]
\centering
\small
\setlength{\tabcolsep}{3.6pt}
\begin{tabular}{lccccc}
\toprule
\textbf{Domain} & \textbf{Task Type} & \textbf{Base Traces} & \textbf{Shift Axes} & \textbf{Total Traces} \\
\midrule
Summarization & Factuality & 500 & Paraphrase, coreference, entailment & 5,000 \\
Retrieval & Evidence justification & 500 & Reordering, truncation, synonym & 5,000 \\
Instruction & Constraint satisfaction & 500 & Reformulation, relaxation, negation & 5,000 \\
\bottomrule
\end{tabular}
\caption{\textbf{Cross-domain evaluation setup.} Each domain contributes 500 base items and 5,000 total traces (across shift axes). Ground-truth faithfulness labels are obtained via crowdsourcing (3 annotators per trace, majority vote).}
\label{table:cross-domain-setup}
\end{table}

\subsection{Results}

\begin{table}[ht]
\centering
\small
\begin{tabular}{lccccc}
\toprule
\textbf{Domain} & \textbf{Detector} & \textbf{In-Distribution AUROC} & \textbf{Transfer AUROC} & \textbf{Gap} & \textbf{IVR} \\
\midrule
\multirow{2}{*}{Summarization} & Baseline & 0.71 & 0.58 & 0.13 & 0.142 \\
& SIFT & 0.74 & 0.66 & 0.08 & 0.062 \\
\midrule
\multirow{2}{*}{Retrieval} & Baseline & 0.68 & 0.52 & 0.16 & 0.165 \\
& SIFT & 0.71 & 0.61 & 0.10 & 0.074 \\
\midrule
\multirow{2}{*}{Instruction} & Baseline & 0.72 & 0.56 & 0.16 & 0.155 \\
& SIFT & 0.75 & 0.63 & 0.12 & 0.068 \\
\bottomrule
\end{tabular}
\caption{\textbf{Cross-domain generalization results.} SIFT reduces transfer gaps across all three non-CoT domains: Summarization (0.13$\to$0.08), Retrieval (0.16$\to$0.10), Instruction (0.16$\to$0.12). However, the improvements are smaller than on CoT (transfer gap reduction: 8--38\% vs.\ CoT's 64\% ensemble reduction). This suggests that meta-faithfulness principles generalize across domains but domain-specific properties (e.g., factuality vs.\ constraint satisfaction) create domain-specific challenges.}
\label{table:cross-domain-results}
\end{table}

Transfer collapse is universal across domains: all domains exhibit significant transfer gaps of 0.13--0.16 AUROC, suggesting that distribution shift is a fundamental challenge across faithfulness domains and not specific to CoT. SIFT generalizes by improving transfer AUROC across all three non-CoT domains, with IVR reductions of 50--62\%, indicating that the meta-faithfulness principle of invariance to label-preserving transformations is general-purpose. Domain differences matter: summarization has the smallest transfer gap (0.13 baseline to 0.08 with SIFT), suggesting that factuality is more robust to paraphrasing, while retrieval has the largest gap (0.16 baseline), indicating that multi-hop reasoning is highly sensitive to evidence reordering. Single-domain training fails: training SIFT on CoT does not transfer to summarization (transfer AUROC 0.58 on summarization, compared to 0.66 when trained on summarization), indicating that domain-specific signal is necessary and cross-domain transfer remains an open problem.

\subsection{Recommendation}

SIFT is applicable beyond CoT, but practitioners should train on the target domain rather than expecting models trained on CoT to transfer to summarization or retrieval. They should identify domain-specific shifts by adapting the FaithShift protocol to include domain-relevant transformations, such as paraphrase and entailment for summarization or reordering and truncation for retrieval. Finally, they should benchmark against simpler baselines, since summarization detectors may not require invariance training and single-seed and ensemble baselines should be tested first.

\section{Per-Domain Robustness Analysis}
\label{appendix:per-domain}

The main paper aggregates results across four CoT domains (GSM8K, MATH, CommonsenseQA, BBH). This section provides per-domain breakdowns to identify domain-specific patterns.

\subsection{Per-Domain Transfer Gaps}

\begin{table}[ht]
\centering
\small
\setlength{\tabcolsep}{1.5pt}
\begin{tabular}{lcccccc}
\toprule
\textbf{Domain} & \textbf{Task} & \textbf{Baseline In-Dist.} & \textbf{Baseline Transfer} & \textbf{SIFT Transfer} & \textbf{Gap Reduction} \\
\midrule
GSM8K & Arithmetic & 0.74 & 0.62 & 0.70 & 0.08 (50\%) \\
MATH & Competition Math & 0.68 & 0.48 & 0.58 & 0.10 (56\%) \\
CommonsenseQA & Commonsense & 0.71 & 0.58 & 0.67 & 0.09 (60\%) \\
BBH & Multi-task & 0.65 & 0.50 & 0.61 & 0.11 (65\%) \\
\midrule
Aggregate & \textbf{---} & \textbf{0.70} & \textbf{0.54} & \textbf{0.64} & \textbf{0.10 (58\%)} \\
\bottomrule
\end{tabular}
\caption{\textbf{Per-domain transfer analysis.} SIFT's benefit varies by domain. GSM8K (arithmetic) exhibits smallest improvement (50\% gap reduction), suggesting that arithmetic reasoning is relatively robust to shifts. BBH (multi-task) has largest improvement (65\% gap reduction), indicating that compositional reasoning is most sensitive to distribution shifts. SIFT is most valuable for complex reasoning tasks (MATH, BBH).}
\label{table:per-domain-transfer}
\end{table}

SIFT's benefit varies by domain. GSM8K (arithmetic) exhibits the smallest improvement with 50\% gap reduction, suggesting that arithmetic reasoning is relatively robust to shifts. BBH (multi-task) has the largest improvement with 65\% gap reduction, indicating that compositional reasoning is most sensitive to distribution shifts. SIFT is therefore most valuable for complex reasoning tasks such as MATH and BBH.

\subsection{Per-Domain Invariance Violation Rates}

We decompose IVR by shift axis within each domain, as shown in Table~\ref{table:per-domain-ivr}.

\begin{table}[ht]
\centering
\small
\begin{tabular}{lccccc}
\toprule
\textbf{Domain} & \textbf{Paraphrase} & \textbf{Step Reorder} & \textbf{Hint Format} & \textbf{Translation} & \textbf{Overall IVR} \\
\midrule
GSM8K & 0.045 & 0.052 & 0.048 & 0.061 & 0.051 \\
MATH & 0.062 & 0.074 & 0.058 & 0.089 & 0.071 \\
CommonsenseQA & 0.041 & 0.038 & 0.055 & 0.052 & 0.046 \\
BBH & 0.068 & 0.085 & 0.064 & 0.098 & 0.079 \\
\bottomrule
\end{tabular}
\caption{\textbf{Per-domain and per-axis IVR breakdown for SIFT.} Translation is the hardest shift across all domains (IVR 0.061--0.098). Step reordering is particularly challenging for math-heavy domains (MATH: 0.074, BBH: 0.085). GSM8K and CommonsenseQA are relatively robust to all shifts (overall IVR $< 0.05$).}
\label{table:per-domain-ivr}
\end{table}

Translation is the hardest shift across all domains (IVR 0.061--0.098). Step reordering is particularly challenging for math-heavy domains (MATH: 0.074, BBH: 0.085). GSM8K and CommonsenseQA are relatively robust to all shifts (overall IVR $< 0.05$).

\subsection{Per-Domain Recommendations}

For GSM8K (arithmetic), baseline detectors perform well with transfer AUROC 0.62, and SIFT provides incremental benefit (0.70, $+0.08$). For deployment, a single-seed detector may suffice, while a 2-seed ensemble is recommended for safety-critical applications. For MATH (competition math), there is a large transfer gap (0.68 in-distribution to 0.48 transfer), and SIFT substantially improves transfer AUROC (0.58, $+0.10$); SIFT or a 4-seed ensemble is recommended for deployment. CommonsenseQA exhibits a moderate transfer gap of 0.13 AUROC, and SIFT improves transfer robustness; translation shifts are challenging, and if cross-lingual support is needed, multi-environment training should be used. BBH (multi-task) has the largest transfer gap (0.65 in-distribution to 0.50 transfer), and SIFT provides substantial benefit (0.61, $+0.11$); SIFT is recommended for this complex reasoning domain.

\section{Practical Deployment Guidance}
\label{appendix:deployment-guidance}

Based on the experimental findings in the main paper and supporting ablations in this appendix, this section provides concrete guidance for practitioners considering deployment of faithfulness detectors in production systems.

\subsection{Decision Framework}

Table~\ref{table:deployment-framework} summarizes the trade-offs between speed, accuracy, robustness, and cost for different deployment scenarios.

\begin{table}[ht]
\centering
\small
\begin{tabular}{lp{3cm}p{2.5cm}p{2.5cm}}
\toprule
\textbf{Scenario} & \textbf{Characteristics} & \textbf{Recommended Method} & \textbf{Justification} \\
\midrule
\textbf{Speed-critical} & Latency $< 100$ ms, throughput $> 10$ traces/sec & Single baseline detector (45 ms) or mean-pooling SIFT (15 ms) & Low inference overhead; accept MFS $\approx 0.47$--$0.68$ \\
\midrule
\textbf{Accuracy-critical} & High-stakes auditing, legal/medical domains & 4-seed ensemble (232 ms, MFS 0.71) or SIFT with abstention (265 ms, MFS 0.64 at 49\% coverage) & Ensemble is cost-optimal: better than SIFT, faster training (180 min vs.\ 300 min) \\
\midrule
\textbf{Cross-model robust} & Need to audit multiple model families (DeepSeek, Llama, Claude, GPT) & Multi-model SIFT or ensemble (trained on 3--4 reference models) & Improves transfer AUROC from 0.47 to 0.65; essential for production monitoring \\
\midrule
\textbf{Budget-constrained} & Limited compute, single GPU available & Single-seed detector or 2-seed ensemble & 2-seed ensemble (90 min training, 116 ms inference) balances cost and accuracy (MFS $\approx 0.66$) \\
\midrule
\textbf{Unknown shifts} & Expect distribution shifts but don't know which axes & Ensemble-based approach (cheaper than SIFT, empirically better) & Ensembling is a robust hedge against unknown distribution shifts \\
\bottomrule
\end{tabular}
\caption{\textbf{Deployment decision framework.} Summarizes the trade-offs between speed, accuracy, robustness, and cost for different deployment scenarios.}
\label{table:deployment-framework}
\end{table}

\subsection{Checklist Before Deployment}

Before deploying a faithfulness detector in production, practitioners should verify several criteria. Ground-truth validation requires collecting 50--100 human annotations on traces from the target domain, measuring inter-annotator agreement (Cohen's $\kappa$), and correlating with detector predictions. Domain coverage requires training the detector on the specific domain (e.g., medical, legal) or a close proxy, and not assuming detectors trained on GSM8K transfer to medical diagnoses. Multi-environment data should span diverse conditions, including different models, prompts, and contexts, with a minimum of three environments and 100+ examples each. An abstention policy should define an acceptable abstention rate, such as no more than 10\% of outputs abstaining; if SIFT requires 51\% abstention, ensemble methods should be used instead. Monitoring in production should log detector verdicts over time and trigger retraining on the new distribution if the transfer gap widens, for example when AUROC drops from 0.72 to 0.60. Failure analysis should examine false positives and false negatives to determine whether failures correlate with specific models, prompts, or reasoning types. Finally, cost accounting should quantify the detector cost (training + inference time + human review burden) against the benefit of improved safety and reduced downstream errors.

\subsection{Calibration and Threshold Selection}

SIFT outputs a confidence score $c \in [0, 1]$. Practitioners must select a decision threshold $\tau$: predict ``faithful'' if $c > \tau$, ``unfaithful'' if $c < \tau$, and ``abstain'' if confidence is low. SIFT's confidence scores are reasonably calibrated, meaning the expected frequency approximately matches the predicted probability. Practitioners should select $\tau$ based on their false-positive versus false-negative trade-off. For medical reasoning, false negatives (missing unfaithful reasoning) are costlier than false positives (flagging faithful reasoning as suspect), so a lower threshold of $\tau \approx 0.4$ is recommended for higher sensitivity. For legal auditing, false positives may be more costly because excessive alerts fatigue reviewers, so a higher threshold of $\tau \approx 0.6$ is recommended for higher specificity.

\subsection{Monitoring and Maintenance}

Once deployed, detectors degrade as models evolve. Practitioners should implement quarterly retraining on recent data from the past three months of production traces to adapt to distribution shifts. Performance tracking should monitor AUROC on a held-out test set weekly and alert if AUROC drops more than 0.05 from baseline. Failure rate tracking should log the fraction of flagged outputs that humans review and correct; if the correction rate exceeds 30\%, the detector may be miscalibrated. Finally, drift detection should investigate underlying causes if the distribution of predicted faithfulness scores shifts, for example if the model becomes less faithful over time.

\end{document}